\documentclass[journal]{IEEEtran}

\usepackage{bm}
\usepackage{amsmath}
\usepackage{epsf}
\usepackage{graphics}
\usepackage{ amssymb }
\usepackage{stackrel}
\usepackage[dvips]{graphicx}
\usepackage{mathtools}
\usepackage{epsfig}
\usepackage{cite}
\usepackage{colortbl}
\usepackage{color}
\usepackage{enumitem}
\usepackage{soul,xcolor}
\allowdisplaybreaks

\usepackage{bm}
\usepackage{amsmath}
\usepackage{epsf}
\usepackage{graphics}
\usepackage{ amssymb }
\usepackage[dvips]{graphicx}
\usepackage{epsfig}
\usepackage{cite}
\usepackage{graphicx}
\usepackage{epsfig}
\usepackage{latexsym}
\usepackage{amsfonts}
\usepackage{here}
\usepackage{rawfonts}

\usepackage[utf8]{inputenc}
\usepackage[english]{babel}
\usepackage{amsmath}
\usepackage{amsfonts}
\usepackage{amssymb}
\usepackage{color} 
\usepackage{bm}
\usepackage{listings}
\usepackage{caption}
\usepackage{amssymb}
\usepackage{amsthm}
\usepackage{graphicx}
\usepackage{epstopdf}
\usepackage{listings}
\usepackage{float}
\usepackage{amsmath}
\usepackage{amssymb}
\usepackage{amsfonts}
\usepackage{epstopdf}

\usepackage{multirow}
\usepackage{amscd}
\usepackage{mathrsfs}
\usepackage{graphicx}
\usepackage{makecell}
\usepackage{color}
\usepackage{url}
\usepackage{bm}
\usepackage{algorithm}
\usepackage{algorithmic}
\usepackage{setspace}
\usepackage{footnote}
\usepackage{xcolor}
\DeclareMathOperator*{\argmin}{arg\,min}
\newtheorem{theorem}{Theorem}

\newtheorem{definition}{Definition}
\newtheorem{proposition}{Proposition}
\newtheorem{corollary}{Corollary}

\newtheorem{remark}{Remark}

\newtheorem{assumption}{Assumption}

\addto\captionsenglish{}
\usepackage{bbm}

\newcommand{\blue}{\color{black}}

\usepackage{multicol}
\usepackage{hyperref}
\usepackage{mathtools}

\usepackage{lipsum,graphicx,subcaption}

\usepackage{diagbox}
\usepackage[protrusion=true,expansion=true]{microtype}
\pdfoutput=1
\usepackage[font=footnotesize]{caption}
\usepackage{graphicx}
\usepackage{grffile}
\usepackage{tabularx}
\usepackage{booktabs}
\makeatletter
\newcommand{\vast}{\bBigg@{4}}

\newcommand{\Vast}{\bBigg@{5}}
\makeatother
\newcolumntype{Y}{>{\centering\arraybackslash}X}
\makeatletter 
\newcommand\semiHuge{\@setfontsize\semiHuge{22.72}{27.38}}
\makeatother 
\begin{document}

\title{\semiHuge Multi-Edge Server-Assisted Dynamic Federated Learning with an Optimized Floating Aggregation Point}


\author{Bhargav Ganguly,~\IEEEmembership{Student Member,~IEEE}, Seyyedali Hosseinalipour,~\IEEEmembership{Member,~IEEE}, Kwang Taik Kim,~\IEEEmembership{Member,~IEEE}, Christopher G. Brinton,~\IEEEmembership{Senior~Member,~IEEE}, Vaneet Aggarwal,~\IEEEmembership{Senior~Member,~IEEE}, David~J.~Love,~\IEEEmembership{Fellow,~IEEE}, and Mung~Chiang,~\IEEEmembership{Fellow,~IEEE}
\vspace{-8mm}}
\maketitle

\begin{abstract}
We propose cooperative edge-assisted dynamic federated learning ({\tt CE-FL}). {\tt CE-FL} introduces a distributed machine learning (ML) architecture, where data collection is carried out at the end devices, while the model training is conducted \textit{cooperatively} at the end devices and the edge servers, enabled via data offloading from the end devices to the edge servers through base stations. {\tt CE-FL} also introduces floating aggregation point, where the local models generated at the devices and the servers are aggregated at an edge server, which varies from one model training round to another to cope with the network evolution in terms of data distribution and users' mobility. {\tt CE-FL} considers the heterogeneity of network elements in terms of communication/computation models and the proximity to one another. {\tt CE-FL} further presumes a \textit{dynamic} environment with online variation of data at the network devices which causes a \textit{drift} at the ML model performance. We model the processes taken during {\tt CE-FL}, and conduct analytical convergence analysis of its ML model training. We then formulate network-aware {\tt CE-FL} which aims to adaptively optimize all the network elements via tuning their contribution to the learning process, which turns out to be a non-convex mixed integer  problem. Motivated by the large scale of the system, we propose a \textit{distributed optimization solver} to break down the computation of the solution across the network elements. We finally demonstrate the effectiveness of our framework with the data collected from a real-world testbed.
\end{abstract}

\vspace{-1.5mm}

\vspace{-3mm}
\section{Introduction}
\vspace{-.3mm}
\noindent Recent advancements in smart devices (e.g., new chip-sets in phones and smart cars) have made them  capable of collecting large amounts of data in real-time. In the example case of smart cars, collection of more than 4 TB of data per day is predicted~\cite{miller2017autonomous,8759041}. Utilizing this data for training a machine learning (ML) model (e.g., for driving assistance) is the  main motivation for implementing distributed ML techniques over the network edge in 6G-and-beyond systems~\cite{ali20206g}. 

Federated learning (FedL) has been promoted as one of the key distributed ML techniques~\cite{yang2021federated}. Its conventional model training architecture~\cite{konevcny2017federated} is depicted in Fig.~\ref{FLsim}. In FedL, each device trains a local model using its own dataset. The local models of devices are periodically transmitted to a server and aggregated together, e.g., via weighted averaging, to form a global model. The server broadcasts the global model among the devices to initiate the next round of local model training. FedL keeps the dataset of the devices local, which is desired in applications with user privacy concerns, e.g., healthcare, and only uses the server as an \textit{aggregator}. However, implementing FedL using its conventional architecture  over large-scale wireless networks with multiple edge servers and large number of devices with heterogeneous communication/computation capabilities faces major challenges discussed next. 

\begin{figure}[t]
\centering
\includegraphics[width=3.1in,height=.7in]{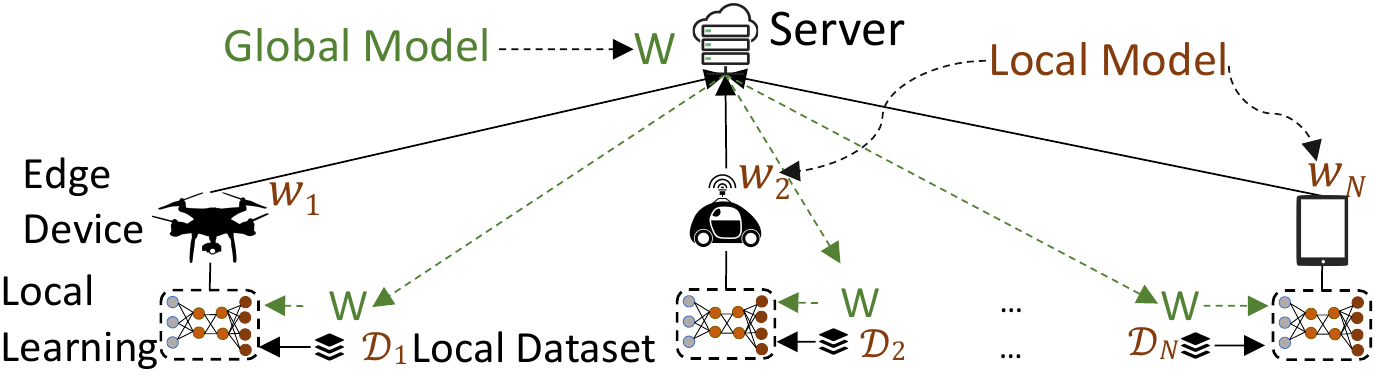}
\vspace{-.8mm}
\caption{Conventional architecture of federated learning.\label{FLsim}}
\vspace{-1.2mm}
\end{figure}

\vspace{-1mm}
\subsection{Motivations and Challenges}
\vspace{-.2mm}
{\color{black}As depicted in Fig.~\ref{fig:1}, we consider the implementation of distributed ML over a network of multiple user equipments (UEs),  multiple edge data centers (DCs), and multiple base stations (BSs),  which is a realistic model for the  network edge~\cite{6678114}.
Our foremost goal is {to introduce a new ML model training architecture to execute distributed ML across the continuum of UEs to DCs}. We further optimize the network-element orchestration, through a distributed algorithm, to have an efficient distributed ML model training procedure.
Our system design is motivated by challenges faced with the implementation of FedL model training architecture over our hierarchical network of interest summarized below:}
\begin{enumerate} [label={(C}{{\arabic*})}]
        \item {\color{black} Conventional FedL neglects the computation power of the  DCs and conducts ML model training  solely at the UEs. Meanwhile,
            some of the UEs may have  large volumes of data they are unable to process. At the same time,  these resource-constrained  UEs might be located near BSs with  access to powerful DCs that can efficiently process data. Also,  excessive computation power of DCs can out-weigh the added latency of
data transfer (i.e., 
 the latency caused by data transfer from the straggling UEs to the DCs can be compensated
by faster data processing at DCs).} \label{c1}
    \item It is not clear which DC should be selected as the model aggregator since the distribution of data changes across time at the UEs. Also, UEs have different channels to the BSs, while different BSs have different delay of data/parameter transfer to different DCs. \label{c2}
    \item {\color{black} There exists heterogeneity across the DCs and the UEs in terms of computation and communication capability. Thus, model training cannot be conducted arbitrary across the network elements. Also, it is not practical to consider the existence of a central controller with global knowledge about the characteristics of the UEs, BSs, and DCs that can solely obtain all the network orchestration decisions pertaining to local and network-wide communication and computational resource allocation.}
    \label{c3}
    \item {\color{black} Although data offloading is not considered in the conventional FedL architecture \cite{konevcny2016federated,mcmahan2017communication}, it is a viable mechanism for many applications of FedL over wireless edge. This is due to the fact that in many FedL applications users do not have strict privacy concerns on their local data. For example, autonomous driving is one of the envisioned applications of FedL, wherein
data collected by the cars are processed for training~\cite{9492062}. In such a setting,
with some form of incentivization (e.g., cashback points and gas credits), vehicle users may become willing
to share their non-private data such as automobile sensor measurements and pictures of
traffic signs.  Furthermore, in privacy-sensitive applications, FedL can be combined with research in \textit{private representation learning}~\cite{azam2022can}, which aims to tackle the privacy concerns associated with the transfer of raw data by obfuscating sensitive information attributes of the users.
Also,  encryption methods~\cite{arora2013secure} can help to facilitate end-to-end data transfer between
end devices and trusted servers to enable the deployment
of efficient data offloading techniques in FedL.\label{C4}}
\end{enumerate}

To respond to the aforementioned challenges, we propose cooperative edge-assisted dynamic/online federated learning ({\tt CE-FL}). {\tt CE-FL} addresses~\ref{c1}\&\ref{C4} via exploiting the computation capability of the edge servers in ML model training, and considering \textit{cooperative process of data points} between the end devices and the servers, where the devices offload a portion of their datasets to the BSs, which in turn disperse the received data across the edge servers. 
To address~\ref{c2}, {\tt CE-FL} introduces the idea of \textit{floating aggregation point}, where the edge server in charge of aggregating the models varies from one model training round to another and is chosen efficiently based on the dynamics of data variations across the devices captured via \textit{model/concept drift}, and the innate characteristics of the servers, BSs, and the end devices. To account for~\ref{c3}, {\tt CE-FL} considers heterogeneity of (i) the number of stochastic gradient descent (SGD) iterations used for model computation across  the end devices and the edge servers, (ii) the mini-batch sizes of SGDs conducted across the servers and end devices, and (iii) the load and power consumption models across the devices and the servers. 
Finally, {\tt CE-FL} introduces a \textit{distributed network orchestration} technique, under which the contribution of all the network elements to ML model training (e.g., number of local SGD iterations, mini-batch sizes, data offloading configuration and routing across the network hierarchy, and the floating aggregator server) is optimized in a distributed fashion.

To put it succinctly, {\tt CE-FL} conducts cooperative model training exploiting computation capability of both the devices and the servers,  while achieving (a) a balanced computation load in the device layer and edge server layer; (b) an efficient data dispersion/routing from the device layer to the BS layer and from the BS layer to the edge server layer; and (iii) an efficient parameter aggregation via selecting a floating aggregation point.

\vspace{-3mm}
\subsection{Related Work}
\textbf{Conventional FedL.}  FedL has attracted tremendous attention from both ML and wireless networking communities. In the former literature~\cite{mohri2019agnostic,huang2021fl,li2021ditto,li2019convergence,haddadpour2019convergence}, fundamental convergence of FedL upon having non-iid data across the clients has been revealed. Also, a variety of new distributed learning techniques inspired by FedL have been invented, e.g., fully decentralized learning architectures~\cite{haddadpour2019convergence}. In the networking and systems literature~\cite{8664630,dinh2019federated,ang2020robust,shlezinger2020uveqfed,ozfatura2021time,yang2020federated,9488906,9163301}, researchers have been mostly studying the performance of FedL under communication and computation heterogeneity. For example, researchers have studied the model training performance under noisy channels~\cite{ang2020robust}, limited energy devices~\cite{8664630}, limited bandwidth~\cite{dinh2019federated}, quantization and sparsification~\cite{shlezinger2020uveqfed, ozfatura2021time}, and wireless aggregation of signals over the air~\cite{yang2020federated}. Also,  device sampling~\cite{9488906} and data sampling~\cite{9163301} has been topics of research. Furthermore, a part of literature focuses on adapting FedL for a variety of new technologies, such as unmanned aerial vehicles~\cite{9039589,wang2021uav}, intelligent reflecting surfaces~\cite{wang2021federated}, and massive MIMO~\cite{vu2020cell}. {\color{black} As compared to this literature, we consider a different network model focusing on the distribution of ML model training across the UE-BS-DC hierarchy.}

\textbf{New Network Architectures for FedL.} Some recent works promote migrating from FedL to new distributed ML architectures considering the characteristics of edge networks. In~\cite{hosseinalipour2020federated}, fog learning is proposed that incorporates the cooperation among the devices and multi-layer architecture of fog systems into ML. In~\cite{chen2020wireless}, collaborative FedL via device-to-device (D2D) communications for model relaying is introduced. In~\cite{lin2021two,hosseinalipour2020multi}, semi-decentralized/hybrid FedL is proposed, which augments the global aggregations of FedL with local aggregations conducted via D2D communications. In~\cite{nguyen2020self}, democratized FedL is studied to exploit the innate capability of heterogeneous devices to train an ML model. Finally, parallel successive learning has been developed in~\cite{psl} as a dynamic/online cooperative learning method that optimizes the utilization of D2D communications according to device heterogeneity. {\color{black} We contribute to this literature via proposing a novel ML network architecture consisting of multiple UEs, BSs, DCs, for which we study efficient network-element orchestration.}

\begin{figure}[t]
\includegraphics[width=.47\textwidth]{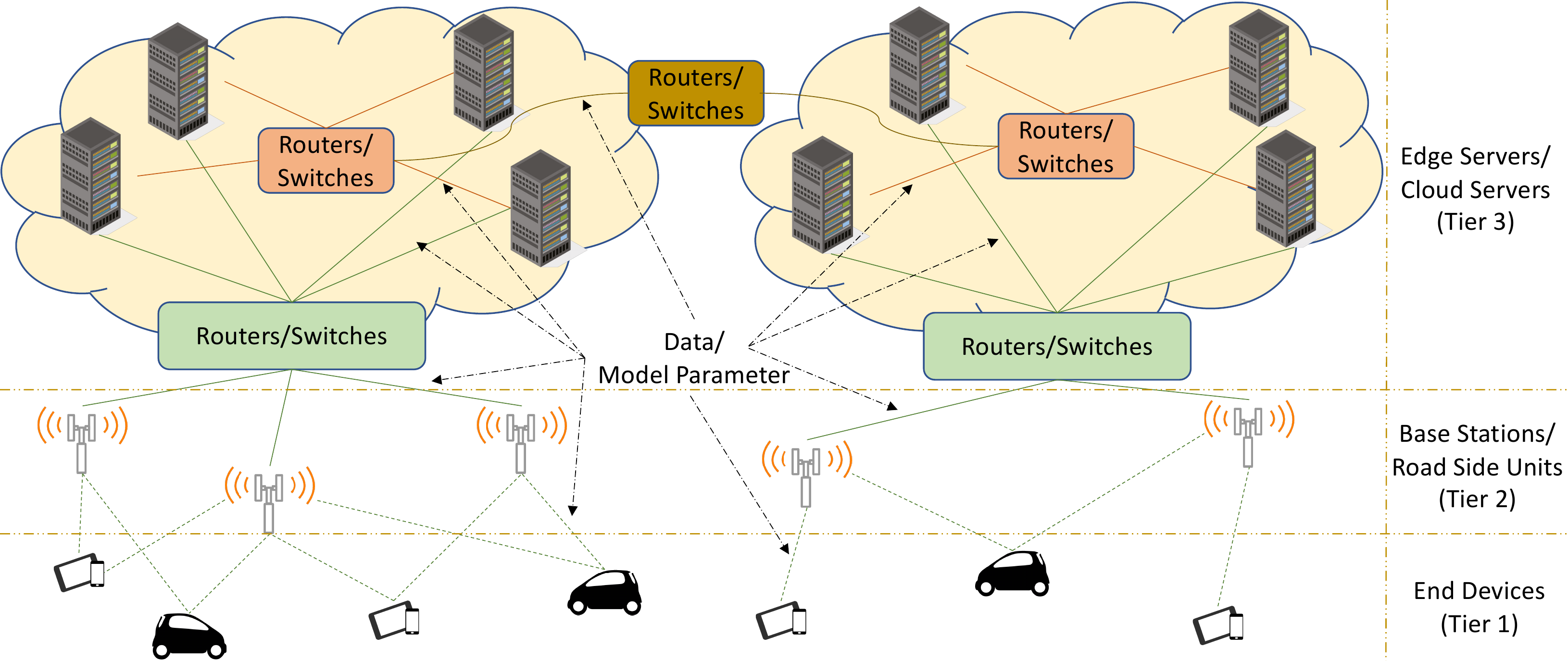}
\centering
\vspace{-1.2mm}
\caption{A schematic of our network model. We consider a three tiered network consisting of end devices, base stations, and edge servers.}
\label{fig:1}
\vspace{-1.2mm}
\end{figure}

\textbf{Hierarchical FedL.} One of the recent architectures of FedL has been hierarchical FedL~\cite{TobeAdded2,TobeAdded4,wang2020local}. The literature in this area is focused on managing the multi-stage model aggregations across the end device-to-cloud continuum. In this literature, researchers aim to reduce the latency of model training via conducting multiple rapid local aggregations at edge servers or lower layers of the network before conducting one global aggregation at the cloud server. However, they only use the edge/cloud servers as ``model aggregators" which do not conduct any data processing as in conventional FedL. 
This neglects the abundant computation capability of servers.

\textcolor{black}{\textbf{FedL with Data Offloading.} Although data sharing is not considered in the original FedL architecture~\cite{konevcny2016federated,mcmahan2017communication}, it has been promoted and investigated by the current literature. In the popular work~\cite{zhao2018federated}, a data sharing mechanism is deployed at the edge devices to mitigate the performance drop encountered as a consequence of non-i.i.d. local datasets.
In~\cite{hosseinalipour2020federated,9488906}, device-to-device (D2D) data offloading is leveraged to improve the efficiency of FedL over heterogeneous networks.
In~\cite{9492062}, a data offloading scheme is proposed for FedL to assist straggler vehicles through offloading their data to edge servers. Compared to these works, we build a novel distributed network orchestration scheme, in which the devices offload a portion of their data to  edge servers through  base stations. In particular, we consider joint data processing at the end devices and edge servers. This leads us to consider  joint load-balancing across the end devices and the edge servers to process the data for the ML task, and further consider  optimal data routing  to transfer data across the network hierarchy. We further formalize the floating aggregator notion into the learning architecture of FedL, where the aggregator server is carefully selected to optimize a trade-off between model performance, energy consumption, and delay. We also develop a distributed network element orchestration scheme to optimize ML model training. 
These make our system model and analysis unique and different from existing works.}



\vspace{-4mm}
 \subsection{Outline and Summary of Contributions}
 Our contributions can be summarizes as follows:
\begin{itemize}[leftmargin=4mm]
    \item We introduce and develop cooperative edge-assisted multi-edge server FedL ({\tt CE-FL}). We analytically characterize the convergence characteristics of {\tt CE-FL}, which leads to new convergence bounds in distributed ML.
    \item We model the processes taken through {\tt CE-FL} and investigate intrinsic heterogeneities existing at different network layers.  Our modeling gives insights on finding the optimized device and server orchestration scheme to train an ML model, which is carried out in conjunction with the floating aggregator selection, to optimize the performance of  {\tt CE-FL}. 
    \item We propose network-aware {\tt CE-FL} optimization problem  that captures the trade offs between
    ML performance, energy consumption, and delay.
    The optimization problem aims to configure the \textit{macro-decisions} across the network (e.g., data offloading and routing configuration across the network, and the floating aggregator), which leads to (i) \textit{load balancing} across the devices and servers, and (ii) optimized data routing across the network layers. It also obtains the \textit{micro-decisions}, which tunes each network element  (e.g., the CPU frequency cycles of the devices, number of local SGD iterations and SGD min-batch sizes at both end devices and edge servers). 
    \item We carefully investigate the characteristics of the optimization problem, which turns out to be a highly non-convex mixed-integer  problem. We propose a distributed network optimization solver with drawing a connection between two methods of surrogate function, used in successive convex optimization, and consensus-based optimization. We further study the optimality of our optimization solver. To the best of our knowledge, our developed solver is among the first decentralized device orchestration mechanisms in the area of network-aware distributed machine learning.
\end{itemize}
     
\textcolor{black}{In the following, we first describe the hierarchical network structure among the UEs, BSs, and DCs, as well as conducting distributed ML under a given network element orchestration scheme in Sec.~\ref{sec: sys_model}. In Sec.~\ref{sec:Conv}, we present our theoretical results on the ML model convergence behavior of {\tt CE-FL}. In Sec.~\ref{sec:NetAwareCEFL}, we aim to optimize the network orchestration scheme involving ML training, data and parameter offloading, and communication/computations overheads through formulating a network optimization problem. We then develop a  decentralized solver to optimize device orchestration for distributed ML in Sec.~\ref{sec:distSolv}. In Sec.~\ref{sec:NumExp}, we present our ablation study and proof-of-concept experiments to demonstrate how the tuned parameters manifested themselves in network costs and ML performance. Finally, in Sec.~\ref{sec:conclusion} we conclude the paper.}

\vspace{-2mm}
\section{System Model and Machine Learning Task} \label{sec: sys_model}
\noindent In this section, first, we describe the network structure of our interest (Sec.~\ref{subsec:netw}). Second, we introduce the dynamic/online model tracking problem (Sec.~\ref{subsec:Modeltrack}). Third, we provide an overview of our proposed cooperative distributed ML methodology (Sec.~\ref{subsec:CEFL}). Fourth, we  detail the ML model training (Sec.~\ref{subsec:ModelTrain}). We finally model communications, computations, and the floating aggregator (Sec.~\ref{subsec:commModel}).

\vspace{-3mm}
\subsection{Network Architecture}\label{subsec:netw}
We consider a hierarchical edge/fog computing network~\cite{6678114,hosseinalipour2020federated}, which consists of multiple user equipments (UEs), multiple base stations (BSs), and multiple edge server/data centers (DCs). A schematic of our network is depicted in Fig.~\ref{fig:1}. The UEs, BSs, and DCs are {collected via the sets} $\mathcal{N}$, $\mathcal{B}$, and $\mathcal{S}$, respectively. Each user UE $n\in\mathcal{N}$ can potentially upload/receive data to{/from} all of the BSs although connecting to some of the BSs may require prohibitively large transmit power. Also, each BS can potentially upload/receive data to/from all the DCs, although data transfers between some BS-DC pairs may impose large delay (e.g., when data needs to go through multiple switch/routers before reaching the destination). 

{\color{black}Our proposed framework can be considered as \textit{an interconnection between mobile edge computing \cite{6678114} and distributed ML}, where the resource limited UEs aim to exploit the abundant resource of DCs to conduct an ML task with a high efficiency.}

\vspace{-3mm} 
 \subsection{Dynamic/Online Model Tracking Problem}\label{subsec:Modeltrack}
 
  We consider a slotted time representation of the system dynamics, where $t\in\mathbb{N}\cup\{0\}$ denotes an arbitrary time index. In our network each $t$ is associated with one global model training and aggregation round which will be described later. We assume distributed model training for an ML task, where the data is collected at the UEs at the bottom layer of our network hierarchy.  At each time instant $t$, let $\overline{{\mathcal{D}}}^{(t)}_{{n}}$ denote the dataset at UE ${n} \in \mathcal{N}$ consisting of $\overline{{{D}}}^{(t)}_{{n}} \triangleq |\overline{{\mathcal{D}}}^{(t)}_{{n}}|$ data points. In contrast to most of the existing literature on federated learning, we consider a \textit{dynamic environment}, in which the number of data points and subsequently the distribution of data across the clients may vary over time~\cite{psl}. 
     Each data point $\xi\in\overline{\mathcal{D}}_n^{(t)}$, $\forall n$, is associated with a feature vector, e.g., RGB colors of a picture, and a label, e.g., the weather condition in the picture.

    Let $f(\mathbf{x};\xi)$ denote the loss of the ML model (e.g., a neural network classifier) on data point $\xi$  \textit{parameterized} by model parameter $\mathbf{x}\in\mathbb{R}^P$, where $P$ denotes the dimension of the model parameter. At each time instant $t$, for a model parameter $\mathbf{x}$, each device $n\in\mathcal{N}$ is associated with a \textit{local loss} function
    \vspace{-1mm}
   \begin{align}\label{eq:localLossInit}
    \overline{F}_{n}^{(t)}(\mathbf{x}) =  \frac{1}{{\overline{D}}_{{n}}^{(t)}}{\underset{\xi \in {\overline{\mathcal{D}}}_{{n}}^{(t)}}{\sum} {f}(\mathbf{x} ; \xi)}.
\end{align}
    The instantaneous \textit{global loss} function of the system is
    \vspace{-1mm}
    \begin{align}\label{eq:globalLossinit}
    F^{(t)}(\mathbf{x}) = \frac{1}{D^{(t)}}\displaystyle \sum_{n \in \mathcal{N}} {{{D}}_{n}^{(t)}} \overline{F}_{n}^{(t)}(\mathbf{x}),
\end{align}
    where {\small ${D}^{(t)}=\sum_{n \in \mathcal{N}} \overline{D}_{{n}}^{(t)}$} is the total number of data points.
    
     The goal of the \textit{dynamic/online model tracking} is to minimize the instantaneous global loss function, i.e., at each time instant $t$, it aims to find the model parameter $\mathbf{x}^{(t)^\star}$ such that
     \begin{equation}\label{eq:GlobLoss}
         \mathbf{x}^{(t)^\star}= \argmin_{\mathbf{x}\in\mathbb{R}^P}F^{(t)}(\mathbf{x}).
     \end{equation}
     Let $T$ denote the duration of ML model training and $[T]\triangleq\{0,...,T-1\}$.
     It is evident that the sequence $\{\mathbf{x}^{(t)^\star}\}_{t\in[T]}$ is a function of devices local dataset dynamics.
{\color{black} In our ML modeling, we quantify dynamics of data distributions at the devices via introducing \textit{model/concept drift} in Definition ~\ref{defn: model drift} in~Sec.~\ref{sec:Conv}.  \textit{Model/concept drift} draws a mathematical connection between the ML model performance and ML model training delay. 

The ML training formulation described by~\eqref{eq:GlobLoss} aims to obtain optimal instantaneous models
for conducting downstream tasks in real-time suitable for  time-varying datasets at the
devices~\cite{rizk2020dynamic}. We will later incorporate the performance of online model training obtained in Sec.~\ref{sec:Conv} into our network element orchestration formulation in Sec.~\ref{sec:NetAwareCEFL}, wherein the inherent trade-off
between instantaneous ML performance, energy consumption, and delay is optimized.
}




\vspace{-3mm}

\subsection{{\tt CE-FL} Overview}\label{subsec:CEFL}
  We introduce \textit{cooperative edge-assisted federated learning} ({\tt CE-FL}) to conduct distributed ML over our hierarchical network. 
To solve~\eqref{eq:GlobLoss}, in {\tt CE-FL}, part of the dataset of the UEs are offloaded to the DCs. {\tt CE-FL} then simultaneously exploits the computation resources at the UEs (bottom layer) and the DCs (top layer) for ML model training, while the BSs (middle layer) are used for data and model parameter relaying. In particular, {\tt CE-FL}  encompasses four processes: (i) UEs data offloading to BSs, (ii) BSs dispersing data among DCs, (iii) UEs and DCs conducting model training, and (iv) model parameter transfer/pulling and aggregation at a floating aggregation DC, which are describe  below.
  
  \begin{enumerate}[leftmargin=4mm]
 \item{\textit{Data offloading from UEs to
 BSs}}: At time instant $t$, each UE offloads/disperses part of its generated data across a subset of BSs. After the data offloading process, we denote the collected dataset at the each BS $b\in\mathcal{B}$ via ${{\mathcal{D}}}_b^{(t)}$ and the remaining dataset at each UE $n\in\mathcal{N}$  by ${{\mathcal{D}}}_n^{(t)}\subseteq {\overline{\mathcal{D}}}_n^{(t)}$.
 \item{\textit{Data relaying from BSs to
 the DCs}}: Each BS offloads/disperses its collected data among a subset of DCs. Since BSs do not have computation power, no data is kept at the BSs. We use $\mathcal{D}_{s}^{(t)}$ to denote the dataset collected at DC ${s} \in \mathcal{S}$.
 \item{\textit{ML model training at UEs and DCs}}: Each UE executes ML model training on the its remaining dataset. Also, ML model training begins at the DCs once BSs relay their data.
 \item{\textit{Parameter transfer across the network and aggregation at a floating point}:}
 After model training ends, the local ML models of  UEs and DCs are aggregated at a \textit{floating} DC, which varies from one model training round to another. The aggregator is adaptively chosen, while taking into account energy and delay of parameter transfer across the network. 
 \end{enumerate}
 
We refer to the set of UEs and DCs as \textit{data processing units} (DPU) and use index $i \in \mathcal{N} \cup \mathcal{S}$ to refer to an arbitrary UE/DC. Next, we will carefully model and formulate the four processes outlined above. We first describe the local model training for a given dataset at the UEs and DCs and then describe how the data and parameter offloading and relaying are carried out.

\vspace{-3mm}
\subsection{Distributed Model Training in~{\tt CE-FL}}\label{subsec:ModelTrain}
As described above, in {\tt CE-FL}, parts of the local datasets of the UEs gets transferred to the DCs through BSs and DCs possess a collected dataset at each time (model training round) $t$. We will describe how the data gets transferred in Sec.~\ref{subsec:commModel} and obtain the optimal data configuration and routing across the network in Sec.~\ref{sec:NetAwareCEFL}. In the following, we describe the procedure carried out across the UEs and DCs to conduct ML model training for an arbitrary data configuration across them.  

To capture the performance of the ML model across UEs and DCs, we describe the \textit{local ML loss} at each DPU (i.e., a UE or a DC) $i\in\mathcal{N}\cup\mathcal{S}$ at time $t$ as\footnote{As compared to~\eqref{eq:localLossInit}, the loss in~\eqref{eq:localLossAfter} is defined on the actual dataset under which the ML model training is carried out in the index of the summation.}
 \begin{align}\label{eq:localLossAfter}
    {F}_{i}^{(t)}(\mathbf{x}) =  \frac{1}{{{D}}_{{i}}^{(t)}}{\underset{\xi \in {{\mathcal{D}}}_{{i}}^{(t)}}{\sum} {f}(\mathbf{x} ; \xi)}.
\end{align}
    Since~\eqref{eq:globalLossinit} measures the ML loss per \textit{data point} and data points do not get generated and lost during the data transfer stage of {\tt CE-FL}, it is straightforward to verify that $F^{(t)}(\mathbf{x})$ in~\eqref{eq:globalLossinit} can be equivalently written as
 $
    F^{(t)}(\mathbf{x}) = \frac{1}{D^{(t)}}\displaystyle \sum_{i \in \mathcal{N}\cup\mathcal{S}} {{\mathcal{D}}_{i}^{(t)}} \overline{F}_{i}^{(t)}(\mathbf{x})
$.




To solve~\eqref{eq:GlobLoss}, {\tt CE-FL} orchestrates both UEs and DCs in local ML model training, where each DPU $i$ aims to minimize its local loss $F^{(t)}_i(\mathbf{x})$ given by~\eqref{eq:localLossAfter}  via \textit{local ML model training}.
To conduct local ML model training we exploit the FedProx method~\cite{li2018federated}, which consists of a series of local stochastic gradient descent (SGD) iterations at each DPU $i$ on the \textit{regularized local loss} using its local dataset $\mathcal{D}_i^{(t)}$. FedProx is shown to be particularly effective for handling non-iid data across the clients which is the case in our system of interest. 
To cope with the communication/computation heterogeneities of network elements,~{\tt CE-FL} uses FedProx with \textit{different} number of local SGD iterations and mini-batch sizes across the network elements, which is among the first in literature.

In~{\tt CE-FL}, the ML model training starts with an initial model broadcast $\mathbf{x}^{(0)}$ from a predetermined DC across all the DPUs. Each subsequent global model training round $t\geq 1$ starts with broadcasting of a global model $\mathbf{x}^{(t)}$ from the respective elected floating aggregation DC.
Given $\mathbf{x}^{(t)}$, each DPU $i$ first initializes its local model parameter as $\mathbf{x}_i^{(t,0)}=\mathbf{x}^{(t)}$ and then conducts $\gamma_i^{(t)}\in\mathbb{N}$ local descent updates on its regularized local loss $g_i(\mathbf{x},\mathbf{x}^{(t)})= F_i^{(t)}(\mathbf{x})+\frac{\mu}{2}\Vert \mathbf{x}-\mathbf{x}^{(t)} \Vert^2$, which ensures the closeness of the local model to the global model $\mathbf{x}^{(t)}$ controlled by regularization parameter $\mu$. The evolution of the local model of each DPU $i$ during the local descent iterations is given by
\begin{align}\label{eq:locMod}
   \hspace{-2mm} \mathbf{x}^{(t,k)}_{i}= \mathbf{x}_i^{(t, k-1)}  \hspace{-1mm}- \eta \widetilde{\nabla} g_i(\mathbf{x}_i^{(t, k-1)} \hspace{-.5mm},\mathbf{x}^{(t)}),~k=1,\cdots,\gamma_i^{(t)} \hspace{-1mm}, \hspace{-3mm}
\end{align}
where $\eta$ denotes the step size and $\widetilde{\nabla} g_i(\mathbf{x}_i^{(t, k-1)},\mathbf{x}^{(t)})$ is the \textit{stochastic} gradient of the regularized local loss given by
\begin{equation}\label{eq:StochGrad}
   \hspace{-2mm} \widetilde{\nabla}{g}_i(\mathbf{x}_i^{(t, k-1)}\hspace{-.5mm},\mathbf{x}^{(t)})\hspace{-.5mm}=\hspace{-.5mm} \widetilde{\nabla} F_i^{(t)}(\mathbf{x}_i^{(t, k-1)})\hspace{-.5mm}+\hspace{-.5mm}{\mu}(\mathbf{x}_i^{(t, k-1)}\hspace{-.5mm}-\mathbf{x}^{(t)}).
      \hspace{-2mm}
\end{equation}
The stochastic gradient of local loss function $\widetilde{\nabla} F_i^{(t)}(\mathbf{x}_i^{(t, k-1)})$ is obtained via sampling a mini-batch (i.e., collection)  of data points ${\mathcal{D}}_i^{(t,k-1)}\subseteq {\mathcal{D}}_i^{(t)}$ with size ${{D}}_i^{(t,k-1)}=m_i^{(t)}D_i^{(t)}$, where $m_i^{(t)}\in(0,1]$ denotes the mini-batch fraction, as follows:
\begin{equation}
\hspace{-3mm}
\widetilde{\nabla} F_i^{(t)}(\mathbf{x}_i^{(t, k-1)})= \frac{1}{m_i^{(t)}D_i^{(t)}} \sum_{\xi\in {\mathcal{D}}_i^{(t,k-1)}} \nabla f(\mathbf{x}_i^{(t, k-1)};\xi).\hspace{-3mm}
    \end{equation}

Replacing~\eqref{eq:StochGrad} in~\eqref{eq:locMod} and recursively expanding the result implies the following relationship between the instantaneous local model and the initial local model at each DPU $i$:
\begin{equation}\label{localEvolution}
   \hspace{-3mm}\mathbf{x}_{i}^{(t, k)} \hspace{-.5mm} - \hspace{-.5mm}\mathbf{x}^{(t)} \hspace{-.5mm}=\hspace{-.5mm} -\eta \displaystyle \sum_{\ell = 0}^{k - 1} a_{{i,\ell}}^{(t)} \widetilde{\nabla} F^{(t)}_{i}(\mathbf{x}_{i}^{(t, \ell)}),\hspace{-3.5mm}
\end{equation}
where $a_{{i,\ell}}^{(t)} = (1-\eta\mu)^{\gamma_{i}^{(t)}-1-\ell}$. We further define $\mathbf{a}_{i}^{(t)} = \big[ a_{{i,0}}^{(t)}, \cdots , a_{{i,\gamma_{i}^{(t)} - 1}}^{(t)}\big]$.


At the end of local model training period, each DPU $i$ computes its accumulated gradient, which using~\eqref{localEvolution}, can be obtained based on the difference between the final local model parameter and the initial model parameter as
\begin{equation}
 \sum_{\ell = 0}^{\gamma_{i}^{(t)} - 1} a_{{i,\ell}}^{(t)} \widetilde{\nabla} F^{(t)}_{i}(\mathbf{x}_{i}^{(t, \ell)})= \big({\mathbf{x}^{(t)}-\mathbf{x}_{i}^{(t, \gamma_{i}^{(t)})}}\big)/{\eta },
\end{equation}
 and subsequently obtains its \textit{normalized} accumulated gradient
\begin{equation}\label{eq:normalGrad}
   \mathbf{d}_{{i}}^{({t})} = \frac{1}{\|\mathbf{a}_{i}^{(t)}\|_{1}} \sum_{\ell = 0}^{\gamma_{i}^{(t)} - 1} a_{{i,\ell}}^{(t)} \widetilde{\nabla} F^{(t)}_{i}(\mathbf{x}_{i}^{(t, \ell)}).
\end{equation}

 {\color{black} The normalization is a necessity to ensure that that the global ML model is not biased towards the DPUs that conduct more local SGD iterations~\cite{wang2020tackling}. Since DPUs possess different number of datapoints,
each DPU $i$ sends vector ${{D}}_{{i}}^{(t)}\mathbf{d}_{{i}}^{({t})}$, called \textit{scaled accumulated gradient}, to the selected floating aggregation server.  We assume that each BS can receive scaled accumulated gradients of multiple UEs,\footnote{Our modeling in Sec.~\ref{subsec:commModel} ensures that each UE is associated with one BS during scaled accumulated gradient transfer to avoid reception of multiple replicas of the gradient of the same UE at the aggregation server.} in which case it sums all the received vectors together to keep the dimension of the transmitted vector to the aggregation server fixed (e.g., see Fig.~3 of~\cite{hosseinalipour2020federated}).
The aggregation DC finally obtains the next global model parameter via first summing all the received vectors together to obtain $\sum_{i \in \mathcal{N} \cup \mathcal{S}} {  {{D}}_{{i}}^{(t)} } \mathbf{d}_i^{(t)}$ and then scaling the resulting vector to carry out the update
\begin{align}\label{eq:11M}
    \mathbf{x}^{(t + 1)} 
    & = \mathbf{x}^{(t)}-\vartheta \eta \frac{1}{ {D}^{(t)}} \displaystyle \sum_{i \in \mathcal{N} \cup \mathcal{S}}  {  {{D}}_{{i}}^{(t)} } \mathbf{d}_i^{(t)} ,
\end{align}
where $\vartheta$ is a scaling factor introduced to compensate for the normalization introduced in~\eqref{eq:normalGrad}.}

{\color{black}\begin{remark}
In this work, we use a single global model for the downstream task at the devices, which is a common approach in FedL literature~\cite{li2019convergence,haddadpour2019convergence,8664630,dinh2019federated}. It is worth mentioning that personalized FedL, which aims to learn a global ML model from which local models are obtained specific to the local distribution of data at the  devices~\cite{fallah2020personalized}, would be another potential approach. We leave the extension of the framework of \texttt{CE-FL} to support personalized FedL as future work.
\end{remark}}
\vspace{-3mm}

\subsection{Communication, Computation, and Floating Aggregation}\label{subsec:commModel}
We next describe the models for data/parameter communications, SGD local computations, and floating aggregation.
\subsubsection{UEs-BSs Communications}
UE-to-BS communications are carried out through wireless channels. For the uplink communications, we consider that the data rate between UE $n\in\mathcal{N}$ and BS $b\in\mathcal{B}$ at time instant $t$ is given by\footnote{We have considered the operation of the transmissions }
\begin{equation}\label{eqn:data-rate 1}
    R_{{{n}},{b}}^{(t)} = V_{{{n}},{b}}^{(t)}\log_2 \left(1+{P_{{{n}},{b}}^{(t)} |h_{{{n}},{b}}^{(t)}|^2}\big/{{(\sigma_{{{n}},{b}}^{(t)})^2}} \right),
\end{equation}
where $V_{{{n}},{b}}^{(t)}$ is the bandwidth, { $(\sigma_{{{n}},{b}}^{(t)})^2=N_0 V_{{{n}},{b}}^{(t)}$ is the noise power with $N_0$ denoting the noise spectral density}, $P_{{{n}},{b}}^{(t)}$ is the transmit power, and $h_{{{n}},{b}}^{(t)}$ is the channel gain between the respective UE-BS pair. {\color{black} We have assumed multiple access techniques such as FDMA~\cite{7110419}, where users utilize non-overlapping bandwidth.}

{\color{black} \begin{remark}
Given the focus of this study, we have chosen to ignore both inter-cell and intra-cell interference similar to works on computation offloading in multi-BS, multi-DC mobile edge computing environments \cite{6678114}, which is also a common assumption in a part of FedL literature focused on MAC layer design~\cite{dinh2019federated}. Consequently, we leave further investigation of interference avoidance techniques and adaptation of technologies such as coordinated multi-point
(CoMP)~\cite{qamar2017comprehensive} and the 3-D beam pattern of a eNB/gNB to our framework as future work.
\end{remark}}

 Similarly, for the downlink, the data rate between BS $b$ and UE $n$ is given by
\begin{equation}
R_{{{b}},{n}}^{(t)}=
    V_{b}^{(t)}\log_2 \left(1+{P_{{b}}^{(t)} |h_{{{b}},{n}}^{(t)}|^2}\big/{({\sigma_{{b}}^{(t)}})^2} \right),   
\end{equation}
 where  $V_{b}^{(t)}$ denotes the downlink channel bandwidth, ${P_{{b}}}^{(t)}$ is the BS $b$ transmit power, $h_{{{b}},{n}}^{(t)}$ is the channel gain between the respective BS and UE, and $(\sigma_{{b}}^{(t)})^2=N_0 V_{{b}}^{(t)}$ is the noise power. Note that the BS-to-UEs communications are performed in \textit{broadcast} mode and thus the BS does not use different transmit powers to transmit signals to different UEs.
\subsubsection{BSs-DCs Communications}
BS-to-DC communications are mostly enabled via wireline communications. 
{\color{black}We consider the deployed data rate $R_{ {{b}}, {s}}^{(t)}$ for communication from BS ${b} \in \mathcal{B}$ to DC ${s} \in \mathcal{S}$, which is later optimized.  Furthermore, for the uplink communication channel between BSs-DCs we denote the maximum rate over the link between an arbitrary $(b,s)$ pair as $R_{b,s}^{\mathsf{max}}$. Also, the maximum capacity, in terms of the arriving data rate, of an arbitrary server is denoted by $R_{s}^ {\mathsf{max}}$. Therefore, we have the following constraints pertaining to the uplink rates at global aggregation round $t$:
\begin{align}
    &  \label{eqn: BS_capacity} 0 \leq R_{ {{b}}, {s}}^{(t)} \leq R_{b,s}^{\mathsf{max}}, ~\forall b \in \mathcal{B}, ~\forall s \in \mathcal{S}, ~\forall{t} \in [T], \hspace{10mm}\\
    &  \label{eqn: Server_capacity} \underset{b \in \mathcal{B}}{\sum} R_{ {{b}}, {s}}^{(t)} \leq R_{s}^{\mathsf{max}} ~\forall s \in \mathcal{S}, ~\forall{t} \in [T]. \hspace{28mm} 
\end{align}}

We further let $R_{ {{s}}, {b}}^{(t)}$  denote the data rate for communication from DC ${s} \in \mathcal{S}$ to BS ${b} \in \mathcal{B}$ in downlink. {\color{black}Since the downlink is only used for the broadcast of ML model parameter, the capacity constraints are not considered.}



\subsubsection{DCs-DCs Communications} Similar to BS-to-DC communications, DC-to-DC information exchange mostly occur over wirelines.  We let $R_{ {{s}}, {s'}}^{(t)}$ denote the data rate between DC $s$ and DC $s'$ at time instant $t$, which is time varying due to the congestion over the links. In general, $R_{ {{s}}, {s'}}^{(t)}\neq R_{ {{s'}}, {s}}^{(t)}$, $\forall s,s'\in\mathcal{S}$. {\color{black}Similarly, since DC-to-DC communications are only used for ML parameter aggregations and broadcasting, the capacity constraints are not considered.}

\subsubsection{Data Configuration at the UE Devices, BSs, and DCs} To capture the data offloading at each UE $n$,  we introduce a continues variable $\rho_{{n}, {b}}^{(t)}\in[0,1]$ to denote the fraction of dataset of that UE which is offloaded to BS $b$ at time $t$, where $ \sum_{{b} \in \mathcal{B}} \rho_{{n}, {b}}^{(t)} \leq 1$, $\forall n\in\mathcal{N}$ since some part of the dataset might be kept local. Given the initial dataset at UE $n$ (i.e., $\overline{\mathcal{D}}_n^{(t)}$ with size $\overline{{D}}_n^{(t)}$) and the  remaining dataset at UE $n$ (i.e., ${\mathcal{D}}_n^{(t)}\subseteq \overline{\mathcal{D}}_n^{(t)}$ with size ${{D}}_n^{(t)}\triangleq |{\mathcal{D}}_n^{(t)}|$), let ${\mathcal{D}}_{ {n}, {b}}^{(t)}\subseteq {\mathcal{D}}_n^{(t)}$ with size ${{D}}_{ {n}, {b}}^{(t)}\triangleq |{\mathcal{D}}_{ {n}, {b}}^{(t)}|$ denote the offloaded dataset from UE $n$ to BS $b$, where 
\vspace{-2mm}
\begin{align}
        {{D}}_{{n},{b}}^{(t)} =  {\overline{D}}_{{n}}^{(t)} \rho_{{n}, {b}}^{(t)}, ~~~
    {{D}}_{{n}}^{(t)} = \left(1- \sum_{{b} \in \mathcal{B}}\rho_{{n}, {b}}^{(t)}\right){\overline{D}}_{{n}}^{(t)}. \label{eq:dataRemainAtDevice}
\end{align}
\vspace{-3.2mm}

\noindent The dataset collected at each BS ${b}\in\mathcal{B}$ denoted by ${\mathcal{D}}_{{b}}^{(t)}$ is the union of the received data points from UEs with size
\begin{equation}
    {{D}}_{{b}}^{(t)} \triangleq |{\mathcal{D}}_{{b}}^{(t)}|= \sum_{n\in\mathcal{N}} \rho_{n,b}^{(t)} \overline{D}_n^{(t)}.
\end{equation}
The BS then disperses all of its collected data across a subset of DCs since it is assumed to have no data processing power. At time $t$, let $\rho_{{{b}},{s}}^{(t)}\in[0,1]$ denote the fraction of dataset of BS $b\in\mathcal{B}$ offloaded to  DC ${s} \in \mathcal{S}$, where $\underset{{s} \in \mathcal{S}}{\sum} \rho_{{{b}},{s}}^{(t)} = 1$. Let $\mathcal{D}_{b,s}^{(t)}\subseteq {\mathcal{D}}_{{{b}}}^{(t)}$ denote the portion of dataset transferred from BS $b$ to DC $s$ and $\mathcal{D}_s^{(t)}$ the dataset collected at DC $s$. We have
\begin{align}
   \hspace{-3mm} {{D}}_{{{b}},{s}}^{(t)} \triangleq |\mathcal{D}_{b,s}^{(t)}|= \rho_{{{b}},{s}}^{(t)} {{D}}_{{{b}}}^{(t)} 
   ,~
    {D}_s^{(t)} \triangleq |\mathcal{D}_s^{(t)}|= \sum_{b\in\mathcal{B}}\rho_{{{b}},{s}}^{(t)} {{D}}_{{{b}}}^{(t)}.\label{eq:DataCollecteatServer}
    \hspace{-3mm} 
\end{align}

Each DPU $i\in\mathcal{N}\cup\mathcal{S}$ (i.e., a UE or a DC) subsequently uses its local dataset ${\mathcal{D}}_{{i}}^{(t)}$ (i.e., ${\mathcal{D}}_{{n}}^{(t)}$ in case of UE $n$ and ${\mathcal{D}}_{{s}}^{(t)}$ in case of DC $s$) to conduct ML model training. 

{\color{black}
\begin{remark}
We note that some recent works have  introduced GANs in federated
learning for a variety of applications ~\cite{FedLGAN,FedLGAN2}. However, GAN architectures, given their long training duration, do not produce favorable results upon being implemented for time-varying data distributions. Consequently, since we consider dynamic/time-varying data distributions at the devices, it is not straightforward to replace the data offloading with GAN architectures at the DCs to reproduce the data at the UEs. Nevertheless, designing of GAN models that require short training time at the DCs and their effectiveness on reconstruction of online data at the devices are interesting future directions.
\end{remark}}
\subsubsection{Energy Consumption and Delay Modeling}
Energy consumption and delay are incurred by two mechanisms: (i) data/parameter transfer across the network, and (ii) data processing for ML model training across the network. We model these two mechanisms in the following.

\textbf{Data and Gradient Transfer Across the Network.} Let ${\beta}^{\mathsf{D}}$ and ${\beta}^{\mathsf{M}}$ denote the number of bits used to represent a data point and vector of scaled accumulated gradient, respectively. The delay of data and gradient transfer from UE $n$ to BS $b$ denoted by $\delta^{\mathsf{D},(t)}_{n,b}$ and $\delta^{\mathsf{M},(t)}_{n,b}$, respectively, are 
\begin{align}
    &\delta^{\mathsf{D},(t)}_{n,b}= \beta^{\mathsf{D}}\overline{D}_{n}^{(t)}\rho_{n,b}^{(t)}/R_{n,b}^{(t)},  ~~\delta^{\mathsf{M},(t)}_{n,b}= \beta^{\mathsf{M}}/R_{n,b}^{(t)}.
\end{align}
Also, the energy consumption of data and model transfer from UE $n$ to BS $b$ denoted by $E^{\mathsf{D}}_{n,b}$ and $E^{\mathsf{M}}_{n,b}$, respectively, are
\begin{align}
    E_{{n,b}}^{\mathsf{D},(t)} =  \delta^{\mathsf{D},(t)}_{n,b} P_{{{n}},{b}}^{(t)},~~
    E_{{n,b}}^{\mathsf{M},(t)} =  \delta^{\mathsf{M},(t)}_{n,b} P_{{{n}},{b}}^{(t)}\label{energy_data}.
\end{align}


Similarly, the delay of data and parameter transfer from BS $b$ to DC $s$ denoted by $\delta^{\mathsf{D},(t)}_{b,s}$ and $\delta^{\mathsf{M},(t)}_{b,s}$, respectively, are 
\begin{align}
    &\delta^{\mathsf{D},(t)}_{b,s}= \beta^{\mathsf{D}}D_{b}^{(t)}\rho_{b,s}^{(t)}/R_{b,s}^{(t)}, ~~
    \delta^{\mathsf{M},(t)}_{b,s}= \beta^{\mathsf{M}}/R_{b,s}^{(t)}.
\end{align}
Thus, assuming that the data transfer from the BSs to the DCs starts after reception of all the datapoints at all the BSs,\footnote{This assumption is imposed to make the formulation more tractable.} the delay of data collection at each DC $s\in\mathcal{S}$ is given by
\begin{equation}\label{eq:delayAtDCdataCollect}
     \delta^{\mathsf{D}, (t)}_s= \max_{b\in\mathcal{B}}\{\delta_{b,s}^{\mathsf{D},(t)}\}+\max_{n\in\mathcal{N},b\in\mathcal{B}}\{\delta_{n,b}^{\mathsf{D},(t)} \}.
\end{equation}
The energy consumption of data and parameter transfer from BS $b$ to DC $s$ denoted by $E^{\mathsf{D},(t)}_{b,s}$ and $E^{\mathsf{M},(t)}_{b,s}$, respectively, are
\begin{align}
    &E^{\mathsf{D},(t)}_{b,s}= \delta^{\mathsf{D},(t)}_{b,s}P_{b,s}^{(t)},~~
    E^{\mathsf{M},(t)}_{b,s}= \delta^{\mathsf{M},(t)}_{b,s}P_{b,s}^{(t)},\label{energy_dataBS} 
\end{align}
where $P_{b,s}^{(t)}$ is the transmit power consumption (during one second period) over the outgoing
link from BS $b$ to DC $s$.

Communications between two DCs are only conducted for parameter transfer and aggregation in our system. The delay and energy associated with parameter transfer between two DCs $s,s'\in\mathcal{S}$ denoted by $\delta^{\mathsf{M},(t)}_{s,s'}$ and $E^{\mathsf{M},(t)}_{s,s'}$, respectively, are
\vspace{-1.5mm}
\begin{align}
    &\delta^{\mathsf{M},(t)}_{s,s'}= \beta^{\mathsf{M}}/R_{s,s'}^{(t)},~~E^{\mathsf{M},(t)}_{s,s'}= \delta^{\mathsf{M},(t)}_{s,s'} P_{s,s'}^{(t)},
\end{align}
    \vspace{-5mm}
    
\noindent where $P_{s,s'}^{(t)}$ is the transmit power consumption over the outgoing
wirelines from DC $s$ to DC $s'$ ($\delta_{s,{s'}}^{\mathsf{M},(t)}=0$ if $s=s'$.).

\begin{table*}[t]
\vspace{-6mm}
\begin{minipage}{0.99\textwidth}
{\footnotesize
\begin{align}
    & \hspace{-6mm} \frac{1}{T} \sum_{t = 0}^{T-1} \mathbb{E} \bigg[  \big\|\nabla {{F}}^{(t)}(\mathbf{x}^{(t)})\big\|^2 \bigg] \leq  \underbrace{\frac{4(F^{(0)}(\mathbf{x}^{0})- F^{\star})}{\vartheta\eta T}}_{(a)} + \frac{4}{\vartheta\eta T}\underbrace{ \sum_{t=0}^{T-1}  \sum_{i \in \mathcal{N} \cup \mathcal{S}} \tau^{(t)} \Delta_{i}^{(t)}}_{(b)}+ 16 \eta L \vartheta \Bigg[\underbrace{\frac{1}{T} \sum_{t = 0}^{T-1}\underset{i \in \mathcal{N} \cup \mathcal{S}}{\sum}  \frac{(p_{i}^{(t)})^2 (1-m_i^{(t)})\Theta_i (\Tilde{\sigma}_{i}^{(t)})^2}{m_i^{(t)} {{D}}_{i}^{(t)}}  \frac{\|\mathbf{a}_{i}^{(t)}\|_{2}^2}{\|\mathbf{a}_{i}^{(t)}\|_{1}^2}}_{(c)}  \Bigg]  \nonumber \\[-.34em]
    &\hspace{-6mm} +\hspace{-.5mm}  12 \eta^2 L^2 \zeta_2\hspace{-.5mm} \Bigg(\hspace{-.5mm}\underbrace{\underset{t \in [T]}{\max} \ \underset{i \in \mathcal{N} \cup \mathcal{S}}{\max} \frac{ (\gamma_i^{(t)})^2 (\|\mathbf{a}_{i}^{(t)}\|_{1} - [a_{i,-1}^{(t)}])}{\|\mathbf{a}_{i}^{(t)}\|_{1}}}_{(d)} \hspace{-.5mm}\Bigg)
    \hspace{-.5mm}+\hspace{-.5mm} {12 \eta^2 L^2} \hspace{-.5mm}\Bigg[\hspace{-.5mm}\underbrace{\frac{1}{T} \sum_{t = 0}^{T-1}\hspace{-.5mm} \underset{i \in \mathcal{N} \cup \mathcal{S}}{\sum}\hspace{-.5mm} \frac{ (1-m_i^{(t)})({{D}}_{i}^{(t)} -1) \Theta_i (\Tilde{\sigma}_{i}^{(t)})^2 p_i^{(t)} \gamma_i^{(t)}} {m_i^{(t)}  \|\mathbf{a}_{i}^{(t)}\|_{1} ({{D}}_{i}^{(t)})^2} (\|\mathbf{a}_{i}^{(t)}\|_{2}^2 - [a_{i,-1}^{(t)}]^2)}_{(e)}\hspace{-.5mm} \Bigg] \hspace{-6mm}  \label{eq:GenConvBound}
\end{align}
}
\end{minipage}
\hrule
\vspace{-5.5mm}
\end{table*}
\textbf{Data Processing Across the Network.} 
For UE $n$, let $f_{{n}}^{(t)} \in [f_{{n}}^{\mathsf{min}},f_{{n}}^{\mathsf{max}}]$ denote the CPU frequency used to process the datapoints at time $t$ and
    $c_{{n}}$ denote the number of required CPU
cycles to process one data point. Then, the delay and energy consumption of data processing at UE $n$ are~\cite{dinh2019federated}
\vspace{-1mm}
    \begin{align}
       & \delta_{{n}}^{\mathsf{P},(t)} = c_{{n}} {\gamma_{{n}}^{(t)} m_{{n}}^{(t)} {{D}}_{{n}}^{( t)}}/{f_{{n}}^{(t)}},\label{delay_proc}\\
        & E_{{n}}^{\mathsf{P},(t)} =  c_{{n}} \gamma_{{n}}^{(t)} m_{{n}}^{(t)} {\mathcal{D}}_{{n}}^{(t)} (f_{{n}}^{(t)})^{2}{\alpha_{{n}}}/{2}.\label{energy_proc}
    \end{align}
    \vspace{-4.5mm}
    
   \noindent  In \eqref{delay_proc} and \eqref{energy_proc}, $\gamma_{{n}}^{(t)}$ is the number of SGD iterations at the device, $m_{{n}}^{(t)}\in[0,1]$ is the mini-batch ratio (i.e., the fraction of ${{D}}_{{n}}^{(t)}$ data points processed at the device in each SGD iterations), and ${\alpha_{{n}}}/{2}$ is the device effective chip-set capacitance.


    
    Also, for each DC $s$, we adopt the DC data processing model~\cite{7775032,8847383}, where each DC $s$ consists of $M_s$ identical server machines. Each machine is assumed to operate  at the speed of processing $z_{{s}}^{(t)}$ data points per second. Let $C_{{s}}$ denote the processing capacity of the machines at DC ${s}$ (i.e., $z_{{s}}^{(t)}\leq C_{{s}}$). The delay and energy of data processing at DC $s$ are
    \vspace{-2mm}
            \begin{align}
         & \hspace{-4mm}\delta_{{s}}^{\mathsf{P},(t)} = \frac{\gamma_{s}^{(t)}m_{{s}}^{(t)} {{D}}_{{s}}^{(t)}}{M_s z_{{s}}^{(t)}} \label{eqn: Data Center speed-delay relation-1},
            \end{align}
            \vspace{-4mm}
             \begin{align}
         & \hspace{-4mm}E_{{s}}^{\mathsf{P},(t)} =\delta_{{s}}^{\mathsf{P},(t)}\times\left[ \varrho\left(\frac{ z_{{s}}^{(t)}}{C_{{s}}}\right)^2\overline{P}_{{s}}M_s + (1-\varrho)\overline{P}_{{s}}M_s \right]\label{energy_server_proc}\hspace{-1mm},\hspace{-4mm}
    \end{align}
    \vspace{-3mm}
    
    \noindent where $\gamma_{s}^{(t)}$ is the number of SGD iterations, $m_s(t)\in[0,1]$ is the mini-batch ratio, $(1-\varrho)$ is the fraction of power consumed in idle state (typically around $0.4$), and $\overline{P}_s$ is the peak energy consumption of a server belonging to DC $s$.

\textbf{Modeling the Floating Aggregation.}
To capture the model aggregation at the floating aggregation DC, we let the binary indicator $I_{{s}}^{(t)} \in \{0,1\}$ identify whether  DC $s$ aggregates the ML model parameters at $t$ ($I_{{s}}^{(t)} = 1$) or not ($I_{{s}}^{(t)} = 0$). We assume that upon completion of the local model training, each UE $n$ offloads its scaled accumulated gradient to a BS, who then aggregates/sums the received gradients and relays the result to the aggregation DC. For the uplink, let indicator $I_{{n,b}}^{(t)}$ take the value of $1$ if UE $n$ offloads its gradient to BS $b$ after the model training, and $0$ otherwise. Also, for the downlink,  let indicator $I_{b,n}^{(t)}$ take the value of $1$ if UE $n$ receives the aggregated model parameters from BS $b$ and $0$ otherwise.\footnote{Note that BSs broadcast the received aggregated global model from the floating aggregator DC. Thus, each UE may receive the aggregated global model from multiple BSs. Thus, introducing $I_{b,n}^{(t)}$ will ensure the association of a UE to at least one BS for global parameter reception.} 

We then express the delay and energy consumption of transferring model parameters from UE $n$ to the aggregation server, which are denoted by ${\delta}_{n}^{\mathsf{A},(t)}$ and ${E}_{n}^{\mathsf{A},(t)}$, respectively, as
\begin{align}
    &\hspace{-4mm}{\delta}_{n}^{\mathsf{A},(t)} = \sum_{b\in\mathcal{B}} \delta_{n,b}^{\mathsf{M},(t)}I_{n,b}^{(t)} + 
    \underset{{b} \in \mathcal{B}}{\sum}
    \underset{{s} \in \mathcal{S}}{\sum} \delta_{b,{s}}^{\mathsf{M},(t)} I_{{s}}^{(t)}I_{n,b}^{(t)},\hspace{-4mm}\\
    &\hspace{-4mm}{E}_{n}^{\mathsf{A},(t)} = \sum_{b\in\mathcal{B}} E_{n,b}^{\mathsf{M},(t)} I_{n,b}^{(t)} + 
    \underset{{b} \in \mathcal{B}}{\sum}
    \underset{{s} \in \mathcal{S}}{\sum} E_{b,{s}}^{\mathsf{M},(t)} I_{{s}}^{(t)}I_{n,b}^{(t)} .\hspace{-4mm}
\end{align}
Also, the delay and energy consumption of transferring model parameters from DC $s$ to the aggregator server are given by
\begin{align}
   & {\delta}_{s}^{\mathsf{A},(t)} =
    \underset{{s'} \in \mathcal{S}}{\sum} \delta_{s,{s'}}^{\mathsf{M},(t)}I_{{s'}}^{(t)},~~~{E}_{s}^{\mathsf{A},(t)} =
    \underset{{s'} \in \mathcal{S}}{\sum} E^{\mathsf{M},(t)}_{s,{s'}} I_{{s'}}^{(t)}.
\end{align}

\begin{table*}[t]
\vspace{-7mm}
\begin{minipage}{0.99\textwidth}
{\footnotesize
 \begin{align}
     \frac{1}{T} \sum_{t = 0}^{T - 1} \mathbb{E} \bigg[  \|\nabla {{F}}^{(t)}(\mathbf{x}^{(t)})\|^2 \bigg]   \leq \frac{4 \sqrt{{\Bar{\gamma}}}}{ \vartheta \sqrt{dT}} \Bigg[ {{{F}}^{(0)}(\mathbf{x}^{(0)}) - {{F}}^{*}} \Bigg] + \frac{4 \tilde{\tau} \sqrt{{\Bar{\gamma}}}}{ \vartheta \sqrt{dT}} + 16 \frac{L \vartheta \Theta_{\mathsf{max}} \Tilde{\sigma}_{\mathsf{max}}^2 }{ m_{\mathsf{min}}}  \sqrt{\frac{d}{{\Bar{\gamma} T}}} + {\frac{12  L^2d \Theta_{\mathsf{max}} \Tilde{\sigma}_{\mathsf{max}}^2 \gamma_{\mathsf{max}}}{\Bar{\gamma} m_{\mathsf{min}} {T}}} + {\frac{12  L^2 \zeta_2 d \gamma_{\mathsf{max}}^2}{{\Bar{\gamma}T}}}
    \label{eq:convSpecial}
 \end{align}
 }
 \end{minipage}
 \hrule
 \vspace{-6.5mm}
 \end{table*}
Since transferring of gradient across the network occurs in parallel, the delay and energy of gradient aggregation at the aggregator DC (i.e., obtaining $\sum_{i \in \mathcal{N} \cup \mathcal{S}} {  {{D}}_{{i}}^{(t)} } \mathbf{d}_i^{(t)}$ and conducting~\eqref{eq:11M}) can be written as
\vspace{-1.2mm}
\begin{align}
     & {\delta}^{\mathsf{A},(t)} = \max\Big\{ \underbrace{\max_{n\in\mathcal{N}}\{ {\delta}_{n}^{\mathsf{A},(t)} +{\delta}_{n}^{\mathsf{P},(t)} \}}_{(a)}, \nonumber \\
&\hspace{22mm}    \underbrace{\max_{s\in\mathcal{S}}\{{\delta}_{s}^{\mathsf{D},(t)}+ {\delta}_{s}^{\mathsf{P},(t)}+ {\delta}_{s}^{\mathsf{A},(t)} \}}_{(b)} \Big\}, \label{delay_agg}
\end{align}
\vspace{-3mm}
\begin{align}
    &{E}^{\mathsf{A},(t)} =
    \sum_{n\in\mathcal{N}} {E}_{n}^{\mathsf{A},(t)}+ \sum_{s\in \mathcal{S}} {E}_{s}^{\mathsf{A},(t)}.\label{energy_aggregation}
\end{align}
\vspace{-3.4mm}

\noindent In~\eqref{delay_agg}, term $(a)$ corresponds to the time that it takes for the reception of all the gradients of the UEs, which consists of local processing time and uploading the resulting gradients to the aggregating DC from the UEs. Also, term $(b)$ captures the time required for the reception of the gradients of the DCs at the aggregator DC, which consists of data reception time at the DCs, local processing time at the DCs, and uploading  the resulting gradient to the aggregation DC. Note that data offloading to the BSs from the UEs and local processing at the UEs can be conducted in parallel, which is the reason behind the existence of the outer $\max$ function in~\eqref{delay_agg} (e.g., UEs can process their local data while their offloaded data is transferred across the network hierarchy to reach the designated DCs).

After aggregation DC aggregates the models via~\eqref{eq:11M}, it propagates the resulting model to the BSs and the DCs.  The delay and energy consumption of 
global parameter reception at BS $b$ from the floating aggregator DC are given by
\begin{equation}
    \hspace{-4mm}{\delta}_{b}^{\mathsf{R},(t)} =
    \underset{{s} \in \mathcal{S}}{\sum} \delta_{s,{b}}^{\mathsf{M},(t)}I_{{s}}^{(t)},~~\hspace{-.4mm}{E}_{b}^{\mathsf{R},(t)} =
   \underset{{s} \in \mathcal{S}}{\sum} E_{s,{b}}^{\mathsf{M},(t)} I_{{s}}^{(t)}, \hspace{-4mm}
   \vspace{-.5mm}
\end{equation}
respectively.  After arriving at BSs, the global model is broadcast by each BS with the delay and energy consumption of
\begin{align}
    &\hspace{-4mm}{\delta}_{b}^{\mathsf{B},(t)} = \max_{{n\in\mathcal{N}} }\Big\{
    \delta_{b,n}^{\mathsf{M},(t)}I_{b,n}^{(t)}\Big\},~~\hspace{-.1mm}{E}_{b}^{\mathsf{B},(t)} =
{\delta}_{b}^{\mathsf{B},(t)} P^{(t)}_b,
\end{align}
respectively.   Also, the delay and energy consumption of 
parameter reception at DC $s$ from the aggregation DC are
\begin{align}
   & {\delta}_{s}^{\mathsf{R},(t)} =
    \underset{{s'} \in \mathcal{S}}{\sum} \delta_{s',{s}}^{\mathsf{M},(t)}I_{{s'}}^{(t)},~~~{E}_{s}^{\mathsf{R},(t)} =
    \underset{{s'} \in \mathcal{S}}{\sum} E^{\mathsf{M},(t)}_{s',{s}} I_{{s'}}^{(t)},
\end{align}
respectively. The overall delay and energy consumption of parameter reception from the aggregation DC are
\vspace{-1mm}
\begin{align}
     & {\delta}^{\mathsf{R},(t)} = \max\Big\{ \max_{b\in\mathcal{B}}\{ {\delta}_{b}^{\mathsf{R},(t)}+{\delta}_{b}^{\mathsf{B},(t)}  \}
,     \max_{s\in\mathcal{S}}\{  {\delta}_{s}^{\mathsf{B},(t)} \} \Big\},\label{delay_Reception}
\end{align}
   \vspace{-3.5mm}
\begin{align}
    &{E}^{\mathsf{R},(t)} =
    \sum_{b\in\mathcal{B}}\big[ {E}_{b}^{\mathsf{R},(t)}+{E}_{b}^{\mathsf{B},(t)}\big]+ \sum_{s\in \mathcal{S}} {E}_{s}^{\mathsf{R},(t)},\label{energy_Reception}
\end{align}
\vspace{-4.5mm}

\noindent respectively.

\vspace{-3.5mm}
\section{ML Convergence Analysis of {\tt CE-FL}}\label{sec:Conv}
\noindent We next investigate the ML model performance obtained via {\tt CE-FL}. In our analysis, we make the following assumptions:
\begin{assumption}[Smoothness] \label{assumption: Smooothness} For each DPU $i \in \mathcal{N} \cup \mathcal{S} $, the local loss function is $L$-smooth, i.e., $\| \nabla F_{i}^{(t)}(\mathbf{x}) - \nabla F_{i}^{(t)}(\mathbf{y}) \| \leq L \|\mathbf{x}-\mathbf{y}\|$, $\forall \mathbf{x},\mathbf{y}\in\mathbb{R}^p$, where $\|.\|$ denotes the 2-norm.
\end{assumption}
\begin{assumption}[Local Data Variability]
\label{assumption: Local Data Variability}
The local data variability
at each DPU $i \in \mathcal{N} \cup \mathcal{S} $ is measured via a finite constant $\Theta_{i} \geq 0$ that satisfies the following inequality $\forall \mathbf{x}\in\mathbb{R}^p$:
\begin{align}
   \hspace{-2.1mm} \|\nabla {f}(\mathbf{x} ; \xi) \hspace{-.2mm}- \hspace{-.2mm} \nabla {f}(\mathbf{x} ; \xi')\| \hspace{-.2mm}\leq\hspace{-.2mm} \Theta_{i} \|\xi - \xi'\|, \hspace{-.1mm} \forall \xi,\hspace{-.2mm}\xi' \hspace{-1mm}\in \mathcal{D}_{i}^{(t)},\forall t.\hspace{-2mm}
\end{align}
\end{assumption}

\begin{assumption}[Bounded Dissimilarity of Local Loss Functions] \label{assumption: bounded dissimilarity}
For a set of arbitrarily normalized coefficients $\{p_i\}$, where $i \in \mathcal{N} \cup \mathcal{S}$, $p_i \geq 0$ and $\underset{i \in \mathcal{N} \cup \mathcal{S}}{\sum}\  p_i = 1$, there exist two finite constants $\zeta_1 \geq 1$ and $\zeta_2 \geq 0$ such that the following inequality holds for any choice of model parameter $\mathbf{x}\in\mathbb{R}^p$:
\begin{equation}
  \hspace{-4mm}  \underset{i \in \mathcal{N} \cup \mathcal{S}}{\sum}   \hspace{-1mm} p_i \big\|\nabla F_{i}^{(t)}(\mathbf{x})\big\|^2 \hspace{-1mm}\leq \zeta_1 \Big\| \underset{i \in \mathcal{N} \cup \mathcal{S}}{\sum} p_i\nabla F_{i}^{(t)}(\mathbf{x})\Big\|^2 \hspace{-1mm}+ \zeta_2, \forall t.  \hspace{-4mm}
\end{equation}
\end{assumption}
It can be seen that $\zeta_1$ and $\zeta_2$ in the above definition measure the level of data heterogeneity across the clients. If the data is completely homogeneous (i.e., i.i.d.) across the DPUs we will have  $\zeta_1=1$ and $\zeta_2=0$ and these parameters will increase as the data across the DPUs become more heterogeneous. {\color{black} In Appendix~\ref{App:param_est}, we introduce a methodology to estimate the parameters introduced in the above three assumptions.}

We next define \textit{model/concept drift}, which was first characterized by us in~\cite{psl} for dynamic ML. This quantity measures the shift in the local loss imposed by the variation of the local data. We slightly modify the definition in~\cite{psl} and incorporate the duration of global aggregation into the definition:
\vspace{-1mm}
\begin{definition}[Model/Concept Drift] \label{defn: model drift}
For DPU $i \in \mathcal{N} \cup \mathcal{S}$, the online model/concept drift for two consecutive rounds of global aggregation $t$ and $t+1$ is measured by $\Delta_i^{(t)} \in \mathbb{R}$, which captures the maximum variation of the fractional loss function for any $\mathbf{x}\in\mathbb{R}^p$ per unit time according to
\begin{align}
    \frac{{{D}}_i^{(t+1)}}{{D}^{(t+1)}}F_{i}^{(t+1)}(\mathbf{x}) - \frac{{{D}}_i^{(t)}}{{{D}}^{(t)}}F_{i}^{(t)}(\mathbf{x}) \leq \tau^{(t)} \Delta_i^{(t)},
\end{align}
where $\tau^{(t)}$ is the duration between global aggregation $t$ and $t+1$, during which model training is not conducted.
\end{definition}

Letting data variance $\sigma_i^{(t)}$ denote the sampled variance of data at DPU $i$ at time $t$, we next obtain the general convergence behavior of {\tt CE-FL} (proof provided in Appendix~\ref{Proof of Theorem 1}):
\vspace{-1mm}
\begin{theorem}[General Convergence of {\tt CE-FL}] \label{Thm: ML convergence main} Assume that $\eta$ satisfies $4 \eta^2 L^2 \underset{t \in [T]}{\max} \ \underset{i \in \mathcal{N} \cup \mathcal{S}}{\max} \frac{ \gamma_i^2(t) (\|\mathbf{a}_{i}^{(t)}\|_{1} - [a_{i,-1}^{(t)}])}{\|\mathbf{a}_{i}^{(t)}\|_{1}} \leq \frac{1}{2 \zeta_1^2 + 1}$. Under {\tt CE-FL}, the cumulative average of the global loss  satisfies~\eqref{eq:GenConvBound}, where $F^* \triangleq \underset{t \in [T]}{\min}~\underset{\mathbf{x} \in \mathbb{R}^p }{\min}{F}^{(t)}(\mathbf{x})$.
\end{theorem}
The bound in~\eqref{eq:GenConvBound} demonstrates convergence characteristics of~{\tt CE-FL}. It reveals the relationship between the loss gap imposed based on the choice of initial model parameter (term $(a)$). It reveals that as data variability $\{\Theta_i\}$ and data variance $\{\sigma_i\}$ (in terms $(c)$ and $(e)$) increase across the devices, the system experiences a worse convergence performance. This is due to the fact that, given a fixed mini-batch size across the devices, larger values of $\{\Theta_i\}$ and $\{\sigma_i\}$ would imply a higher noise in estimation of true gradient via SGD.
{\color{black} Finally, the effect of increasing the local SGD iterations can be particularly seen in term $(d)$, where as ${\gamma_i^{(t)}}$ increases the convergence bound increases proportionally to the data heterogeneity across the devices ($\zeta_2$ in the coefficient of term $(d)$). Note that $\zeta_1$, $\zeta_2$ quantity the extent of spatial data heterogeneity in terms of ML model loss across the network. Thus, as the data heterogeneity increases, the ML convergence bound favors lesser locally conducted SGD iterations to avoid local model bias. Furthermore, the bound imposed on the step size $\eta$ in the statement of Theorem~\ref{Thm: ML convergence main} implies that as the data heterogeneity measure $\zeta_1$ or number of local SGD iterations $\{\gamma_i\}$ increase, the smaller values of step size are tolerable to avoid local model bias.}


Given the general convergence behavior obtained in Theorem~\ref{Thm: ML convergence main}, we next obtain a specific choice of step size, conditions on the noise of SGD, and a condition on the model drift, under which {\tt CE-FL} converges (proof provided in Appendix~\ref{proof corollary 1}):
\vspace{-2mm}
\begin{corollary}[Convergence of {\tt CE-FL}] \label{corollary: step size main}
 Consider the conditions stated in Theorem \ref{Thm: ML convergence main}. Further assume that $\eta$ is small enough such that $\eta =\sqrt{{d}/({{\Bar{\gamma}T}}})$, where $d = |\mathcal{N} \cup \mathcal{S}|$ and {\small $\Bar{\gamma}= \displaystyle \sum_{t = 1}^{t= T} \sum_{i\in \mathcal{N} \cup \mathcal{S}} \gamma_i^{(t)}$}. If (i) the local data variability satisfy  $\Theta_i \leq \Theta_{\mathsf{max}}$, $\forall i$, for some positive $\Theta_{\mathsf{max}}$,  (ii) the variances of datasets satisfy $\Tilde{\sigma}_{i}^{(t)} \leq \Tilde{\sigma}_{\mathsf{max}}$, $\forall i$, for some positive $\Tilde{\sigma}_{\mathsf{max}}$, (iii) the mini-batch ratios satisfy ${m_{i}^{(t)}} \geq m_{\mathsf{min}}$, $\forall i$, for some positive $m_{\mathsf{min}}$, (iv) the number of SGD iterations satisfy $\gamma_i^{(t)} \leq \gamma_{\mathsf{max}},~\forall i$, for some positive $\gamma_{\mathsf{max}}$,
 and (v) the duration of global aggregations is bounded as {\small $ \displaystyle \tau^{(t)} \leq \max \big\{ \frac{\tilde{\tau}}{T\sum_{i \in \mathcal{N} \cup \mathcal{S}} \Delta_i^{(t)}},0 \big\}$}, for some positive $\tilde{\tau}$, then the cumulative average of the global loss satisfies~\eqref{eq:convSpecial}, implying {\small $ \displaystyle \frac{1}{T} \sum_{t = 1}^{T} \mathbb{E} \Big[ \big\|\nabla {{F}}^{(t)}(\mathbf{x}^{(t)})\big\|^2\Big]= \mathcal{O}(1/\sqrt{T})$}.
\end{corollary}
\vspace{-4mm}
\begin{remark}
Since time varying local loss functions are of our interest, we used the cumulative average of the online global loss as a performance metric. If data is static across the DPUs, i.e., zero model drift, the global loss function becomes time-invariant and the above corollary implies {\small $\frac{1}{T} \sum_{t = 1}^{T} \mathbb{E} \Big[ \big\|\nabla {{F}}(\mathbf{x}^{(t)})\big\|^2\Big]= \mathcal{O}(1/\sqrt{T})$}, revealing that {\tt CE-FL} approaches a stationary point of the global loss function.
\end{remark}
\vspace{-2mm}
Considering Corollary~\ref{corollary: step size main}, guaranteeing the convergence requires  sufficiently small step size, bounded noise of SGD, and reasonably fast global aggregations. In particular, the condition between ($\tau^{(t)}$) and model/concept drift ($\Delta^{(t)}_i$) implies that the speed of triggering new global aggregations should be \textit{inversely} proportional to the model/concept drift: higher concept drift should be met with rapid global aggregations to ensure that the global model can \textit{track} the variations of the UEs' datasets. 

{\color{black} It is worth mentioning that the theoretical convergence results obtained in Theorem~\ref{Thm: ML convergence main} and Corollary~\ref{corollary: step size main}  describe the ML model performance over the online datasets at the devices used for \textit{training}, which is a common approach in the literature of FedL. In practice, the generalization of the trained model to
an unseen \textit{test} dataset is facilitated  via regularization
methods and proper split of test vs. train datasets~\cite{neyshabur2017exploring}.}

\vspace{-1mm}
\section{Network-Aware Model Training through {\tt CE-FL}}\label{sec:NetAwareCEFL}
\noindent 
We will next tie the ML performance of {\tt CE-FL} to the characteristics of the network elements which are used to realize ML model training. 
In particular, to formulate our problem, in our proposed ML model training paradigm, \textit{heterogeneity} at four levels should be considered: (i) UEs with different computation/communication capabilities; (ii) DCs with different number of servers and power consumption profiles; (iii) BSs with different access to the DCs and data transfer rates; (iv) data distribution across the DPUs. To jointly consider all of these heterogeneities, we formulate the network-aware  {\tt CE-FL}:
\vspace{-5mm}
\begin{align}
    & (\small\bm{\mathcal{P}}): {\min} \quad \xi_1 {\underbrace{\frac{1}{T} \sum_{t = 0}^{T-1} \mathbb{E} \bigg[  \|\nabla {{F}}^{(t)}  (\mathbf{x}^{(t)})\|^2 \bigg]}_\text{(a)}}+
    \xi_2 \sum_{t=0}^{T-1}\underbrace{\big[\delta^{\mathsf{A},(t)}+\delta^{\mathsf{R},(t)}\big]}_{(b)} \nonumber \\&  + \xi_3 \sum_{t=0}^{T-1}\Big[\underbrace{ \xi_{3,1}\sum_{n\in\mathcal{N}}\sum_{b\in\mathcal{B}} E_{n,b}^{\mathsf{D},(t)}+ \xi_{3,2} \sum_{b\in\mathcal{B}}\sum_{s\in\mathcal{S}} E^{\mathsf{D},(t)}_{b,s}}_{(c)}\Big]\nonumber  \\
    &
    +\xi_3 \sum_{t=0}^{T-1}\Big[\underbrace{\xi_{3,3} \sum_{n\in\mathcal{N}}E^{\mathsf{P},(t)}_n+\xi_{3,4}\sum_{s\in\mathcal{S}}E^{\mathsf{P},(t)}_s}_{(d)}\Big] \label{eqn: optimization objective text} \\&
    +\xi_3 \sum_{t=0}^{T-1}\Big[\underbrace{\xi_{3,5} E^{\mathsf{A},(t)}+\xi_{3,6}E^{\mathsf{R},(t)}}_{(e)}\Big] 
     \nonumber  \\
   & \textrm{s.t.}~ {\footnotesize \eqref{eqn: BS_capacity},\eqref{eqn: Server_capacity},\eqref{eq:dataRemainAtDevice},\eqref{eq:DataCollecteatServer},\eqref{energy_data},\eqref{eq:delayAtDCdataCollect},
   \eqref{energy_dataBS},\eqref{delay_proc},\eqref{energy_proc},\eqref{eqn: Data Center speed-delay relation-1},\eqref{energy_server_proc},\eqref{energy_aggregation},\eqref{energy_Reception}}, \nonumber \\
    & \underset{b \in \mathcal{B}}{\sum} \rho_{n,b}^{(t)} \leq 1, \ \ n \in \mathcal{N} \label{rhoUE-BS} \\
    & \underset{s \in \mathcal{S}}{\sum} \rho_{b, s}^{(t)} = 1, \ \ b \in \mathcal{B}  \label{rhoBS-DC} \\
         & \underset{{s} \in  \mathcal{S}}{\sum}{I}_{\mathsf{s}}^{(t)} = 1 \label{const: Aggregator indicator summation constraint main}\\
         & \underset{{b} \in  \mathcal{B}}{\sum}{I}^{(t)}_{n,b} = 1 \label{const: UE-BS indicator summation constraint 1}, ~\forall n\in\mathcal{N}\\
         & \underset{{b} \in  \mathcal{B}}{\sum}{I}_{b,n}^{(t)} = 1, ~\forall n\in\mathcal{N}\label{const: BS-UE indicator summation constraint 1}\\
         & {\delta}_{n}^{\mathsf{A},(t)} +{\delta}_{n}^{\mathsf{P},(t)} \leq {\delta}^{\mathsf{A},(t)}, ~\forall n\in \mathcal{N} \label{delay_agg_UE_mainpaper}\\
    & {\delta}_{s}^{\mathsf{D},(t)}+ {\delta}_{s}^{\mathsf{P},(t)}+ {\delta}_{s}^{\mathsf{A},(t)} \leq {\delta}^{\mathsf{A},(t)}, ~\forall s\in \mathcal{S} \label{delay_agg_servers_mainpaper} \\
   & {\delta}_{b}^{\mathsf{R},(t)}+{\delta}_{b}^{\mathsf{B},(t)} \leq {\delta}^{\mathsf{R},(t)}, ~\forall b\in\mathcal{B} \label{delay_BS_Reception_mainpaper}\\
    & {\delta}_{s}^{\mathsf{B},(t)} \leq {\delta}^{\mathsf{R},(t)}, ~\forall s\in\mathcal{S} \label{delay_server_Reception_mainpaper} \\
             & 0\leq z_{{s}}^{(t)} \leq C_{\mathsf{s}}, \  s \in \mathcal{S} \label{const: Data Center data capacity 2 1 main}\\
   &  0 \leq \rho_{n, b}^{(t)} \leq 1, \ \  n\in \mathcal{N}, b\in \mathcal{B}
    \label{const: BS offload ratio constraint main} \\
    &  0 \leq \rho_{b, s}^{(t)} \leq 1, \ \  b\in \mathcal{B}, s\in \mathcal{S}\\
    & f^{\mathsf{min}}_{n} \leq f_{n}^{(t)} \leq f^{\mathsf{max}}_{n},  \ \  n \in \mathcal{N}  \label{const: CPU frequency constraint main}\\
    & 0 \leq m_{i}^{(t)} \leq 1, \ \   i \in \mathcal{N} \cup \mathcal{S} \label{const: mini batch fraction constraint main} \\
    & \gamma_{i}^{(t)} \geq 0, \ \  i \in \mathcal{N} \cup \mathcal{S} \label{const: epoch eqn 1 main}\\
     & \delta^{\mathsf{A},(t)} \geq 0 ,~ \delta^{\mathsf{R},(t)} \geq 0 \label{eqn:delay_reception_opt_main}\\
    & {I}^{(t)}_{n,b}\in\{0,1\}~,{I}^{(t)}_{b,n}\in\{0,1\},~n\in\mathcal{N},~b\in\mathcal{B}\label{indicatorBSFloat}\\
     & {I}_{s}^{(t)} \in \{0,1\}, \ \   s \in \mathcal{S} \label{indicatorFloat} \\
    &\textrm{\textbf{Variables:}} \nonumber \\[-.1em]
      & \bm{w}:
       \{[{\rho^{(t)}_{n,b}}]_{n\in\mathcal{N},b\in\mathcal{B}},[{\rho^{(t)}_{b,s}}]_{b\in\mathcal{B},s\in\mathcal{S}},[f_n^{(t)}]_{n\in\mathcal{N}},[z_s^{(t)}]_{s\in\mathcal{S}}, \nonumber \\
      & \nonumber
      [\gamma_i^{(t)}]_{i\in\mathcal{N} \cup \mathcal{S}},[m_i^{(t)}]_{i\in\mathcal{N} \cup \mathcal{S}},[I_s^{(t)}]_{s\in\mathcal{S}},[I_{n,b}^{(t)}]_{n\in\mathcal{N},b\in\mathcal{B}}, \\
      & \nonumber  [I_{b,n}^{(t)}]_{b\in\mathcal{B},n\in\mathcal{N}},{\delta^{\mathsf{A},(t)}, \delta^{\mathsf{R},(t)}}, \textcolor{black}{ [R_{ {{b}}, {s}}^{(t)}]_{b\in\mathcal{B},s\in\mathcal{S}}}\}_{t=1}^{T}
    \end{align}
    \vspace{-6mm}
    
\subsubsection{Objective and Variables}
The objective of~$\bm{\mathcal{P}}$ captures a trade-off between ML model performance (term $(a)$), which is replaced with the right hand side of the bound in  \eqref{eq:convSpecial}, the delay of obtaining new global parameters (term $(b)$), and the energy consumption of model training (term $(c),(d),(e)$). In bound \eqref{eq:convSpecial}, we replace $\tau^{(t)}$ with $\delta^{\mathsf{A}, (t)}  + \delta^{\mathsf{R}, (t)}$, which is an upper bound on it, for the tractability of the solution. Constants $\zeta_1,\zeta_2,\zeta_3$ weigh these (possibly competing) objectives.\footnote{\textcolor{black}{We highlight that \textit{ML performance weight} $\xi_1$, which is the coefficient of the ML model performance bound (term $(a)$) in the objective function of $\bm{\mathcal{P}}$ balances the trade-off between the global ML model performance against network costs (i.e., energy consumption and delay) associated with \texttt{CE-FL}. In our experiments (Sec.~\ref{subsection: ML weight tradeoff}), we show how varying $\xi_1$ impacts  device SGD mini-batch ratios and energy consumption.}}

Also, the constants $\zeta_{3,1}-\zeta_{3,6}$ in terms $(c),(d),(e)$ weigh the impact of consumed energy  for data transfer (term $(c)$), local model computation (term $(d)$), and model aggregation (term $(e)$). The value of these coefficients  may vary in different applications. The optimization variables of our problem are the UE-BS data offloading ratios $ [{\rho^{(t)}_{n,b}}]_{n\in\mathcal{N},b\in\mathcal{B}}$, BS-DC data offloading rations $[{\rho^{(t)}_{b,s}}]_{b\in\mathcal{B},s\in\mathcal{S}}$, CPU frequency at the devices $[f_n^{(t)}]_{n\in\mathcal{N}}$, 
speed of data processing at the DCs $[z_s^{(t)}]_{s\in\mathcal{S}}$, number of SGD iterations at DPUs $[\gamma_i^{(t)}]_{i\in\mathcal{N}\cup\mathcal{S}}$,
mini-batch size of SGD $[m_i^{(t)}]_{i\in\mathcal{N}\cup\mathcal{S}}$, the index of the floating aggregation DC captured via $[I_s(t)]_{s\in\mathcal{S}}$,  UE-to-BS association for offloading the final UE's local gradient $[I_{n,b}^{(t)}]_{n\in\mathcal{N},b\in\mathcal{B}}$, BS-to-UE association for receiving the aggregated global model $[I_{b,n}^{(t)}]_{b\in\mathcal{B},n\in\mathcal{N}}$, $\forall t$.

Constraints~\eqref{energy_data},\eqref{energy_dataBS},\eqref{energy_proc},\eqref{energy_server_proc},\eqref{delay_agg},\eqref{energy_aggregation},\eqref{delay_Reception}, and \eqref{energy_Reception}, describe the terms used in the objective function. Also, constraints~\eqref{rhoUE-BS} and \eqref{rhoBS-DC} guarantee a feasible data dispersion in UE-BS and BS-DC communications. Also,~\eqref{const: Aggregator indicator summation constraint main} and~\eqref{indicatorFloat} ensure the existence of only one floating aggregation DC at each global aggregation. Similarly, \eqref{const: UE-BS indicator summation constraint 1} and \eqref{const: BS-UE indicator summation constraint 1} along with~\eqref{indicatorBSFloat} guarantee proper  BS-UE communications. 
To help with the decomposition of the problem, we revisited \eqref{delay_agg} and considered $\delta^{\mathsf{A}, (t)}$ as an optimization variable accompanied with constraints \eqref{delay_agg_UE_mainpaper} and \eqref{delay_agg_servers_mainpaper}. Using a similar argument for \eqref{delay_Reception}, we made $\delta^{\mathsf{R}, (t)}$ an optimization variable and incorporated \eqref{delay_BS_Reception_mainpaper} and \eqref{delay_server_Reception_mainpaper}.  Finally,~\eqref{const: Data Center data capacity 2 1 main}-\eqref{const: epoch eqn 1 main} ensure the feasibility of the solution.

 Roughly speaking, $\bm{\mathcal{P}}$ aims to achieve the lowest model loss, while minimizing the delay and energy consumption. This will result in (i) a \textit{simultaneous load balancing} across the UEs and DCs for data processing, and (ii) an \textit{efficient data/parameter routing} across the network hierarchy, and (iii) optimized floating aggregation DC that causes minimal delay and energy of parameter aggregation and reception.
\subsubsection{Challenges Faced in Solving the Problem}\label{chal}
There are three challenges faced in solving~$\bm{\mathcal{P}}$ discussed below. 
\begin{enumerate}[leftmargin=5mm]
    \item  $\bm{\mathcal{P}}$ belongs to the category of mixed integer optimization problems, which are in general highly non-trivial to solve. This is due the existence of discrete/binary variables used to denote the floating server selection (i.e., $I_s^{(t)}$, $\forall s\in\mathcal{S}$) and UE-BS association in uplink and downlink communications (i.e., $I^{(t)}_{n,b}$ and $I^{(t)}_{b,n}$, $\forall b\in\mathcal{B}, n\in\mathcal{N}$) in conjunction with the rest of the continuous optimization variables. 
    \item The objective of $\bm{\mathcal{P}}$ given by \eqref{eqn: optimization objective text} is highly non-convex with respect to the continuous optimization variables.  In particular, in~\eqref{eqn: optimization objective text},
the ML loss term $(a)$ given by \eqref{eq:convSpecial}
is non-convex with respect to local SGD iteration counts (i.e, $\gamma_i^{(t)},~\forall i\in\mathcal{N}\cup\mathcal{S}$ and the offloading parameters (i.e., $\rho^{(t)}_{n,b}$ and $\rho^{(t)}_{b,s}$, $\forall n\in\mathcal{N},b\in\mathcal{B},s\in\mathcal{S}$, which are incorporated in the number of data points $D_i^{(t)}$ in the bound in \eqref{eq:convSpecial} through~\eqref{eq:dataRemainAtDevice} and~\eqref{eq:DataCollecteatServer}). Furthermore, term (b) in \eqref{eqn: optimization objective text} is non-convex with respect to the optimization variables due to the multiplication between the optimization variables in~\eqref{delay_proc} and~\eqref{eqn: Data Center speed-delay relation-1}, which are encapsulated in the processing delay (i.e., $\delta^{\mathsf{P}, (t)}_s$ and $\delta^{\mathsf{P}, (t)}_n$) in $\delta^{\mathsf{A}, (t)}$ as described by~\eqref{delay_agg}.
Similarly, the computation energy expressions given by \eqref{energy_proc} and \eqref{energy_server_proc} for the UEs and DCs which are incorporated into the energy terms in term $(d)$ in~\eqref{eqn: optimization objective text} contain multiplication of optimization variables making them non-convex.
\item In a large-scale network, there is no central entity to solve $\bm{\mathcal{P}}$. In particular, it is impractical to consider a central DC with the knowledge of all the DPUs capabilities and link data rates, which are prerequisites to solve the problem.
\end{enumerate}
We are thus motivated to develop a unique network element orchestration scheme via (i) effective \textit{relaxation} of the integer variables, (ii) \textit{successive convexification} of the problem, and (iii) \textit{distributing the solution computation} across the network elements. Nevertheless, achieving this goal is not trivial and requires a careful investigation of $\bm{\mathcal{P}}$, which is carried out next.

\vspace{-2mm}
\section{Distributed Network Orchestration in {\tt CE-FL}}\label{sec:distSolv}
\noindent We next develop a distributed solution for $\bm{\mathcal{P}}$, where the computation burden of obtaining the solution is spread across the network elements (i.e., UEs, BSs, and DCs). It is worth mentioning that our methodology is among the first distributed network element orchestration schemes in  the broad area of \textit{network-aware distributed machine learning}, where we show how distributed optimization techniques can be exploited to orchestrate the devices for a distributed ML task. In our methodology, each network element  will solve a reduced/truncated version of $\bm{\mathcal{P}}$ to obtain its associated optimization variables (e.g., mini-batch size and number of SGD iterations at the UEs), while forming a consensus with other network elements on the rest of the variables (e.g., the floating aggregator DC). We also study the {optimality} of the obtained solution.  

The design of the distributed solution framework addresses the three challenges associated with $\bm{\mathcal{P}}$ (Sec.~\ref{chal}). In the following, we discuss our approach to addressing them.

\textbf{Relaxing the Integer Variables.} We first relax the integer variables to be continuous (i.e., $I_s^{(t)}\in[0,1]$, $\forall s\in\mathcal{S}$, and $I^{(t)}_{n,b},I^{(t)}_{b,n}\in [0,1]$, $\forall b\in\mathcal{B}, n\in\mathcal{N}$). Then, we \textit{force} them to take binary values via incorporating the following constraints:
\begin{align}
    &\sum_{s\in\mathcal{S}} I_s^{(t)}\left(1- I_{s}^{(t)}\right)\leq 0 ,\label{con1add}\\
    &  \sum_{b\in\mathcal{B}}   I_{n,b}^{(t)}\left(1- I_{n,b}^{(t)}\right)\leq 0,~ n\in\mathcal{N}   \label{con12add}     ,\\
     &   \sum_{b\in\mathcal{B}}  I_{b,n}^{(t)}\left(1- I_{b,n}^{(t)}\right)\leq 0 ,~ n\in\mathcal{N}         \label{con11add} , \\
     &\sum_{s\in\mathcal{S}} I_s^{(t)}=1,~~~\sum_{b\in\mathcal{B}} I_{n,b}^{(t)}=\sum_{b\in\mathcal{B}} I_{b,n}^{(t)}=1,~n\in\mathcal{N}, \label{con21add}\\
     &I_s^{(t)}\in[0,1],~  s\in\mathcal{S}, \label{con22add}
     \\&I^{(t)}_{n,b},I^{(t)}_{b,n}\in [0,1], ~ b\in\mathcal{B}, n\in\mathcal{N}.\label{con3add}
\end{align}
The above constraints ensure that the indicated continuous variables would take binary values under any feasible solution. Also, they guarantee that \textit{only one} DC will be selected as the aggregator and \textit{only one} BS will be associated with each UE during uplink/downlink parameter transfer. Note that the introduced constraints in~\eqref{con1add}-\eqref{con11add} are non-convex.


\textbf{Distribution/Decomposition of Variables of $\small\bm{\mathcal{P}}$.} We break down the optimization variables in $\small\bm{\mathcal{P}}$ into two categories: (i) \textit{local variables}, which are obtained locally at each network element, and (ii) \textit{shared variables}, which are first optimized locally and then synchronized via a consensus mechanism across the adjacent network elements (e.g., UE-BS or BS-DC pairs).
The UE-BS offloading parameters denoted by $\{[{\rho^{(t)}_{n,b}}]_{n\in\mathcal{N},b\in\mathcal{B}}\}_{t=1}^{T}$ determine the number of datapoints accumulated at the BSs, and subsequently dictate the dataset gathered at the DCs. Thus,  ${\rho^{(t)}_{n,b}}$ for a given UE and BS $n \in \mathcal{N},b \in \mathcal{B}$ are {shared variables}.
A similar justification leads to ${\rho^{(t)}_{b, s}}$ being shared by the constituent BS and DC $b \in \mathcal{B}, s \in \mathcal{S}$. The floating server selection indicators $\{[I_s^{(t)}]_{s\in\mathcal{S}}\}_{t=1}^{T}$ are also {shared variables} among all the network elements  (i.e., ${\mathcal{N} \cup \mathcal{S} \cup \mathcal{B}}$) as it impacts the delay of parameter transfer. Furthermore, it is evident from \eqref{delay_agg_UE_mainpaper}-\eqref{delay_agg_servers_mainpaper} that aggregation delay  $\{[{\delta}^{\mathsf{A},(t)}]\}_{t=1}^{T}$ are shared variables for set of nodes in $\mathcal{N} \cup \mathcal{S}$. 
Also, \eqref{delay_BS_Reception_mainpaper}-\eqref{delay_server_Reception_mainpaper} imply that the reception delay  $\{[{\delta}^{\mathsf{R},(t)}]\}_{t=1}^{T}$  are shared variables between the nodes in $\mathcal{B} \cup \mathcal{S}$.
Rest of the variables associated with $\bm{\mathcal{P}}$
are local variables assigned to their respective individual nodes. For instance, the variables associated with ML model training and data processing at each UE $n \in \mathcal{N}$ (i.e.,   SGD mini-batch ratio $\{[m_n^{(t)}]\}_{t=1}^{T}$, the number of SGD iterations $\{[\gamma_n^{(t)}]\}_{t=1}^{T}$, and CPU frequency $\{[f_n^{(t)}]\}_{t=1}^{T}$) are local variables.


To compute the shared variables, we first \textit{expand} the solution vector $\bm{w}$ of problem~$\bm{\mathcal{P}}$ by creating \textit{local copies} of the shared variables at their constituent nodes and introducing equality constraints to enforce agreement among the local copies. We summarize the variables computed at each node below:
\begin{enumerate}[leftmargin=4.2mm]
    \item Each UE $n$: 
    {\small $ \bm{w}^{\mathsf{Local}}_{n} =  \big\{ [f_n^{(t)}] ,[m_n^{(t)}], [\gamma_n^{(t)}], {[I_{n,b}^{(t)}]_{b\in\mathcal{B}}} \big\}_{t=1}^{T} , \\
       \bm{w}^\mathsf{Shared}_{n} = \big\{ [{\rho^{(t), n}_{n,b}}]_{b\in\mathcal{B}}, [I_s^{ (t), n}]_{s\in\mathcal{S}}, [{\delta}^{\mathsf{A},(t), n}]\big\}_{t=1}^{T}, \label{eqn: UE_local_vars}
    $}
    where {\small $\rho^{(t), n}_{n,b},I_s^{ (t), n},{\delta}^{\mathsf{A},(t), n}$} denote the local copies of the shared variables (i.e., {\small $\rho^{(t)}_{n,b},I_s^{ (t)},{\delta}^{\mathsf{A},(t)}$}) at node $n$.
    \item Each BS $b$: 
    {\small
    $ \bm{w}^{\mathsf{Local}}_{b} = \big\{ [I_{b,n}^{(t)}]_{n\in\mathcal{N}} \big\}_{t=1}^{T} ,
      \bm{w}^\mathsf{Shared}_{b} =  \nonumber \big\{ [{\rho^{(t), b}_{n,b}}]_{n\in\mathcal{N}}, [{\rho^{(t), b}_{b,s}}]_{s\in\mathcal{S}} , [I_s^{(t), b}]_{s\in\mathcal{S}}, [{\delta}^{\mathsf{A},(t), b}], [{\delta}^{\mathsf{R},(t), b}] $, {\color{black}
      $[{R^{(t), b}_{b,s}}]_{s\in\mathcal{S}}$} $\big\}_{t=1}^{T}
    $}
         where {\small${\rho^{(t), b}_{n,b}},{\rho^{(t), b}_{b,s}},I_s^{(t), b},{\delta}^{\mathsf{A},(t), b},{\delta}^{\mathsf{R},(t), b}$, ${R^{(t), b}_{b,s}} $} denote the local copies of the respective shared variables (i.e., {\small${\rho^{(t)}_{n,b}},{\rho^{(t)}_{b,s}},I_s^{(t)},{\delta}^{\mathsf{A},(t)},{\delta}^{\mathsf{R},(t)}$, \textcolor{black}{${R^{(t)}_{b,s}}$}}) at node $b$.
    \item Each DC $s$: 
     {\small $\bm{w}^{\mathsf{Local}}_{s} = \big\{  [z_s^{(t)}],[\gamma_s^{(t)}],  
      [m_s^{(t)}] \big\}_{t=1}^{T}  
   , \bm{w}^\mathsf{Shared}_{s} = \big\{ [{\rho^{(t), s}_{n,b}}]_{n\in\mathcal{N},b\in\mathcal{B}}, [{\rho^{(t), s}_{b,s}}]_{b\in\mathcal{B}}, [I_s^{(t), s}]_{s\in\mathcal{S}},
      [{\delta}^{\mathsf{A},(t), s}], 
[{\delta}^{\mathsf{R},(t), s}]$, {\color{black} $[{R^{(t), s}_{b,s}}]_{b\in\mathcal{B}}$} $\big\}_{t=1}^{T}$}
       where ${\rho^{(t), s}_{n,b}},{\rho^{(t), s}_{b,s}},{\delta}^{\mathsf{A},(t), s},I_s^{(t), s},{\delta}^{\mathsf{R},(t), s}$ denote the local copies of the respective shared variables (i.e., ${\rho^{(t)}_{n,b}},{\rho^{(t)}_{b,s}},{\delta}^{\mathsf{A},(t)},I_s^{(t), s},{\delta}^{\mathsf{R},(t)}, \textcolor{black}{{R^{(t)}_{b,s}}}$) at  node $s$.
    \item For each network element ${d} \in \mathcal{N} \cup \mathcal{S} \cup \mathcal{B}$, we let $\bm{w}_{{d}}$ encompass all the associated variables: 
  $\bm{w}_{{d}} = \bm{w}^{\mathsf{Local}}_{{d}} \cup  \bm{w}^\mathsf{Shared}_{{d}}.
   $
    Thus, with some abuse of notation we denote the extended variable space of $\bm{\mathcal{P}}$ via $\bm{w}$ defined as
    \vspace{-1mm}
    \begin{equation}
       \bm{w} = [\bm{w}_{d}]_{{d} \in \mathcal{N} \cup \mathcal{S} \cup \mathcal{B}} \triangleq \bigcup_{{d} \in \mathcal{N} \cup \mathcal{S} \cup \mathcal{B}} \bm{w}_{{d}}. \label{eqn: parameter_union_over_nodes}
    \end{equation}
\end{enumerate}
\vspace{-1mm}
Additionally, equality constraints introduced to $\bm{\mathcal{P}}$ to enforce the equality of the local copies of the shared variables are
\vspace{-1mm}
    \begin{align}
       & \hspace{-2mm} \rho^{(t),n}_{n,b}  -\rho^{(t),b}_{n,b} = 0, ~\forall n\in \mathcal{N}, b \in \mathcal{B}, ~\forall t, \label{eqn: local copy 1}\\
      & \hspace{-2mm} {\rho^{(t), n}_{n,b}} - {\rho^{(t), s}_{n,b}} = 0,   ~\forall n\in \mathcal{N}, s \in \mathcal{S}, ~{b\in\mathcal{B}}, ~\forall t, \\
      &\hspace{-2mm}  \rho^{(t),b}_{b,s}  - \rho^{(t),s}_{b,s} = 0 , ~\forall b\in \mathcal{B}, s \in \mathcal{S}, ~\forall t, \\
      &\hspace{-2mm}  I_s^{ (t), {d}} - I_s^{(t), {d}'} = 0, ~\forall {d}, {d}' \in \mathcal{N} \cup \mathcal{B} \cup \mathcal{S}, ~{s\in\mathcal{S}}, ~\forall t, \\
      &\hspace{-2mm}  {\delta}^{\mathsf{A},(t), {d}} - {\delta}^{\mathsf{A},(t), {d}'} = 0, ~\forall {d}, {d}' \in \mathcal{N} \cup \mathcal{S},  ~\forall t, \\
       &\hspace{-2mm}  {\delta}^{\mathsf{R},(t), {d}} - {\delta}^{\mathsf{R},(t), {d}'} = 0, ~\forall {d}, {d}' \in \mathcal{B} \cup \mathcal{S} , ~\forall t,\\
       &  {\color{black}\hspace{-2mm}R^{(t),b}_{b,s}  - R^{(t),s}_{b,s} = 0 , ~\forall b\in \mathcal{B}, s \in \mathcal{S}, ~\forall t.}\label{eqn: local copy 2}
    \end{align}
    \vspace{-7.5mm}
    
    Imposing \eqref{eqn: local copy 1}-\eqref{eqn: local copy 2} leads to \textit{separability} of the optimization
    problem while ensuring agreement on the shared variables, allowing the development of a distributed solution later.
 Upon extending $\bm{\mathcal{P}}$ by incorporating relaxation of integer variables, segregation of variables into local and shared categories, followed by imposing agreement on the local copies of shared variables as discussed above, we now proceed to developing the distributed solver. We first highlight some characteristics of $\bm{\mathcal{P}}$, which are exploited in our distributed solution framework. 
 \vspace{-1mm}
 
 \textbf{Structure of $\bm{\mathcal{P}}$.}
 Let us denote the objective of $\bm{\mathcal{P}}$ given by \eqref{eqn: optimization objective text} as $\small{\mathcal{J}}$. We note that the convex constraints of $\bm{\mathcal{P}}$ comprise linear summations in \eqref{rhoUE-BS}-\eqref{const: BS-UE indicator summation constraint 1} and variable ranges in \eqref{const: Data Center data capacity 2 1 main}-\eqref{eqn:delay_reception_opt_main},  \eqref{con22add}-\eqref{con3add}, defined independently across nodes.
 We thus combine these convex constraints 
 for each node ${d}\in \mathcal{N} \cup \mathcal{B} \cup \mathcal{S}$ as constraint vector $ \bm{{\mathcal{D}}}_{{d}}(\bm{w}_{{d}}) \leq \mathbf{0}$ 
 The constraints \eqref{eq:dataRemainAtDevice},\eqref{eq:DataCollecteatServer},\eqref{energy_data},\eqref{eq:delayAtDCdataCollect},
  \eqref{energy_dataBS},\eqref{delay_proc},\eqref{energy_proc},\eqref{eqn: Data Center speed-delay relation-1},\eqref{energy_server_proc},\eqref{energy_aggregation},\eqref{energy_Reception},\eqref{delay_agg_UE_mainpaper},\eqref{delay_agg_servers_mainpaper},\eqref{delay_BS_Reception_mainpaper},\eqref{delay_server_Reception_mainpaper},\eqref{con1add}-\eqref{con11add} associated with $\bm{\mathcal{P}}$ are non-convex and can be jointly denoted by vector of constraints $\bm{\mathcal{C}}(\bm{w}) \leq \mathbf{0}$. Furthermore, we observe that  \eqref{eqn: local copy 1}-\eqref{eqn: local copy 2} are linear equality constraints involving one or more network elements, thus can be equivalently written as an equality constraint in vector form, i.e., $\sum_{{d} \in \mathcal{N} \cup \mathcal{S} \cup \mathcal{B}} \bm{\mathcal{G}}(\bm{{w}}_{{d}}) = \mathbf{0}$. Then, the augmented version of $\bm{\mathcal{P}}$ denoted by $\bm{\widehat{\mathcal{P}}}$ which encompasses all the aforementioned constraints can be written as
  \vspace{-1mm}
\begin{align}
    & \bm{\widehat{\mathcal{P}}}: \underset{\bm{w}}{\min} \quad \small{\mathcal{J}}(\bm{w}) \label{eqn: obj func} \\
    & \textrm{s.t.}~~ \bm{\mathcal{C}}(\bm{w}) \leq \mathbf{0}  \label{eqn: non-convex constraints}\\
    & ~~~~\displaystyle \sum_{{d} \in \mathcal{N} \cup \mathcal{S} \cup \mathcal{B}} \bm{\mathcal{G}}(\bm{{w}}_{{d}}) = \mathbf{0}  \label{eqn: convex separable constraints} \\
    & ~~~~\bm{{\mathcal{D}}}_{{d}}(\bm{w}_{{d}}) \leq \mathbf{0}, ~{d} \in \mathcal{N} \cup \mathcal{S} \cup \mathcal{B}. \label{eqn: convex constraints}
\end{align}
\begin{algorithm}[!t]
	\caption{Successive Convex Solver Wrapper} \label{alg:dist_NOVA main}
    \begin{algorithmic}[1]
         \STATE \textcolor{black}{{\textbf{Input:} Initialize  ${\bm{w}}^{(0)}$, step size $\zeta$}} 
          \STATE \textcolor{black}{{\textbf{Output:} Final iterate $\bm{w}$ }}
          \STATE { \blue Initialize  
          iteration count $\ell = 0$.}
          \WHILE{${\bm{w}}^{(\ell)}$ has not converged}
          \STATE Compute $\hat{\bm{w}}({\bm{w}}^{(\ell)})$, the distributed parallel solution of $\widehat{\bm{\mathcal{P}}}_{{\bm{w}}^{(\ell)}}$ using \underline{PD \texttt{CE-FL}} (Algorithm \ref{alg:dist_opt_algo main}) 
          \STATE Obtain ${\bm{w}}^{(\ell+1)}$ using update rule \eqref{eqn:NOVA param update}
          \STATE $\ell \longleftarrow \ell + 1$
          \ENDWHILE
    \end{algorithmic}
\end{algorithm}
Next we describe our successive convex methodology, which is an \textit{iterative procedure}, wherein $\small \widehat{\bm{\mathcal{P}}}$ is relaxed and solved.




\textbf{Successive Convex Solver.} The psudo-code of our successive convex solver is given in Algorithm~\ref{alg:dist_NOVA main}.  The algorithm starts with an initial feasible point $\bm{{w}}^{(0)}$ satisfying the constraints of $\small\bm{\mathcal{P}}$. During each iteration $\ell$ of the algorithm, we convexify $\small \widehat{\bm{\mathcal{P}}}$ at the given feasible point $\bm{{w}}^{(\ell)}$ to obtain \textit{surrogate problem} $\widehat{\bm{\mathcal{P}}}_{{\bm{w}}^{(\ell)}}$ which can be further distributed and solved across the network elements ${d} \in \mathcal{N} \cup \mathcal{S} \cup \mathcal{B}$ in parallel. We denote the distributed solution of the surrogate problem $\widehat{\bm{\mathcal{P}}}_{{\bm{w}}^{(\ell)}}$  by $\hat{\bm{w}}({\bm{w}}^{(\ell)})$.   Subsequently, we update the variables as ($\zeta <1$)
\begin{align}
          {\bm{w}}^{(\ell+1)} = {\bm{w}}^{(\ell)} + {\zeta}(\hat{\bm{w}}({\bm{w}}^{(\ell)})-{\bm{w}}^{(\ell)}) \label{eqn:NOVA param update}.
\end{align}
The key aspects of Algorithm \ref{alg:dist_NOVA main} are thus (i) obtaining the convex approximation, i.e., $\widehat{\bm{\mathcal{P}}}_{{\bm{w}}^{(\ell)}}$, and (ii) the design of the distributed solution. Henceforth, we first describe the convex approximation technique used to relax our original problem, and then build our parallel distributed solver.

\textbf{Convexification of $\widehat{\bm{\mathcal{P}}}$.}
 We use a proximal gradient method  
 to relax the objective ${\mathcal{J}}(\bm{w})$. The non-convex constraints $\bm{\mathcal{C}}(\bm{w})$ are also convexified such that the approximations upper-bound the original constraints. In particular, at iteration $\ell$ of our successive convex solver (Algorithm \ref{alg:dist_NOVA main}), given the current solution 
 of $\small \widehat{\bm{\mathcal{P}}}$, i.e., $\bm{{w}}^{(\ell)}$, we obtain convex approximations of objective $\small{\mathcal{J}}$ and non-convex constraints ${\bm{\mathcal{C}}}$ denoted by $\small{\widetilde{\mathcal{J}}}$ and $\small\bm{\widetilde{\mathcal{C}}}$, respectively, as
\begin{align}
    &\hspace{-2.3mm} {\widetilde{\mathcal{J}}}(\bm{w};\bm{{w}}^{(\ell)}) = \sum_{{d} \in \mathcal{N} \cup \mathcal{S} \cup \mathcal{B}} \small{\widetilde{\mathcal{J}}}_{{d}}(\bm{w}_{{d}};\bm{{w}}^{(\ell)}), \label{eqn:prox grad obj 1}\\
    &\hspace{-2.3mm} \nonumber \small{\widetilde{\bm{\mathcal{J}}}}_{{d}}(\hspace{-.3mm}\bm{w}_{{d}};\bm{{w}}^{(\ell)}\hspace{-.3mm}) \hspace{-.3mm}=\hspace{-.3mm} \textcolor{black}{\frac{1}{|\mathcal{N} \cup \mathcal{S} \cup \mathcal{B}|} \bm{\mathcal{J}}(\hspace{-.3mm}\bm{w}^{(\ell)}}\hspace{-.3mm})\hspace{-.3mm}+\hspace{-.3mm} \nabla_{\bm{w}_{{d}}} \small{\bm{\mathcal{J}}}(\bm{{w}}^{(\ell)})^\top \hspace{-.3mm}(\hspace{-.3mm}\bm{w}_{{d}} - \bm{{w}}^{(\ell)}_{{d}}\hspace{-.3mm})\\
    & \hskip 2cm   + \frac{\lambda_1}{2} \|\bm{w}_{{d}} - \bm{{w}}^{(\ell)}_{{d}}\|^2. \label{eqn:prox grad obj 2} \\
    &\hspace{-2.3mm} {\widetilde{\bm{\mathcal{C}}}}(\bm{w};\bm{{w}}^{(\ell)}) = \sum_{{d} \in \mathcal{N} \cup \mathcal{S} \cup \mathcal{B}} \small{\widetilde{\bm{\mathcal{C}}}}_{{d}}(\bm{w}_{{d}};\bm{{w}}^{(\ell)}), \label{eqn:prox grad non convex constr 1}\\
    &\hspace{-2.3mm} \nonumber \small{\widetilde{\bm{\mathcal{C}}}}_{{d}}(\hspace{-.3mm}\bm{w}_{{d}};\bm{{w}}^{(\ell)}\hspace{-.3mm}) \hspace{-.3mm}=\hspace{-.3mm} \frac{1}{|\mathcal{N} \cup \mathcal{S} \cup \mathcal{B}|} \bm{\mathcal{C}}(\hspace{-.3mm}\bm{w}^{(\ell)}\hspace{-.3mm})\hspace{-.3mm}+\hspace{-.3mm} \nabla_{\bm{w}_{{d}}} \small{\bm{\mathcal{C}}}(\bm{{w}}^{(\ell)})^\top \hspace{-.3mm}(\hspace{-.3mm}\bm{w}_{{d}} - \bm{{w}}^{(\ell)}_{{d}}\hspace{-.3mm})\\
    & \hskip 2cm   + \frac{L_{\nabla{\bm{\mathcal{C}}}}}{2} \|\bm{w}_{{d}} - \bm{{w}}^{(\ell)}_{{d}}\|^2. \label{eqn:prox grad non-convex constr 2}
\end{align}
In~\eqref{eqn:prox grad non convex constr 1}-\eqref{eqn:prox grad non-convex constr 2}, $L_{\nabla{\bm{\mathcal{C}}}}$ is the Lipschitz constant which is a characteristic of the constraint function $\bm{\mathcal{C}}$. The above formulation implies that ${\widetilde{\bm{\mathcal{C}}}}(\bm{w};\bm{{w}}^{(\ell)}) \geq \bm{\mathcal{C}}(\bm{w})$~\cite{paralleldistbook}.
With $\lambda_1 > 0$, the proximal gradient-based relaxation in \eqref{eqn:prox grad obj 1}-\eqref{eqn:prox grad obj 2} ensures strong convexity of the \textit{surrogate objective function} ${\widetilde{\mathcal{J}}}$. 
Thus, at each iteration $\ell$ of Algorithm \ref{alg:dist_NOVA main}, we formulate the relaxed convex approximation of $\widehat{\bm{\mathcal{P}}}$ at current iterate $\bm{{w}}^{(\ell)}$, i.e., $\widehat{\bm{\mathcal{P}}}_{\bm{{w}}^{(\ell)}}$, as 
\begin{align}
    & \widehat{\bm{\mathcal{P}}}_{\bm{{w}}^{(\ell)}}: \underset{\bm{w}}{\min} \quad \sum_{{d} \in \mathcal{N} \cup \mathcal{S} \cup \mathcal{B}} \small{\widetilde{\mathcal{J}}}_{{d}}(\bm{w}_{{d}};\bm{{w}}^{(\ell)}) \label{eqn: relaxed objective}\\
        & \textrm{s.t.}~~ \sum_{{d} \in \mathcal{N} \cup \mathcal{S} \cup \mathcal{B}} \small\bm{\widetilde{\mathcal{C}}}_{\mathsf{d}}(\bm{w}_{{d}};\bm{{w}}^{(\ell)}) \leq 0  \label{eqn: non-convex relaxed constraints}\\
     &~~~~ \displaystyle \sum_{{d} \in \mathcal{N} \cup \mathcal{S} \cup \mathcal{B}} \bm{\mathcal{G}}(\bm{{w}}_{{d}}) = 0  \label{eqn: convex separable constraints 2}\\
     & ~~~~\bm{{\mathcal{D}}}_{{d}}(\bm{w}_{{d}}) \leq 0, ~ {d} \in \mathcal{N} \cup \mathcal{S} \cup \mathcal{B}. \label{eqn: convex constraints 2}
\end{align}
\begin{algorithm}[!t]
	\caption{Iterative Distributed Primal Dual Method (\underline{PD \texttt{CE-FL}})}  \label{alg:dist_opt_algo main}
    \begin{algorithmic}[1]
           \STATE \textcolor{black}{\textbf{Input:} Initialize $\bm{\Lambda}^{[0]} \geq \mathbf{0}, \bm{\Omega^{[0]}}$}
           \STATE \textcolor{black}{\textbf{Output:} Final iterates of $\{\bm{\Lambda}$, $\bm{\Omega} \}$}
           \STATE  \textcolor{black}{Iteration count $i = 0$}
           \WHILE{$\{\bm{\Lambda}_{{d}}$, $\bm{\Omega}_{{d}} \}$ not converged $\forall {d} \in \mathcal{N} \cup \mathcal{S} \cup \mathcal{B}$}
           \FOR{${d} \in \mathcal{N} \cup \mathcal{S} \cup \mathcal{B}$ parallely}
           \STATE Obtain  $\hat{\bm{w}}_{{d}}^{[i]}({\bm{{w}}^{(\ell)}})$ via gradient projection method on~\eqref{eqn: Primal Problem} \%\% primal descent
           \STATE Obtain $\{ \bm{\Lambda}^{[i]}_{{d}},  \bm{\Omega}^{[i]}_{{d}}\}$ using \eqref{eqn: Lambda update Local}-\eqref{eqn: combined_consensus_update} \%\% dual ascent
           \ENDFOR
           \STATE $i \longleftarrow i + 1$
           \ENDWHILE
    \end{algorithmic}
\end{algorithm}
The convex constraint \eqref{eqn: convex constraints 2} is separable across the nodes. However, \eqref{eqn: non-convex relaxed constraints}-\eqref{eqn: convex separable constraints 2} are in the form of summation  over nodes, which does not let the problem to be trivially distributed. To facilitate a parallel distributed solution to $\widehat{\bm{\mathcal{P}}}_{\bm{{w}}^{(\ell)}}$, we perform Lagrangian-based dualization 
for constraints \eqref{eqn: non-convex relaxed constraints}-\eqref{eqn: convex separable constraints 2} as 
\begin{align}
    \mathcal{L}(\bm{w}, \bm{\Lambda}, \bm{\Omega};{\bm{{w}}^{(\ell)}}) &= \sum_{{d} \in \mathcal{N} \cup \mathcal{S} \cup \mathcal{B}} \small{\widetilde{\mathcal{J}}}_{{d}}(\bm{w}_{{d}};{\bm{{w}}^{(\ell)}}) \nonumber \\
    & \hskip -3cm + \bm{\Lambda}^\top \big(\sum_{{d} \in \mathcal{N} \cup \mathcal{S} \cup \mathcal{B}} \small\bm{\widetilde{\mathcal{C}}}_{{d}}(\bm{w}_{{d}};{\bm{{w}}^{(\ell)}}) \big) + \bm{\Omega}^\top \big(\sum_{{d} \in \mathcal{N} \cup \mathcal{S} \cup \mathcal{B}} \bm{\mathcal{G}}(\bm{{w}}_{{d}}) \big). \label{eqn:Lagrangian dual}
\end{align}
In \eqref{eqn:Lagrangian dual}, $\bm{{w}}$ and $\{\bm{\Lambda}, \bm{\Omega} \}$  are the \textit{primal} and \textit{dual} variables, respectively.  Also,  $\bm{\Lambda}$ and $\bm{\Omega}$ are the Lagrangian multipliers associated with the convexified inequality constraints $\bm{\widetilde{\mathcal{C}}}_.(., {\bm{{w}}^{(\ell)}})$ and  linear equality constraints $\bm{\mathcal{G}}(.)$, respectively. 

Based on the above formulation, we next define
  the following max-min optimization problem:
  \vspace{-1mm}
  \begin{align}
    & \nonumber \max_{{\bm{\Lambda} \geq \mathbf{0}}
    } \min_{ \bm{w}} ~{\mathcal{L}}(\bm{w}, \bm{\Lambda}, \bm{\Omega} ;{\bm{{w}}^{(\ell)}}) \\
   &  \text{s.t.}~~\bm{{\mathcal{D}}}_{\mathsf{d}}(\bm{w}_{{d}}) \leq 0, ~ {d} \in \mathcal{N} \cup \mathcal{S} \cup \mathcal{B}. \label{eqn:minmax opt problem} 
\end{align}
Note that solution to the above max-min problem is optimal for $\widehat{\bm{\mathcal{P}}}_{\bm{{w}}^{(\ell)}}$ due to its strongly-convex objective and convex constraints.
In~\eqref{eqn:minmax opt problem}, for fixed values of dual parameters  $\{\bm{\Lambda}, \bm{\Omega}\}$, the inner minimization problem can be distributed across the network elements, where each node ${d}$ aims to solve the partial Lagrangian minimization problem of the following format:
  \vspace{-1mm}
\begin{align}
   & \min_{\bm{w}_{{d}}}  \small{\widetilde{\mathcal{J}}}_{{d}}(\bm{w}_{{d}};{\bm{{w}}^{(\ell)}}) \nonumber
    + {\bm{\Lambda}^\top ~ \small\bm{\widetilde{\mathcal{C}}}_{{d}}(\bm{w}_{{d}};{\bm{{w}}^{(\ell)}})} + \bm{\Omega}^\top \bm{\mathcal{G}}(\bm{{w}}_{{d}}) \\
   &  \text{s.t.}~~\bm{{\mathcal{D}}}_{{d}}(\bm{w}_{{d}}) \leq 0. \label{eqn: minimization block}
\end{align}
With the above max-min formulation and the division of the problem into sub-problems given by
\eqref{eqn:minmax opt problem}-\eqref{eqn: minimization block}, in the following we construct a solution framework which enables parallel updates of the primal variables followed by a \textit{decentralized consensus scheme} to update the dual-variables alternately. 

\begin{algorithm}[!t]
	\caption{Iterative Decentralized Consensus Method (\underline{Consensus \texttt{CE-FL}})} \label{alg:dist_opt_consensus_algo}
    \begin{algorithmic}[1]
            \STATE \textbf{Input:} Initialization of variables $\bm{\Gamma}^{\{ 0\}}_d = [\bm{\Lambda}^{[i]}_d, \bm{\Omega}^{[i]}_d]$ at each node $d$, maximum number of consensus rounds $J$
            \STATE \textcolor{black}{\textbf{Output:} Final local iterates $\{ \bm{\Gamma}^{\{ J\}}_d \}_{{d} \in \mathcal{N} \cup \mathcal{S} \cup \mathcal{B}}$ at convergence}
            \STATE Initialize iteration count iteration count $j = 0$
            \WHILE {$j \leq J$ }
             \FOR{${d} \in \mathcal{N} \cup \mathcal{S} \cup \mathcal{B}$ parallely}
            \STATE Receive $\{\bm{\Gamma}^{\{ j -1 \}}_{d'}\}_{({d}, {d}') \in {\mathcal{E}}}$
            \STATE Communicate $\bm{\Gamma}^{\{ j -1 \}}_{d}$ to neighbors $\{{d}': ({d}, {d}') \in {\mathcal{E}} \}$ 
            \vspace{-3mm}
            \STATE Obtain $\bm{\Gamma}^{\{ j \}}_{d}$ using consensus update rule \eqref{eqn: consensus update equation}
            \ENDFOR
             \STATE $j \longleftarrow j + 1$
            \ENDWHILE
    \end{algorithmic}
\end{algorithm}
\textbf{Iterative Distributed Primal-Dual Algorithm with Decentralized Consensus.} The pseudo-code of our \textit{iterative} distributed primal-dual method is given in Algorithm \ref{alg:dist_opt_algo main}. We first initialize the dual variables as $\{\bm{\Lambda}^{[0]}, \bm{\Omega}^{[0]}\}$. 
Then, during each iteration $i$, for current estimates of dual variables $\{\bm{\Lambda}^{[i-1]}, \bm{\Omega}^{[i-1]}\}$, we first highlight that the partial Lagrangian based minimization subproblem, i.e., \eqref{eqn: minimization block}, has a convex objective function. Also, the constraint 
$\bm{{\mathcal{D}}}_{{d}}(\bm{w}_{{d}}) \leq 0$ consists of box  
and polyhedrons constraints given by  \eqref{rhoUE-BS}-\eqref{eqn:delay_reception_opt_main}, \eqref{con22add}-\eqref{con3add}, which are closed convex projection sets. Hence, we leverage gradient projection algorithm (GPA) \cite{GPAarticle} and obtain the solution of the primal variable minimization subproblems given by \eqref{eqn: minimization block} for each individual node ${d}$ at each iteration $i$ as
    \begin{align}
    \hspace{-4mm}
    \hat{\bm{w}}_{{d}}^{[i]}({\bm{{w}}^{(\ell)}})  & = \argmin_{\bm{w}_{{d}}: \bm{\mathcal{D}}_{{d}}(\bm{w}_{{d}})\leq \mathbf{0} }  {\mathcal{L}}(\bm{w}_{{d}}, \bm{\Lambda}^{[i-1]}_{{d}}, \bm{\Omega}^{[i-1]}_{{d}} ;{\bm{{w}}^{(\ell)}}). \label{eqn: Primal Problem}\hspace{-4mm}
\end{align}
Upon obtaining $\{\hat{\bm{w}}_{{d}}^{[i]}({\bm{{w}}^{(\ell)}})\}_{{d} \in \mathcal{N} \cup \mathcal{S} \cup \mathcal{B}}$, we perform gradient ascent updates on dual variables $\{\bm{\Lambda}, \bm{\Omega}\}$. These correspond to the outer maximization in \eqref{eqn:minmax opt problem}, and thus involve computation of gradient of the Lagrangian function $\mathcal{L}(\bm{w}, \bm{\Lambda}, \bm{\Omega};{\bm{w}}^{(\ell)})$ defined in \eqref{eqn:Lagrangian dual} at the current primal variable estimates $\{\hat{\bm{w}}_{{d}}^{[i]}({\bm{{w}}^{(\ell)}})\}_{{d} \in \mathcal{N} \cup \mathcal{S} \cup \mathcal{B}}$ solved distributedly via \eqref{eqn: Primal Problem}. During each iteration $i$ of Algorithm \ref{alg:dist_opt_algo main}, this update is described as
\vspace{-4mm}

{\small
\begin{align}
    &\hspace{-3mm} \bm{\Lambda}^{[i]} \hspace{-.4mm}=\hspace{-.4mm} \Bigg[\hspace{-.3mm}\bm{\Lambda}^{[i-1]} \hspace{-.3mm}+\hspace{-.5mm} \frac{\kappa~\nabla_{\bm{\Lambda} } \mathcal{L}\big(\hat{\bm{w}}_{{d}}^{[i]}({\bm{{w}}^{(\ell)}}),\hspace{-.3mm} \bm{\Lambda}^{[i-1]}\hspace{-.3mm},\hspace{-.3mm} \bm{\Omega}^{[i-1]};\hspace{-.3mm}{\bm{{w}}^{(\ell)}}\big)}{|\mathcal{N} \cup \mathcal{S} \cup \mathcal{B}|} \hspace{-.5mm}\Bigg]^{+}\hspace{-2mm} \nonumber \\
    &~ =\hspace{-.5mm} \Bigg[\hspace{-.3mm}\bm{\Lambda}^{[i-1]} \hspace{-.3mm}+ \hspace{-.3mm}\frac{\kappa}{|\mathcal{N} \cup \mathcal{S} \cup \mathcal{B}|} \underbrace{\hspace{-.3mm}\displaystyle\sum_{{d} \in \mathcal{N} \cup \mathcal{S} \cup \mathcal{B}}\hspace{-1.2mm} \small\bm{\widetilde{\mathcal{C}}}_{{d}}\big(\hat{\bm{w}}_{{d}}^{[i]}({\bm{{w}}^{(\ell)}}); {\bm{{w}}^{(\ell)}} \big)\hspace{-.3mm}}_\text{(a)}\hspace{-.9mm}\Bigg]^{+}\hspace{-2.5mm}, \label{eqn: Lambda update}\hspace{-4mm}\\
    &\nonumber \hspace{-3mm} \bm{\Omega}^{[i]} = \bm{\Omega}^{[i-1]} + \frac{\varepsilon~\nabla_{\bm{\Omega} }\mathcal{L}(\hat{\bm{w}}_{{d}}^{[i]}({\bm{{w}}^{(\ell)}}), \bm{\Lambda}^{[i-1]}, \bm{\Omega}^{[i-1]};{\bm{{w}}^{(\ell)}})}{|\mathcal{N} \cup \mathcal{S} \cup \mathcal{B}|} \\
    &\hspace{-3mm} ~~~~~ = \bm{\Omega}^{[i-1]} + \frac{\varepsilon}{|\mathcal{N} \cup \mathcal{S} \cup \mathcal{B}|} \underbrace{\displaystyle\sum_{{d} \in \mathcal{N} \cup \mathcal{S} \cup \mathcal{B}} \bm{\mathcal{G}}\big(\hat{\bm{w}}_{{d}}^{[i]}({\bm{{w}}^{(\ell)}})\big)}_\text{(b)}
    \label{eqn: Omega update},\hspace{-4mm}
\end{align}
}
\vspace{-4mm}

\noindent where $\kappa$ and $\varepsilon$ are the step sizes. 
(a) and (b) in \eqref{eqn: Lambda update} and \eqref{eqn: Omega update} cannot be computed directly due to the need for a central processor; however, we desire to update $\{\bm{\Lambda}, \bm{\Omega} \}$ locally to obtain a distributed solution. To this end, we conduct updates of dual variables via a \textit{decentralized consensus scheme},
where the dual-ascent update consists of two steps. First, local copies of the dual variables get updated at each node $d$ of the network 
\vspace{-5mm}
\begin{align}
    & \bm{\Lambda}^{[i]}_{{d}} = \bm{\Lambda}^{[i-1]}_{{d}} + \kappa ~\small\bm{\Tilde{\mathcal{C}}}_{{d}}(\hat{\bm{w}}_{{d}}^{[i]}({\bm{{w}}^{(\ell)}}); {\bm{{w}}^{(\ell)}} )\label{eqn: Lambda update Local},\\
     & \bm{\Omega}^{[i]}_{{d}} = \bm{\Omega}^{[i-1]}_{{d}} + \varepsilon~ \small\bm{{\mathcal{G}}}_{{d}}(\hat{\bm{w}}_{{d}}^{[i]}({\bm{{w}}^{(\ell)}})).
    \label{eqn: Omega update Local}
\end{align}
\vspace{-5mm}

\noindent Then, the nodes exchange their local copies with neighboring nodes to form a consensus on the dual variables as follows:
\vspace{-4mm}

{\small
\begin{equation}
\hspace{-2mm}
     [\bm{\Lambda}^{[i]}_{{d}}, \bm{\Omega}^{[i]}_{{d}}]  = \underline{\textrm{Consensus~\texttt{CE-FL}}}({d}, [\bm{\Lambda}^{[i]}_{{d}}, \bm{\Omega}^{[i]}_{{d}}]) ,
     ~ \bm{\Lambda}^{[i]}_{{d}} = \big[ \bm{\Lambda}^{[i]}_{{d}} \big]^+\hspace{-1mm}. 
     \label{eqn: combined_consensus_update}
    \hspace{-2mm}
\end{equation}}
\vspace{-4mm}

The alternating optimization of primal and dual variables described by \eqref{eqn: Primal Problem}, \eqref{eqn: Lambda update Local}-\eqref{eqn: combined_consensus_update} continues until convergence of sequence $\{\bm{\Lambda}_{{d}}$, $\bm{\Omega}_{{d}} \}_{{d} \in \mathcal{N} \cup \mathcal{S} \cup \mathcal{B}}$.
The \underline{Consensus \texttt{CE-FL}} in~\eqref{eqn: combined_consensus_update}  is described in Algorithm \ref{alg:dist_opt_consensus_algo}, which relies on local message exchange across neighboring nodes, which we discuss next. 

\textbf{Decentralized Network-Wide Consensus (Algorithm \ref{alg:dist_opt_consensus_algo}).} 
We first describe the communication graph used to synchronize the dual variables across the nodes. We consider a 0-1 edge connection between the nodes, where two nodes either engage in sharing optimization variables or do not communicate. This communication graph is formed before the ML model training, data offloading, and parameter offloading processes happen and is only used to come up with a distributed solution for the network optimization problem to orchestrate the network elements during ML model training. 
We consider a bi-level hierarchical structure for this communication graph   (see Fig.~\ref{fig:1}) wherein the UEs can possibly perform D2D communications with other UEs in their vicinity as well as uplink-downlink communications with at least one BSs (no direct communication link between DCs and UEs is assumed). We also assume that each BS is capable of communication to at least a DC, but it may or may not be communicating to other BSs.


Let ${\mathcal{H}}=({\mathcal{V},\mathcal{E}})$ denote this communication graph with vertex set $\mathcal{V}=\mathcal{N}\cup \mathcal{B} \cup\mathcal{S}$  and edge set $\mathcal{E}$, where $(d,d')\in\mathcal{E}$ implies the communication between two nodes $d,d'$ during the computation of solution of the network optimization. Also, let $\bm{{A}}=[A_{d,d'}]_{d,d' \in \mathcal{N} \cup \mathcal{S} \cup \mathcal{B}}$
denote its adjacency matrix, where
   $ {{A}}_{{d}, {d}'} = 1, ~\forall({d}, {d}') \in {\mathcal{E}}, 
    {{A}}_{{d}, {d}^{'}} = 0 ~\forall({d}, {d}^{'})  \notin {\mathcal{E}}.$
Since there is no connection between the UEs and DCs, we have
$
     {{A}}_{n, s} = 0, ~\forall s \in \mathcal{S},n\in\mathcal{N}.
    $
    Also, since each UE is connected to at least one BS, we have $ \exists b \in \mathcal{B}: ~{{A}}_{n, b} = 1,~\forall n\in\mathcal{N}$.
Finally, since each BS engages in variable transfer to at least one DC:
$
    \exists s \in \mathcal{S}:  {{A}}_{b, s} = 1,~\forall b\in\mathcal{B}.
$
Also, we assume that each DC is at least connected to another DC:
$
   \exists s'\in\mathcal{S}: {{A}}_{s, s'} = 1, ~\forall s \in \mathcal{S}$.

We next describe the consensus procedure performed over $\mathcal{H}$ to locally update the dual variables in \eqref{eqn: Lambda update}-\eqref{eqn: Omega update}. Our procedure is described in Algorithm \ref{alg:dist_opt_consensus_algo}. For notation simplicity, we represent the dual variables via a vector $\bm{\Gamma} = [\bm{\Lambda}, \bm{\Omega}]$. At iteration $i$ of \underline{PD \texttt{CE-FL}} (Algorithm \ref{alg:dist_opt_algo main}) the call of \underline{Consensus \texttt{CE-FL}} subroutine (Algorithm \ref{alg:dist_opt_consensus_algo}) via \eqref{eqn: combined_consensus_update} is triggered with initialization
$ \bm{\Gamma}^{\{ 0\}}_d = [\bm{\Lambda}^{(i)}_d, \bm{\Omega}^{(i)}_d]   .$
Afterward, at each iteration $j$ of Algorithm \ref{alg:dist_opt_consensus_algo}, 
each network node ${d} \in \mathcal{N} \cup \mathcal{S} \cup \mathcal{B} $ sends its current local value of dual variables, i.e., $\bm{\Gamma}^{\{ j - 1\}}_{d}$, and in turn receives the value of which from its neighbors, i.e., $\{\bm{\Gamma}^{\{ j -1 \}}_{d'}\}_{({d}, {d}') \in {\mathcal{E}}}$. 
Subsequently, the following update is executed at each network node ${d}$:
\vspace{-1mm}
\begin{align}
    \bm{\Gamma}^{\{ j\}}_{d} = \bm{\Gamma}^{\{ j - 1\}}_{d} {{W}}_{{d}, {d}}+\displaystyle \sum_{d':{({d}, {d}') \in {\mathcal{E}}}} \bm{\Gamma}^{\{ j - 1\}}_{d'} {{W}}_{{d}, {d}'}, \label{eqn: consensus update equation}
\end{align}
\vspace{-4mm}

\noindent where ${{W}}_{{d}, {d}'}$ is the \textit{weight} that node $d$ assigns to its neighbor $d'$.
It is important to construct ${{W}}_{{d}, {d}'}$ for all pairs of neighboring nodes $(d, d')$ so that the local estimates of the dual variables asymptotically attain their global values.
Letting $\textrm{degree}(d)$ denote the degree of node $d$, we consider  the {consensus weights} as ${{W}}_{{d}, {d}}=1-z\times \textrm{degree}(d)$, and  ${{W}}_{{d}, {d'}}=z$  $\forall ({d}, {d}') \in {\mathcal{E}}$, where $z$ is a constant satisfying $z< \frac{1}{\max_{d\in \in \mathcal{N} \cup \mathcal{S} \cup \mathcal{B} ~\textrm{degree}(d)}}$, which is proven to show fast convergence for consensus~\cite{xiao2004fast}. With the knowledge of $z$ (e.g., trivially chosen as $z=\frac{1}{|\mathcal{N} \cup \mathcal{S} \cup \mathcal{B}|}-\hat{z}$ for a small $\hat{z}>0$), the consensus weights can be distributedly obtained. 
We next study the convergence of our optimization solver 
(proof provided in Appendix~\ref{NOVA stationary proof}).

\begin{figure*}[t]
\vspace{-4mm}
\centering
\hspace{-8mm}
\hfil
\subfloat[\centering{Data distribution}]{\includegraphics[scale = 0.27]{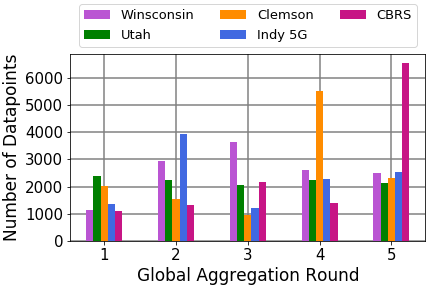}
\label{fig:active_data_dist_plot}}
\hspace{-8mm}
\hfil
\subfloat[\centering{E2E data rate distribution}]{\includegraphics[scale = 0.27]{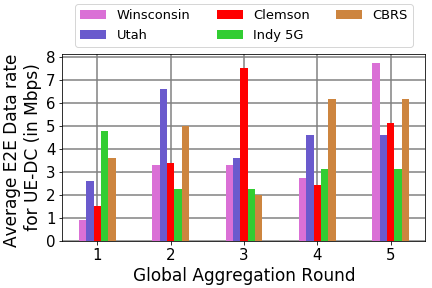}
\label{fig:active_data_rate_plot}}
\hspace{-1mm}
\hfil
\hspace{-5mm}\subfloat[Aggregator switching 
]{ \includegraphics[scale = 0.27]{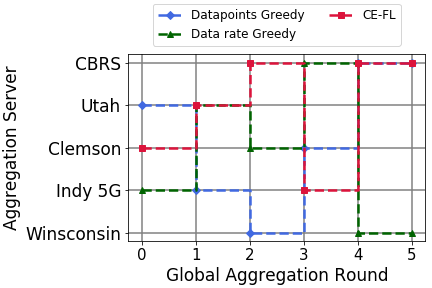}
\label{fig:active_switching}}
\vspace{-1mm}
\caption{{\color{black}Datapoint distribution, E2E data rate distribution, and \texttt{CE-FL} aggregation DC switching pattern against baseline methods.}}
\label{fig:active_selection_plots_1}
\vspace{-.1mm}
\end{figure*}

\vspace{-1.5mm}
\begin{theorem}[Convergence of the Optimization Solver] If $J\rightarrow \infty$ (see Algorithm~\ref{alg:dist_opt_consensus_algo}), the sequence $\{\bm{w}^{(\ell)}\}$ generated by Algorithm \ref{alg:dist_NOVA main} is feasible for $\bm{\mathcal{P}}$ and non-increasing, which asymptotically reaches a stationary solution of $\bm{\mathcal{P}}$.
\end{theorem}

{\color{black} In Appendix~\ref{App:complexity}, we provide the complexity analysis for our~{\tt CE-FL} methodology.}

\begin{table}[b]
\caption{{\color{black}Energy consumption comparison across varying target accuracies.}} 
 \vspace{-1.4mm}
{\scriptsize
\begin{tabularx}{0.5\textwidth}{c *{6}{Y}}
\toprule[.2em]
\multirow{0}{*}{} & \multicolumn{3}{c}{\bf{F-MNIST} (Target Acc.)} & \multicolumn{3}{c}{\bf{CIFAR-10} (Target Acc.)} \\
\cmidrule(lr){2-4} \cmidrule{5-7}
& \bf{60 \%} & \bf{70 \%} & \bf{80 \%} & \bf{40 \%} & \bf{50 \%} & \bf{60 \%}\\
\midrule
{\tt CE-FL} (in KJ) & 19.5 & 32.8  & 47.3  & 17.3 & 26.6 & 42.7 \\
FedNova (in KJ) & 23.4 & 49.2 & 78.7  & 24.7 &  35.1 & 60.3 \\
FedAvg (in KJ) & 27.9 & 57.1 & 83.2  & 29.1 &  39.0 & 67.8 \\
\midrule

 vs FedNova Savings (\%) & 16.7 & 33.3 & 39.9 & 30 & 24.2 & 29.2 \\
 vs FedAvg Savings (\%) & 30.1 & 42.6 & 43.1 & 40.5 & 31.8 & 37.0 \\
\bottomrule
\end{tabularx}

}
\label{tab:ML_energy_table}
\end{table}
\vspace{-5mm}
\section{Numerical Evaluations}\label{sec:NumExp}
\noindent In this section, we first describe the simulation setup (Sec.~\ref{Sec:Sim1}) and then discuss the results (Sec.~\ref{Sec:Sim2}).


\vspace{-4mm}
\subsection{Simulations Setup and Testbed Configuration}\label{Sec:Sim1}
We acquire realistic models of communication models and the units' power consumption through data gathering from 5G/4G and CBRS network testbeds that include BSs, UEs, and DCs. 
The data collection technique is explained in Appendix \ref{5G appendix}. The DCs for the 5G/4G data collection are located at the Indy 5G Zone \cite{indy5gzone}, Discovery Park District \cite{cbrsconvergence}, Wisconsin, Utah, and Clemson, respectively \cite{cloudlab}. Following that, a larger default network setting was generated for the numerical evaluations by post-processing the measured data (see Appendix \ref{5G scaled up dataset creation}). \textcolor{black}{
The capacity of DCs is chosen $R_{s}^{\mathsf{max}}\in [40,50] \text{Gbps}$, $\forall s$, and capacity of each BS-DC link is chosen as $R_{b,s}^{\mathsf{max}}\in [3,4]\text{Gbps} $, $\forall b,s$.
} The created dataset consists of $20$ UEs, $10$ BSs, and $5$ DCs, with each server along with $2$ BSs and $4$ UEs comprising a \textit{sub-network}.  Each sub-network is characterized by high intra-network and low inter-network data transfer rates.
\begin{figure}[t]
\vspace{-4mm}
\centering
\hspace{-8mm}
\hfil
\subfloat[\centering{Delay incurred}]{\includegraphics[scale = 0.27]{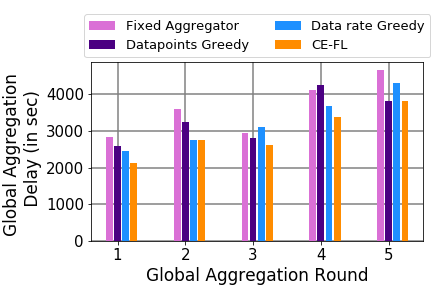}
\label{fig:active_delay}}
\hspace{-1mm}
\hfil
\hspace{-5mm}\subfloat[Energy consumption 
]{ \includegraphics[scale = 0.262]{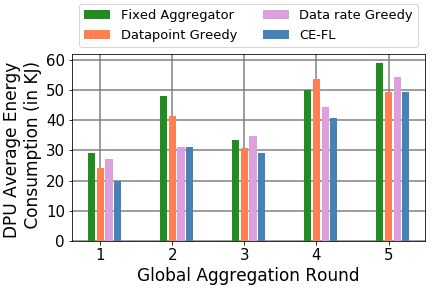}
\label{fig:active_energy}}
\vspace{-1mm}
\caption{{\color{black}Delay and energy comparison between \texttt{CE-FL} and the baselines.}} \label{fig:active_selection_plots}
\vspace{-.2mm}
\end{figure}
 \begin{table}[b]
\caption{{\color{black}Model training delay comparison across varying target accuracies.}}
{\scriptsize
\begin{tabularx}{0.50\textwidth}{c *{6}{Y}}
\toprule[.2em]
\multirow{0}{*}{} & \multicolumn{3}{c}{\bf{F-MNIST} (Target Acc.)} & \multicolumn{3}{c}{\bf{CIFAR-10} (Target Acc.)} \\
\cmidrule(lr){2-4} \cmidrule{5-7}
& \bf{60 \%} & \bf{70 \%} & \bf{80 \%} & \bf{40 \%} & \bf{50 \%} & \bf{60 \%}\\
\midrule
{\tt CE-FL} (in sec) & 2721.2 & 3521.4 & 4577.8 & 2850.7 & 3988.6  & 4912.5 \\
FedNova (in sec) & 3035.2 &  4359.9 & 5432.6    & 3271.5 & 4729.9 & 5873.4 \\
FedAvg (in sec) & 3721.6 &  4882.7 & 6156.3    & 4022.1 & 5623.9 & 6433.8 \\
\midrule
 vs FedNova Savings (\%) & 10.3 &  19.2 & 15.7 & 12.9 & 15.7 & 16.4 \\
 vs FedAvg Savings (\%) & 26.9 & 27.9 & 25.6 & 29.1 & 29.1 & 23.7 \\
\bottomrule
\end{tabularx}
}
\label{tab:ML_delay_table}
\end{table}
\vspace{-4mm}
\subsection{Results and Discussion}\label{Sec:Sim2}

The existence of several variables and their 
coupled behavior makes the analysis of $\bm{\mathcal{P}}$ hard. Thus, we perform an \textit{ablation study} on $\bm{\mathcal{P}}$ by isolating different sets of optimization variables to show the behavior of $\bm{\mathcal{P}}$ under different network settings. The default network setting is described in Appendix \ref{Network Settings for experiments}.   
\subsubsection{Performance of \texttt{CE-FL} for Dynamic ML Model Training}
We compare the ML model performance of \texttt{CE-FL} in terms of energy consumption and delay against FedNova~\cite{wang2020tackling} in terms of classification accuracy obtained on Fashion-MNIST \cite{xiao2017/online} and CIFAR-10 \cite{krizhevsky2009learning} datasets in Table \ref{tab:ML_energy_table} and \ref{tab:ML_delay_table} respectively {(See Appendix \ref{Network Settings for experiments} for the description of the datasets)}. {\color{black}FedNova is an advanced implementation of FedAvg algorithm and has been shown to outperform FedAvg udner non-uniform SGD iterations across ML data processing units~\cite{wang2020tackling}.} We consider time-varying datasets at the UEs, where at each global aggregation round, UEs acquire datasets with sizes distributed according to normal distribution with mean $2000$ and variance $200$. 
{We note that ML training under FedNova was performed with average CPU/data-processing frequencies, mini-batch sizes and number of SGD iterations at the DPUs.} 
{\color{black}Tables \ref{tab:ML_energy_table}, \ref{tab:ML_delay_table} demonstrate the energy and delay savings that \texttt{CE-FL} obtains against FedNova and AedAvg upon reaching different classification accuracies.
We observe that \texttt{CE-FL} outperforms FedNova, which in turn beats FedAvg, in terms of network delay and energy costs for chosen target classification accuracies on aforementioned datasets. }

\subsubsection{Floating aggregation point in {\tt CE-FL}}
We examine how time varying and unequal data distributions at the DPUs control the aggregation DC throughout the ML training in Fig.~\ref{fig:active_selection_plots_1}. 


{\color{black} For the experiments, we consider two greedy baselines: (i) datapoint greedy, and (ii) data rate greedy. The datapoint greedy strategy chooses the DC whose \textit{subnetwork} has the highest concentration of datapoints at each global aggregation as the floating aggregator. 
Also, data rate greedy method designs a strategy based on the end-to-end (E2E) data transfer rates between UEs and the DCs.
Mathematically, we define the E2E data-rate between arbitrary UE $n$ and DC $s$ as
\begin{align}
    R_{n,s}^{\mathsf{E2E}, (t)} = \underset{b \in \mathcal{B}}{\max} \left\{\frac{1}{\frac{1}{R^{(t)}_{n,b}} + \frac{1}{R^{\mathsf{max}}_{b,s}}} \right\}.
\end{align}
Then, at each round of global aggregation, the DC with the highest average E2E data rate across all the UEs is chosen as the floating aggregator DC.
We depict the evolution of datapoint concentration and average E2E data rates at the DC across global aggregation rounds in Figs. \ref{fig:active_data_dist_plot} and \ref{fig:active_data_rate_plot}, and how the choice of the aggregator varies in the greedy strategies in comparison to {\tt CE-FL} in Fig. \ref{fig:active_switching}.
Comparing the data distribution across the network depicted in Fig.~\ref{fig:active_data_dist_plot} and the optimal aggregator selected in {\tt CE-FL} in~Fig.~\ref{fig:active_switching} reveals an interesting phenomenon. In particular, the optimal aggregator selection of {\tt CE-FL} matches that of the {datapoint greedy method} (i.e., selecting the DC in the area with the highest data concentration)
 when data concentrations are extremely skewed across the network (i.e., $t = 5$). A similar takeaway can be obtained via comparing Fig. \ref{fig:active_data_rate_plot} and \ref{fig:active_switching}, where {\tt CE-FL} only matches that of the data rate greedy strategy (at $t=2$) when the E2E data rate is skewed toward Utah DC, which also has a high data concentration.}


 However, Fig.~\ref{fig:active_switching} reveals that {\tt CE-FL} does not always favor the DCs with the highest data concentration (i.e., $t \in \{ 1,2,3,4\}$). This is due to the fact that congestion across the links and heterogeneity of the network elements in terms of computation and proximity are further considered in active aggregation DC selection in~$\bm{\mathcal{P}}$. 

We then investigate the energy and delay savings obtained under \texttt{CE-FL} active aggregator selection against fixed aggregator strategy (averaged over the 5 DCs) and the greedy method. Fig.~\ref{fig:active_delay}, \ref{fig:active_energy} confirms the efficacy of the active selection paradigm in {\tt CE-FL} in terms of network resource savings,  highlighting the need to jointly taking into account   the heterogeneities of network elements, congestion of the links, and uneven data concentrations as in {\tt CE-FL} to select the aggregation DC.

\subsubsection{Impact of model drift on the behavior of {\tt CE-FL}\label{subsection: drift delay}}
Fig.~\ref{fig:model_drift_plots} reveals how time varying model drift dictates the delay of  conducting global aggregation rounds and the frequency of data processing at the UEs. It can be seen that increasing the model drift  results in reduced global aggregation delays and faster data processing. This implies that when the temporal variations of the local datasets at the DPUs is large, to have an adequate ML model for the online datasets at the devices, {\tt CE-FL} promotes rapid global aggregations and fast data processing. 
\begin{figure}[t]
\vspace{-4mm}
\centering
\hspace{-8mm}
\hfil
\subfloat[\centering{Drift vs. Delay Trade-off}]{\includegraphics[scale = 0.27]{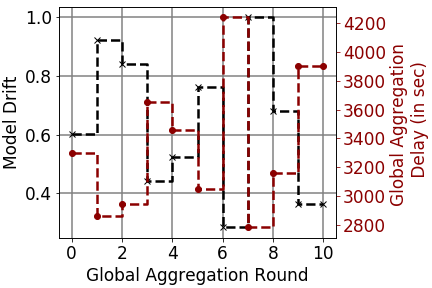}
\label{fig:psi_fn_of_sigma}}
\hspace{-1mm}
\hfil
\hspace{-5mm}\subfloat[Drift vs. CPU frequency Trade-off 
]{ \includegraphics[scale = 0.27]{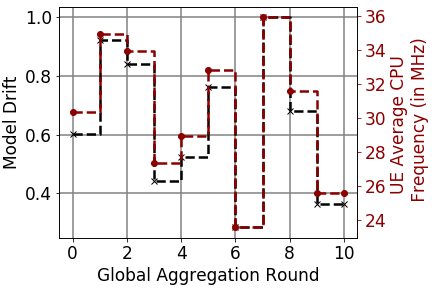}
\label{fig:model_drift_cpu_freq_plot}}
\vspace{-1.5mm}
\caption{Impact of model drift on system behavior.} \label{fig:model_drift_plots}
\vspace{-2mm}
\end{figure}

\begin{figure}[t]
\centering
\hspace{-8mm}
\hfil
\subfloat[\centering{ML loss weight vs. average mini-batch size of SGD}]{\includegraphics[scale = 0.27]{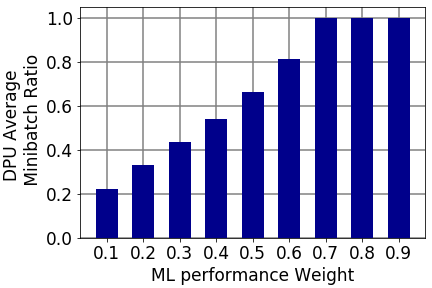} \label{fig: ml_importance_minibatch_plot}}
\hspace{-1mm}
\hfil
\hspace{-5mm}\subfloat[\centering{ML loss weight vs. average energy consumption of DPUs}
]{ \includegraphics[scale = 0.27]{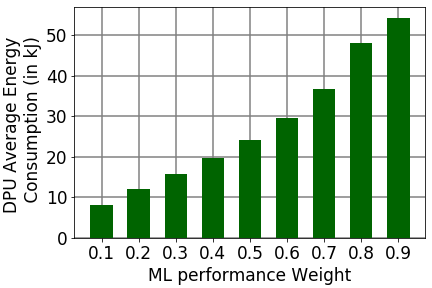} \label{fig: ml_importance_energy_plot}
}
\vspace{-1.5mm}
\caption{Impact of ML loss importance on the system.} \label{fig: ml_importance_plots}
\vspace{-1mm}
\end{figure}


\subsubsection{Impact of ML performance weight on local model training\label{subsection: ML weight tradeoff}}
We characterize the impact of ML performance weight  (i.e., $\xi_1$ in the objective of $\bm{\mathcal{P}}$) on the mini-batch ratios and the energy consumption at the DPUs in Fig. \ref{fig: ml_importance_plots}. As can be seen from the two subplots, increasing the ML performance importance results in an increase in the SGD mini-batch ratios to have more accurate local models and consequently increases the energy consumption at the DPUs. This implies that in applications where the accuracy of the ML model is of particular importance, {\tt CE-FL} will sacrifice resource savings to obtain an ML model with a better accuracy.


\subsubsection{Decentralized network optimization solver\label{subsection: opt solver performance}}
We first develop the centralized solver for $\bm{\mathcal{P}}$ and investigate the performance of our  decentralized solver against it for a network with $|\mathcal{N}| = 20$, $|\mathcal{B}| = 10$ and $|\mathcal{S}| = 5 $. The centralized solver solves $\bm{\mathcal{P}}$ via Algorithm \ref{alg:dist_NOVA main} and Algorithm \ref{alg:dist_opt_algo main} while assuming the knowledge of all the intrinsic parameters of all the UEs, DCs and BSs used in $\bm{\mathcal{P}}$. Thus, the centralized solver performs global dual updates \eqref{eqn: Lambda update}, \eqref{eqn: Omega update} without the requirement \underline{Consensus \texttt{CE-FL}} (Algorithm  \ref{alg:dist_opt_consensus_algo}). The comparisons are depicted in Fig.~\ref{Fig:centvsDec}. As can be seen from Fig.~\ref{fig:consensus_gap_plot}, our distributed solver has a comparable performance to the centralized counterpart under various consensus rounds $J \in \{ 10, 50, 70\}$ (see Algorithm \ref{alg:dist_opt_consensus_algo}); as the number of consensus rounds increases the gap between the performance of the two narrows.
   We next demonstrate the performance of our distributed solver under varying network sizes while keeping number of consensus rounds $J=30$ in Fig.~\ref{fig:network_size_ablation_plot}. We vary the number of UEs as $|\mathcal{N}| \in \{ 10, 15, 20, 30\}$ while fixing  $|\mathcal{B}| = 10$ and $|\mathcal{S}| = 5 $. From Fig.~\ref{fig:network_size_ablation_plot}, it can be seen that increasing the size of the network indeed improves the solution of the solver due to processing larger number of datapoints across the DPUs leading to a better ML performances. However, \ref{fig:network_size_ablation_plot} also highlights the shrinking gains of increasing the number of UE devices, where an initial increase in  $|\mathcal{N}| = 10$ to $|\mathcal{N}| = 15$ results in significant performance enhancement; whereas, further increase in the number of UEs doesn't translate to notable improvements. This indicates that as the number of DPUs are increased beyond a certain limit, the ML performance gains are obscured by the larger cumulative network energy consumption and network delays attributed to data offloading, ML processing and aggregations.

\begin{figure}[t]
\vspace{-4mm}
\centering
\hspace{-8mm}
\hfil
\subfloat[\centering{Centralized vs. decentralized solution for different number of consensus rounds}]{ \includegraphics[scale = 0.26]{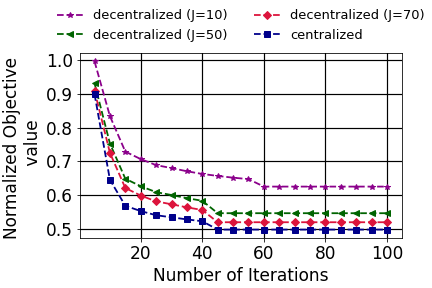} \label{fig:consensus_gap_plot}}
\hspace{-1mm}
\hfil
\hspace{-4mm}\subfloat[\centering{Decentralized algorithm convergence for different number of UEs in the system}
]{ \includegraphics[scale = 0.26]{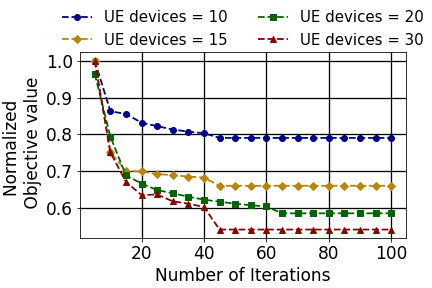} \label{fig:network_size_ablation_plot}
}
\vspace{-1.5mm}
\caption{Performance of our decentralized optimization solver.\label{Fig:centvsDec}}
\end{figure}

\vspace{-3mm}
\section{Conclusion and Future Work} \label{sec:conclusion}
\noindent We proposed {\tt CE-FL}, which presumes a scenario in which the data processing for ML model training occurs simultaneously across the DCs and the UEs, which is enabled via offloading a part of the local datasets of the UEs to the DCs through the BSs. {\tt CE-FL} further assumes a realistic scenario in which the number of datapoints and the data distribution across the UEs varies over time and incorporates the concept of floating aggregation DC to the distributed ML model training. We analytically characterized  the ML performance of  {\tt CE-FL} and obtained its convergence behavior. We then formulated network-aware {\tt CE-FL} as an optimization problem, which will lead to a joint load balancing across the UEs and DCs and efficient data routing across the   network hierarchies. We then developed a distributed optimization solver to solve our formulated problem via drawing a connection between consensus-based optimization techniques and surrogate approximation methods. For the future work, studying the performance of {\tt CE-FL} under device dropouts, link failures, and asynchronous model transfers can be considered. Also, we have done some experimental works in~\cite{nguyen2021fly} to quantify the effect of varying aggregator in a decentralized setting under proof-of-work-based metrics, the extension of which to concretized formulations is open.

\vspace{-3mm}
\section{Acknowledgments}
The authors would like to thank Hyoyoung Lim for her help in data collection from 5G testbed.
\vspace{-2mm}
\bibliographystyle{IEEEtran}
\bibliography{references}
\vspace{-14mm}
\begin{IEEEbiographynophoto}{Bhargav Ganguly (S'22)} is a Ph.D. student at Purdue University. He received his M.Tech. in EE from IIT Kanpur in 2019.
\end{IEEEbiographynophoto}
\vspace{-13.5mm}
\begin{IEEEbiographynophoto}{Seyyedali Hosseinalipour (M'20)}
received his Ph.D. in EE from NCSU in 2020. He is currently an Assistant Professor at  University at Buffalo (SUNY). 
\end{IEEEbiographynophoto}
\vspace{-13.5mm}
\begin{IEEEbiographynophoto}{Kwang Taik Kim (SM'22)} received his Ph.D. in ECE from Cornell University in 2008. He is currently a Research Assistant Professor at Purdue University. 
\end{IEEEbiographynophoto}
\vspace{-13.5mm}
\begin{IEEEbiographynophoto}{Christopher G. Brinton (SM'20)}
is an Assistant Professor of ECE at Purdue University. He received his Ph.D. in EE from Princeton University in 2016.
\end{IEEEbiographynophoto}
\vspace{-13.5mm}
\begin{IEEEbiographynophoto}{Vaneet Aggarwal (SM'15)} received his Ph.D. in EE from Princeton University in 2010. He is currently a Professor at Purdue University.
\end{IEEEbiographynophoto}
\vspace{-13.5mm}
\begin{IEEEbiographynophoto}{David Love (F'15)} is Nick Trbovich Professor of ECE at Purdue University. He received his Ph.D. in EE from University of Texas at Austin in 2004.
\end{IEEEbiographynophoto}
\vspace{-13.5mm}
\begin{IEEEbiographynophoto}{Mung Chiang (F'12)} is the John A. Edwardson Dean of the College of Engineering, Roscoe H. George Distinguished Professor of ECE, and Executive Vice President of Purdue University. He received his Ph.D. from Stanford University in  2003.
\end{IEEEbiographynophoto}

\newpage
\onecolumn
\appendices

\section{Notations Used in the Proofs}
For the ease of our analysis and presentation of results, we define the following quantities:
\begin{align}
    & p_{i}^{(t)} = \frac{ D_i^{(t)}}{D^{(t)}}, \label{eqn: ML data ratio definition} \\
    & a_{{i,\ell}}^{(t)} = (1-\eta\mu)^{\gamma_{i}^{(t)}-1-\ell}, \\
    & \mathbf{a}_{i}^{(t)} = \big[ a_{{i,0}}^{(t)}, \cdots , a_{{i,\gamma_{i}^{(t)} - 1}}^{(t)}\big] .
    \end{align}
    The normalized accumulated stochastic gradient and full-batch gradient for device $i$ are denoted by $\mathbf{d}_{{i}}({t})$ and $\mathbf{h}_{{i}}({t})$, respectively
\begin{align}
  & \mathbf{d}_{{i}}^{({t})} = \frac{1}{\|\mathbf{a}_{i}^{(t)}\|_{1}} \sum_{\ell = 0}^{\gamma_{i}^{(t)} - 1} a_{{i,\ell}}^{(t)} \widetilde{\nabla} F^{(t)}_{i}(\mathbf{x}_{i}^{(t, \ell)}), \\
    & \mathbf{h}_{i}^{(t)} = \frac{1}{\|\mathbf{a}_{i}^{(t)}\|_{1}} \displaystyle \sum_{\ell = 0}^{\gamma_{i}^{(t)} - 1} a_{{i,\ell}}^{(t)} \nabla F_{i}^{(t)}(\mathbf{x}_i^{(t,\ell)}).
\end{align}
At the instance of global aggregation at round $t$, the aggregator server aggregates the local models using the rule
\begin{align}
    \mathbf{x}^{(t + 1)} -\mathbf{x}{(t)}
    & = -\vartheta \displaystyle \sum_{i \in \mathcal{N} \cup \mathcal{S}} p_i^{(t)} \eta \mathbf{d}_i^{(t)} .
\end{align}
Also, we note that the global loss function we aim to minimize can be alternatively expressed as:
\begin{align}
    {F}^{(t)}(\mathbf{x}) = \displaystyle \sum_{i \in \mathcal{N} \cup \mathcal{S}} p_i^{(t)}  F_{i}^{(t)}(\mathbf{x}).
\end{align}

\section{Variance of SGD Iterations}
\begin{proposition}
During each aggregation round $t$, for each device/server $i \in \mathcal{N} \cup \mathcal{S}$, the variance of stochastic gradient during mini-batch gradient descent iteration $k$, i.e., $\widetilde{\nabla} F_i^{(t)}(\mathbf{x})$ is upper-bounded by
\begin{align}
    \mathbb{E}[ \|\widetilde{\nabla} F_i^{(t)}(\mathbf{x}) - \nabla F_{i}^{(t)}(\mathbf{x}) \|^2 ] & \leq \frac{2(1-m_{i}^{(t)})({{D}}_{i}^{(t)} -1)}{m_{i}^{(t)} {({D}}_{i}^{(t)})^2} (\Tilde{\sigma}_{i}^{(t)})^2 \Theta_i^2.
\end{align}
where $m_i^{(t)}$ is the respective minibatch size, ${{D}}_{i}^{(t)}$ is  sizes of corresponding local dataset, $(\Tilde{\sigma}_{i}^{(t)})^2$ is the sample variance of the local dataset, and $\Theta_i$ is described in Assumption \eqref{assumption: Local Data Variability}. 
\end{proposition}
\begin{proof}
At each data processing node $i \in \mathcal{N} \cup \mathcal{S}$, during $k^{th}$ local gradient descent iteration of global aggregation round $t$, the node chooses a sample of data points ${\mathcal{D}}_i^{(t,k)}$, where we have 
\begin{align}
    & {\mathcal{D}}_i^{(t,k)}\subseteq {\mathcal{D}}_i^{(t)}, \\
    & {{D}}_i^{(t,k)}=m_i^{(t)}D_i^{(t)}, ~m_i^{(t)} \in (0,1].
\end{align}
For every datapoint $\xi$ contained in the local dataset ${\mathcal{D}}_i^{(t)}$, i.e., $\xi \in {\mathcal{D}}_{i}(t)$, we define the following indicator variable capturing the inclusion of $\xi$ in  ${\mathcal{D}}_i^{(t,k)}$ as
\begin{align}
& Z_{\xi} = 1, \  \ \xi \in {\mathcal{D}}_i^{(t,k)}, \\
& Z_{\xi} = 0, \  \ \xi \notin {\mathcal{D}}_i^{(t,k)}.
\end{align}
Note that the mini-batch dataset ${\mathcal{D}}_i^{(t,k)}$ is obtained via uniform random sampling without replacement of $m_i^{(t)}$ fraction of dataset ${\mathcal{D}}_{i}(t)$. Thus, we have
\begin{align}
    & P[Z_{\xi} = 1] = \frac{{D_i^{(t)}-1 \choose {{D}}_i^{(t,k)} -1}}{{D_i^{(t)} \choose {{D}}_i^{(t,k)}}} = \frac{{{D}}_i^{(t,k)}}{{{D}}_i^{(t)}} = m_{i}^{(t)}, \\
    & \mathbb{E}[Z_{\xi}] = \mathbb{E}[Z_{\xi}^2] = m_{i}^{(t)}, \\
    & Var(Z_{\xi}) = \mathbb{E}[Z_{\xi}^2] - \mathbb{E}[Z_{\xi}]^2 = m_{i}^{(t)}(1-m_{i}^{(t)}).
\end{align}

    On the other hand, for two distinct datapoints $\xi \neq \Tilde{\xi}$, where $\xi,\Tilde{\xi} \in {\mathcal{D}}_i^{(t)}$, we have
     \begin{align}
    & \mathbb{E}[Z_{\xi} Z_{\Tilde{\xi}}] = P[Z_{\xi} = 1, Z_{\Tilde{\xi}} = 1] = P[Z_{\xi} = 1 | Z_{\Tilde{\xi}} = 1] P[Z_{\Tilde{\xi}} = 1]  = \Big(\frac{m_{i}^{(t)} {{D}}_{i}^{(t)} -1}{{{D}}_{i}^{(t)}-1}\Big) m_{i}^{(t)}, \\
    & Cov[Z_{\xi} Z_{\Tilde{\xi}}] = \mathbb{E}[Z_{\xi} Z_{\Tilde{\xi}}]- \mathbb{E}[Z_{\xi}] \mathbb{E}[Z_{\Tilde{\xi}}] = \Big(\frac{m_{i}^{(t)} {{D}}_{i}^{(t)} -1}{{{D}}_{i}^{(t)}-1}\Big) m_{i}^{(t)} - (m_{i}^{(t)})^2 = -\frac{1}{{{D}}_{i}^{(t)}-1}\Big(1-m_{i}^{(t)}\Big)m_{i}^{(t)}.
\end{align}
Henceforth, we compute the variance of the noise introduced due to mini-batching during each epoch of SGD at ML training time, i.e., $\mathbb{E}[ \|\widetilde{\nabla} F_i^{(t)}(\mathbf{x}) - \nabla F_{i}^{(t)}(\mathbf{x}) \|^2 ]$, as follows:
\begin{align}
   \nonumber \mathbb{E}[ \|\widetilde{\nabla} F_i^{(t)}(\mathbf{x}) - \nabla F_{i}^{(t)}(\mathbf{x}) \|^2 ] &= \mathbb{E}[ \|\widetilde{\nabla} F_i^{(t)}(\mathbf{x})\|^2 ] - \big(\mathbb{E}[ \|\widetilde{\nabla} F_i^{(t)}(\mathbf{x})\|] \big)^2 \\
     &  \nonumber = \frac{1}{(m_{i}^{(t)})^2 ({{D}}_{i}^{(t)})^2} \Big[ \underbrace{\sum_{\xi \in {\mathcal{D}}_i^{(t)}} \|\nabla f(\mathbf{x} ; \xi)\|^2 Var(Z_{\xi})}_\text{$T_1$} \\
    &  \hskip 1cm + \underbrace{\sum_{\xi \in{\mathcal{D}}_i^{(t)}} \sum_{\Tilde{\xi} \in {\mathcal{D}}_i^{(t)} \setminus \{\xi\}} \langle \nabla f(\mathbf{x} ; \xi), \nabla f(\mathbf{x} ; \Tilde{\xi}) \rangle Cov[Z_{\xi} Z_{\Tilde{\xi}}}_\text{$T_2$}] \Big] \label{eqn: SGD noise 1}.
\end{align}
We next obtain the expressions for $T_1$ and $T_2$  in \eqref{eqn: SGD noise 1} as follows:
\begin{align}
  & \label{eqn: T1 calc} T_1 = \sum_{\xi \in  {\mathcal{D}}_i^{(t)}} \| \nabla f(\mathbf{x} ; \xi) \|^2 Var(Z_{\xi}) = m_{i}^{(t)}(1-m_{i}^{(t)}) \sum_{\xi \in  {\mathcal{D}}_i^{(t)}} \|\nabla f(\mathbf{x} ; \xi)\|^2,\\
  & \nonumber T_2 = \sum_{\xi \in  {\mathcal{D}}_i^{(t)}} \sum_{\Tilde{\xi} \in  {\mathcal{D}}_i^{(t)} \setminus \xi} \langle \nabla f(\mathbf{x} ; \xi), \nabla f(\mathbf{x} ; \Tilde{\xi}) \rangle  Cov[Z_{\xi} Z_{\Tilde{\xi}}] \\
  & \label{eqn: T2 calc}~~~ = -\frac{1}{{{D}}_{i}^{(t)} -1}\Big(1-m_{i}^{(t)}\Big)m_{i}^{(t)} \sum_{\xi \in {\mathcal{D}}_{i}^{(t)}} \sum_{\Tilde{\xi} \in {\mathcal{D}}_{i}^{(t)} \setminus \xi} \langle \nabla f(\mathbf{x} ; \xi), \nabla f(\mathbf{x} ; \Tilde{\xi}) \rangle.
\end{align}
Replacing \eqref{eqn: T1 calc}-\eqref{eqn: T2 calc} into \eqref{eqn: SGD noise 1} yields
\begin{align}
    \mathbb{E}[ \|\widetilde{\nabla} F_i^{(t)}(\mathbf{x}) - \nabla F_{i}^{(t)}(\mathbf{x}) \|^2 ] & = \frac{(1-m_{i}^{(t)})}{m_{i}^{(t)} ({{D}}_{i}^{(t)})^2({{D}}_{i}^{(t)} - 1)}\Big[ ({{D}}_{i}^{(t)} - 1)\sum_{\xi \in {\mathcal{D}}_{i}^{(t)}} \|\nabla f(\mathbf{x} ; \xi)\|^2 \\
  & \nonumber \hskip 2.5cm - \sum_{\xi \in {\mathcal{D}}_{i}^{(t)}} \sum_{\Tilde{\xi} \in {\mathcal{D}}_{i}^{(t)} \setminus \{\xi\}} \langle \nabla f(\mathbf{x} ; \xi) , \nabla f(\mathbf{x} ; \Tilde{\xi}) \rangle \Big] \\
  &  = \frac{(1-m_{i}^{(t)})}{m_{i}^{(t)} ({{D}}_{i}^{(t)})^2({{D}}_{i}^{(t)} - 1)} \Big[ ({{D}}_{i}^{(t)} - 1)\sum_{\xi \in {\mathcal{D}}_{i}^{(t)}} \|\nabla f(\mathbf{x} ; \xi)\|^2 \\
  & \nonumber \hskip 2cm - \big\| \displaystyle \sum_{\xi \in {\mathcal{D}}_{i}^{(t)}} \nabla f(\mathbf{x} ; \xi) \big\|^2 + \sum_{\xi \in {{D}}_{i}^{(t)}} \|\nabla f(\mathbf{x} ; \xi)\|^2 \Big] \\
  &  = \frac{(1-m_{i}^{(t)})}{m_{i}^{(t)} ({{D}}_{i}^{(t)})^2({{D}}_{i}^{(t)} - 1)} \Big[ {{D}}_{i}^{(t)} \sum_{\xi \in {\mathcal{D}}_{i}^{(t)}} \|\nabla f(\mathbf{x} ; \xi)\|^2 \\
  & \nonumber \hskip 3cm - ({{D}}_{i}^{(t)})^2 \|\nabla  F_{i}^{(t)}(\mathbf{x})\|^2\Big] \\
  & = \frac{(1-m_{i}^{(t)})}{m_{i}^{(t)} {{D}}_{i}^{(t)}} \underbrace{\Bigg[\frac{\sum_{\xi \in {\mathcal{D}}_{i}^{(t)}} \|\nabla f(\mathbf{x} ; \xi)-\nabla  F_{i}^{(t)}(\mathbf{x})\|^2}{{{D}}_{i}^{(t)} - 1}\Bigg]}_\text{(a)} \label{eqn: SGD noise 2}.
\end{align}
We then upper bound term (a) in \eqref{eqn: SGD noise 2} as follows:
\begin{align}
     \nonumber\frac{\sum_{\xi \in {\mathcal{D}}_{i}^{(t)}} \|\nabla f(\mathbf{x} ; \xi)-\nabla  F_{i}^{(t)}(\mathbf{x})\|^2}{{{D}}_{i}^{(t)} - 1} 
     & \nonumber= \frac{\sum_{\xi \in {\mathcal{D}}_{i}^{(t)}} \frac{1}{({{D}}_{i}^{(t)})^2} \| {{D}}_{i}^{(t)} \nabla f(\mathbf{x} ; \xi) - \sum_{\xi' \in {\mathcal{D}}_{i}^{(t)}} \nabla f(\mathbf{x} ; \xi') \|^2}{{{D}}_{i}^{(t)} - 1} \\
     &\nonumber \leq \frac{\sum_{\xi \in {\mathcal{D}}_{i}^{(t)}} \frac{{{D}}_{i}^{(t)} -1}{({{D}}_{i}^{(t)})^2} \sum_{\xi' \in {\mathcal{D}}_{i}^{(t)}} \|  \nabla f(\mathbf{x} ; \xi) -  \nabla f(\mathbf{x} ; \xi') \|^2}{{{D}}_{i}^{(t)} - 1} \\
     & \leq \frac{\sum_{\xi \in {\mathcal{D}}_{i}^{(t)}} \frac{{{D}}_{i}^{(t)} -1}{({{D}}_{i}^{(t)})^2} \Theta_i^2 \sum_{\xi' \in {\mathcal{D}}_{i}^{(t)}} \| \xi - \xi' \|^2}{{{D}}_{i}^{(t)} - 1} \label{eqn: consequence of local data variability} \\
     &\nonumber = \frac{{{D}}_{i}^{(t)} -1}{({{D}}_{i}^{(t)})^2} \Theta_i^2 \frac{\sum_{\xi \in {\mathcal{D}}_{i}^{(t)}}  \sum_{\xi' \in {\mathcal{D}}_{i}^{(t)}} \| \xi - \xi' + \Tilde{\lambda}_{i}^{(t)} - \Tilde{\lambda}_{i}^{(t)} \|^2}{{{D}}_{i}^{(t)} - 1} \\
     &\nonumber = \frac{{{D}}_{i}^{(t)} -1}{({{D}}_{i}^{(t)})^2} \Theta_i^2 \frac{ \sum_{\xi \in {\mathcal{D}}_{i}^{(t)}}  \sum_{\xi' \in {\mathcal{D}}_{i}^{(t)}} \| \xi - \Tilde{\lambda}_{i}^{(t)} \|^2 + \|\xi' - \Tilde{\lambda}_{i}^{(t)} \|^2 - 2 \langle \xi - \Tilde{\lambda}_{i}^{(t)}, \xi' - \Tilde{\lambda}_{i}^{(t)}\rangle }{{{D}}_{i}^{(t)} - 1} \\
     &\nonumber = \frac{({{D}}_{i}^{(t)} -1)}{({{D}}_{i}^{(t)})^2} \Theta_i^2 \frac{ {{D}}_{i}^{(t)} \sum_{\xi \in {\mathcal{D}}_{i}^{(t)}}   \| \xi - \Tilde{\lambda}_{i}^{(t)} \|^2 + {{D}}_{i}^{(t)} \sum_{\xi' \in {\mathcal{D}}_{i}^{(t)}} \|\xi' - \Tilde{\lambda}_{i}^{(t)} \|^2}{{{D}}_{i}^{(t)} - 1} \\
     & \leq \frac{2({{D}}_{i}^{(t)} -1)}{{{D}}_{i}^{(t)}} (\Tilde{\sigma}_{i}^{(t)})^2 \Theta_i^2,
     \label{eqn: SGD noise 3}
\end{align}
where \eqref{eqn: consequence of local data variability} is a consequence of local data variability associated with DPUs as described in Assumption~\ref{assumption: Local Data Variability}. 

Replacing term (a) in \eqref{eqn: SGD noise 2} with the expression in \eqref{eqn: SGD noise 3} allows us to obtain the following final expression for the noise variance of mini-batch SGD:
\begin{align}
    \mathbb{E}[ \|\widetilde{\nabla} F_i^{(t)}(\mathbf{x}) - \nabla F_{i}^{(t)}(\mathbf{x}) \|^2 ] & \leq \frac{2(1-m_{i}^{(t)})({{D}}_{i}^{(t)} -1)}{m_{i}^{(t)} {({D}}_{i}^{(t)})^2} (\Tilde{\sigma}_{i}^{(t)})^2 \Theta_i.
\end{align}
\end{proof}

\section{Proof of Theorem 1} \label{Proof of Theorem 1}

\begin{theorem}[General Convergence behavior of {\tt CE-FL}] \label{Thm: ML convergence appendix} Assume that the learning rate $\eta$ satisfies $4 \eta^2 L^2 \underset{t \in [T]}{\max} \ \underset{i \in \mathcal{N} \cup \mathcal{S}}{\max} \frac{ \gamma_i^2(t) (\|\mathbf{a}_{i}^{(t)}\|_{1} - [a_{i,-1}^{(t)}])}{\|\mathbf{a}_{i}^{(t)}\|_{1}} \leq \frac{1}{2 \zeta_1^2 + 1}$. The cumulative average of the global loss function under {\tt CE-FL} satisfies the following upper bound: 
\begin{align}
    \nonumber  \frac{1}{T} \sum_{t = 0}^{T -1} \mathbb{E} \bigg[  \|\nabla {{F}}^{(t)}(\mathbf{x}^{(t)})\|^2 \bigg]
    & \leq \frac{4}{\vartheta \eta T} \Bigg[ {{{F}}^{(0)}(\mathbf{x}^{(0)}) - {{F}}^{*}} \Bigg] + \frac{4}{\vartheta \eta T} \sum_{t = 0}^{T - 1} \sum_{i \in \mathcal{N} \cup \mathcal{S}} \tau^{(t)} \Delta_i^{(t)} \\
    & \nonumber + 16 \eta L \vartheta \Bigg[\frac{1}{T} \sum_{t = 0}^{T - 1}\underset{i \in \mathcal{N} \cup \mathcal{S}}{\sum}   \frac{(p_{i}^{(t)})^2 (1-m_{i}^{(t)})({{D}}_{i}^{(t)} -1) \Theta_i^2 (\Tilde{\sigma}_{i}^{(t)})^2}{m_{i}^{(t)} ({{D}}_{i}^{(t)})^2} \frac{\|\mathbf{a}_{i}^{(t)}\|_{2}^2}{\|\mathbf{a}_{i}^{(t)}\|_{1}^2}  \Bigg] \\ 
     & \nonumber + {12 \eta^2 L^2} \Bigg[\frac{1}{T} \sum_{t = 0}^{T - 1}\underset{i \in \mathcal{N} \cup \mathcal{S}}{\sum} \frac{ (1-m_{i}^{(t)})({{D}}_{i}^{(t)} -1) \Theta_i^2 (\Tilde{\sigma}_{i}^{(t)})^2 p_i^{(t)} \gamma_i^{(t)}} {m_{i}^{(t)}  \|\mathbf{a}_{i}^{(t)}\|_{1} ({{D}}_{i}^{(t)})^2} (\|\mathbf{a}_{i}^{(t)}\|_{2}^2 - [a_{i,-1}^{(t)}]^2) \Bigg] \\
    &  + 12 \eta^2 L^2 \zeta_2 \Big(\underset{t \in [T]}{\max} \ \underset{i \in \mathcal{N} \cup \mathcal{S}}{\max} \frac{(\gamma_i^{(t)})^2 (\|\mathbf{a}_{i}^{(t)}\|_{1} - [a_{i,-1}^{(t)}])}{\|\mathbf{a}_{i}^{(t)}\|_{1}}\Big).
\end{align}
\end{theorem}
\begin{proof}

Since normalized stochastic gradient is unbiased, we have:
\begin{align}
    \mathbb{E}[\mathbf{d}_{{i}}^{(t)} - \mathbf{h}_{{i}}^{(t)}] = 0.
\end{align}
Further, assuming independence of random noise associated with data-processing units $i \neq j$, $i, j \in \mathcal{N} \cup \mathcal{S}$ introduced due to mini-batch SGD, we have
\begin{align}
    \mathbb{E}[\langle \mathbf{d}_{{i}}^{(t)} - \mathbf{h}_{{i}}^{(t)}, \mathbf{d}_{{j}}^{(t)} - \mathbf{h}_{{j}}^{(t)} \rangle] = 0.
\end{align}
We restate the update rule for global model
\begin{align}
    \mathbf{x}^{(t + 1)} -\mathbf{x}^{(t)}
    & = -\vartheta \displaystyle \sum_{i \in \mathcal{N} \cup \mathcal{S}} p_i^{(t)} \eta \mathbf{d}_i^{(t)} .
\end{align}
As a consequence of Lipshitz smoothness property of global loss function (Assumption \eqref{assumption: Smooothness}), we obtain:
\begin{align}
    & \label{eqn: conseq of lipschitz smoothness} \mathbb{E}[{{F}}^{(t)}(\mathbf{x}^{(t+1)}) ] - {{F}}^{(t)}(\mathbf{x}^{(t)})  \leq - \vartheta \eta \underbrace{\mathbb{E}[\langle \nabla {{F}}^{(t)}(\mathbf{x}^{(t)}), \underset{i \in \mathcal{N} \cup \mathcal{S}}{\sum} p_{i}^{(t)}\mathbf{d}_{{i}}^{(t)}  \rangle]}_\text{$T_1$} + \frac{\vartheta^{2} \eta^2L}{2} \underbrace{\mathbb{E}[\| \underset{i \in \mathcal{N} \cup \mathcal{S}}{\sum} p_{i}^{(t)}\mathbf{d}_{{i}}^{(t)}\|^2]}_\text{$T_2$}.
\end{align}
Here, the expectation is w.r.t mini-batches $\{ {\mathcal{D}}_i^{(t,k)} \}$, $\forall i \in \mathcal{N} \cup \mathcal{S}$ and $k \in \big[ \gamma_{i}^{(t)} \big]$. In the following, we first bound term $T_1$ as follows:
\begin{align}
    & \nonumber T_1 = \mathbb{E}\Bigg[ \Big\langle \nabla {{F}}^{(t)}(\mathbf{x}^{(t)}), \underset{i \in \mathcal{N} \cup \mathcal{S}}{\sum} p_{i}^{(t)}(\mathbf{d}_{{i}}^{(t)}-\mathbf{h}_{{i}}^{(t)})  \Big\rangle \Bigg] + \mathbb{E} \Bigg[ \Big \langle \nabla {{F}}^{(t)}(\mathbf{x}^{(t)}), \underset{i \in \mathcal{N} \cup \mathcal{S}}{\sum} p_{i}^{(t)}\mathbf{h}_{{i}}^{(t)}  \Big \rangle \Bigg] \\
   &\nonumber \quad =\mathbb{E} \Bigg[ \Big \langle \nabla {{F}}^{(t)}(\mathbf{x}^{(t)}), \underset{i \in \mathcal{N} \cup \mathcal{S}}{\sum} p_{i}^{(t)}\mathbf{h}_{{i}}^{(t)}  \Big \rangle \Bigg] \\
   & \quad = \frac{1}{2}\|\nabla {{F}}^{(t)}(\mathbf{x}^{(t)})\|^2 +\frac{1}{2}\mathbb{E}[\|\underset{i \in \mathcal{N} \cup \mathcal{S}}{\sum} p_{i}^{(t)}\mathbf{h}_{{i}}^{({t})}\|^2]-\frac{1}{2}\mathbb{E}[\|\nabla {{F}}^{(t)}(\mathbf{x}^{(t)})-\underset{i \in \mathcal{N} \cup \mathcal{S}}{\sum} p_{i}^{(t)}\mathbf{h}_{{i}}^{(t)}\|^2].
\end{align}
The last equality above uses the following fact that
\begin{equation*}
    2 \langle \bm{a}, \bm{b} \rangle = \|\bm{a}\|^2 + \|\bm{b}\|^2 - \|\bm{a}-\bm{b}\|^2,
\end{equation*}
which holds for any real valued vectors $\bm{a}$ and $\bm{b}$.

We will next bound term $T_2$ as follows:
\begin{align}
   &\nonumber T_2 = \mathbb{E}[\| \underset{i \in \mathcal{N} \cup \mathcal{S}}{\sum} p_{i}^{(t)}\mathbf{d}_{{i}}^{(t)}\|^2] \\
   &\nonumber \quad = \mathbb{E}[\| \underset{i \in \mathcal{N} \cup \mathcal{S}}{\sum} p_{i}^{(t)}(\mathbf{d}_{{i}}^{(t)}-\mathbf{h}_{{i}}^{(t)}) + \underset{i \in \mathcal{N} \cup \mathcal{S}}{\sum} p_{i}^{(t)}\mathbf{h}_{{i}}^{(t)}\|^2] \\
   & \label{eqn: square reln 1} \quad \leq 2 \mathbb{E}[\| \underset{i \in \mathcal{N} \cup \mathcal{S}}{\sum} p_{i}^{(t)}(\mathbf{d}_{{i}}^{(t)}-\mathbf{h}_{{i}}^{(t)})\|^2] + 2 \mathbb{E}[\| \underset{i \in \mathcal{N} \cup \mathcal{S}}{\sum} p_{i}^{(t)}\mathbf{h}_{{i}}^{(t)}\|^2] \\
   & \label{eqn: conseq of independence} \quad \quad = 2 \underset{i \in \mathcal{N} \cup \mathcal{S}}{\sum} (p_{i}^{(t)})^2\mathbb{E}[\|\mathbf{d}_{{i}}^{(t)}-\mathbf{h}_{{i}}^{(t)}\|^2] + 2 \mathbb{E}[\| \underset{i \in \mathcal{N} \cup \mathcal{S}}{\sum} p_{i}^{(t)}\mathbf{h}_{{i}}^{(t)}\|^2],
\end{align}
where \eqref{eqn: square reln 1} is due to the fact that $\|\bm{a}+\bm{b}\|^2 \leq 2\|\bm{a}\|^2 + 2\|\bm{a}\|^2$ for arbitrary vectors $\bm{a}$ and $\bm{b}$. Also, \eqref{eqn: conseq of independence} uses the special property of $\{\mathbf{d}_{{i}}^{(t)}\}, \{\mathbf{h}_{{i}}^{(t)}\}$, that is, $\mathbb{E}[\langle \mathbf{d}_{{i}}^{(t)} - \mathbf{h}_{{i}}^{(t)}, \mathbf{d}_{{j}}^{(t)} - \mathbf{h}_{{j}}^{(t)} \rangle] = 0, \ i \neq j$.
Further, we can expand $\mathbf{d}_{{i}}^{(t)}$ and $\mathbf{h}_{{i}}^{(t)}$ and obtain
\begin{align}
     \label{eqn: conseq of random matrix sum norm} T_2 &\leq \Bigg[\underset{i \in \mathcal{N} \cup \mathcal{S}}{\sum} \frac{2(p_{i}^{(t)})^2}{\|\mathbf{a}_{i}^{(t)}\|_{1}^2} \displaystyle \sum_{\ell = 0}^{\gamma_i^{(t)} -1} [a_{i,\ell}^{(t)}]^2 \mathbb{E}[\|\widetilde{\nabla} F^{(t)}_{i}(\mathbf{x}_{i}^{(t, \ell)}) - \nabla F_{i}^{(t)}(\mathbf{x}_i^{(t,\ell)})\|^2] \Bigg] + 2 \mathbb{E}[\| \underset{i \in \mathcal{N} \cup \mathcal{S}}{\sum} p_{i}^{(t)}\mathbf{h}_{{i}}^{(t)}\|^2] \\
    & \label{eqn: conseq of prop 1} \leq \Bigg[\underset{i \in \mathcal{N} \cup \mathcal{S}}{\sum}  \frac{4(p_{i}^{(t)})^2 (1-m_{i}^{(t)})({{D}}_{i}^{(t)} -1) \Theta_i^2 (\Tilde{\sigma}_{i}^{(t)})^2}{m_{i}^{(t)} ({{D}}_{i}^{(t)})^2}  \frac{\|\mathbf{a}_{i}^{(t)}\|_{2}^2}{\|\mathbf{a}_{i}^{(t)}\|_{1}^2}  \Bigg] + 2 \mathbb{E}\Bigg[\bigg\| \underset{i \in \mathcal{N} \cup \mathcal{S}}{\sum} p_{i}^{(t)}\mathbf{h}_{{i}}^{(t)}\bigg\|^2\Bigg],
\end{align}
where \eqref{eqn: conseq of random matrix sum norm} is a consequence of Lemma 2 from~\cite{wang2020tackling} and \eqref{eqn: conseq of prop 1} is a consequence of Proposition 1. 
By plugging $T_1$ and $T_2$ back in \eqref{eqn: conseq of lipschitz smoothness}, we have:
\begin{align}
    \nonumber  \mathbb{E}[{{F}}^{(t)}(\mathbf{x}^{(t+1)}) ] - {{F}}^{(t)}(\mathbf{x}^{(t)}) & \leq - \vartheta \eta T_1 + \frac{\vartheta^2 \eta^2L}{2} T_2 \\
    & \nonumber = -\frac{\vartheta \eta}{2} \|\nabla {{F}}^{(t)}(\mathbf{x}^{(t)})\|^2 \underbrace{- \frac{\vartheta \eta}{2}(1-2 \vartheta \eta L)\mathbb{E}[\| \underset{i \in \mathcal{N} \cup \mathcal{S}}{\sum} p_{i}^{(t)} \mathbf{h}_{{i}}^{(t)}\|^2]}_\text{(a)} \\
    & \nonumber + \vartheta^{2} \eta^2L  \Bigg[\underset{i \in \mathcal{N} \cup \mathcal{S}}{\sum}  \frac{2(p_{i}^{(t)})^2 (1-m_{i}^{(t)})({{D}}_{i}^{(t)} -1) \Theta_i^2 (\Tilde{\sigma}_{i}^{(t)})^2}{m_{i}^{(t)} ({{D}}_{i}^{(t)})^2}  \frac{\|\mathbf{a}_{i}^{(t)}\|_{2}^2}{\|\mathbf{a}_{i}^{(t)}\|_{1}^2}  \Bigg] \\
    &  + \frac{\vartheta \eta}{2}\mathbb{E}[\|\nabla {{F}}^{(t)}(\mathbf{x}^{(t)})-\underset{i \in \mathcal{N} \cup \mathcal{S}}{\sum} p_{i}^{(t)}\mathbf{h}_{{i}}^{(t)}\|^2].
\end{align}
Assuming $\vartheta \eta L \leq \frac{1}{2}$ renders  term (a) to be negative, which leads to simplification of the above bound to
\begin{align}
    \nonumber \frac{\mathbb{E}[{{F}}^{(t)}(\mathbf{x}^{(t+1)}) ] - {{F}}^{(t)}(\mathbf{x}^{(t)})}{\vartheta \eta} & \leq -\frac{1}{2} \|\nabla {{F}}^{(t)}(\mathbf{x}^{(t)})\|^2 \\
    & \nonumber +\vartheta \eta L  \Bigg[\underset{i \in \mathcal{N} \cup \mathcal{S}}{\sum}  \frac{2(p_{i}^{(t)})^2 (1-m_{i}^{(t)})({{D}}_{i}^{(t)} -1) \Theta_i^2 (\Tilde{\sigma}_{i}^{(t)})^2}{m_{i}^{(t)} ({{D}}_{i}^{(t)})^2}  \frac{\|\mathbf{a}_{i}^{(t)}\|_{2}^2}{\|\mathbf{a}_{i}^{(t)}\|_{1}^2}  \Bigg] \\
    & \nonumber + \frac{1}{2}\mathbb{E}[\|\nabla {{F}}^{(t)}(\mathbf{x}^{(t)})-\underset{i \in \mathcal{N} \cup \mathcal{S}}{\sum} p_{i}^{(t)}\mathbf{h}_{{i}}^{(t)}\|^2] \\
    & \label{eqn: initial ML inequality} \leq -\frac{1}{2} \|\nabla {{F}}^{(t)}(\mathbf{x}^{(t)})\|^2 \\
    & \nonumber + \vartheta \eta L  \Bigg[\underset{i \in \mathcal{N} \cup \mathcal{S}}{\sum}  \frac{2(p_{i}^{(t)})^2 (1-m_{i}^{(t)})({{D}}_{i}^{(t)} -1) \Theta_i^2 (\Tilde{\sigma}_{i}^{(t)})^2}{m_{i}^{(t)} ({{D}}_{i}^{(t)})^2}  \frac{\|\mathbf{a}_{i}^{(t)}\|_{2}^2}{\|\mathbf{a}_{i}^{(t)}\|_{1}^2}  \Bigg] \\
    & \nonumber + \frac{1}{2} \underset{i \in \mathcal{N} \cup \mathcal{S}}{\sum} p_{i}^{(t)} \underbrace{\mathbb{E}[\|\nabla F_{i}^{(t)}(\mathbf{x}^{(t)} )-\mathbf{h}_{{i}}^{(t)}\|^2]}_\text{$T_3$}.
\end{align}
The last inequality uses the fact that ${F}^{(t)}(\mathbf{x}) = \displaystyle \sum_{i \in \mathcal{N} \cup \mathcal{S}} p_i^{(t)}  F_{i}^{(t)}(\mathbf{x})$ and is a consequence of Jensen's Inequality: $\|\underset{i \in \mathcal{N} \cup \mathcal{S}}{\sum} p_{i}^{(t)} z_i \|^2 \leq \underset{i \in \mathcal{N} \cup \mathcal{S}}{\sum} p_{i}^{(t)} \|z_i\|^2$ since we also have $\underset{i \in \mathcal{N} \cup \mathcal{S}}{\sum} p_{i}^{(t)} = 1$. We next aim to bound $T_3$ as follows:
\begin{align}
    \nonumber T_3 & = {\mathbb{E}[\|\nabla F_{i}^{(t)}(\mathbf{x}^{(t)} )-\mathbf{h}_{{i}}^{(t)}\|^2]} \\
     & = \mathbb{E}\Bigg[\Big\|\nabla F_{i}^{(t)}(\mathbf{x}^{(t)} ) - \frac{1}{\|\mathbf{a}_{i}^{(t)}\|_{1}} \displaystyle \sum_{\ell = 0}^{\gamma_{i}^{(t)} - 1} a_{{i,\ell}}^{(t)} \nabla F_{i}^{(t)}(\mathbf{x}_{i}^{(t, \ell)} ) \Big\|^2 \Bigg] \label{eqn: T3 definition}\\
    &\nonumber = \mathbb{E}\Bigg[\Big\| \frac{1}{\|\mathbf{a}_{i}^{(t)}\|_{1}} \displaystyle \sum_{\ell = 0}^{\gamma_{i}^{(t)} - 1} a_{{i,\ell}}^{(t)} \big(\nabla F_{i}^{(t)}(\mathbf{x}^{(t)} ) - \nabla F_{i}^{(t)}(\mathbf{x}_{i}^{(t, \ell)}\big) \Big\|^2 \Bigg] \\
    & \label{eqn: conseq of Jensen 1} \leq \frac{1}{\|\mathbf{a}_{i}^{(t)}\|_{1}} \displaystyle \sum_{\ell = 0}^{\gamma_{i}^{(t)} - 1} \Bigg\{ a_{{i,\ell}}^{(t)} \mathbb{E}\bigg[\Big\|  \big(\nabla F_{i}^{(t)}(\mathbf{x}^{(t)}) - \nabla F_{i}^{(t)}(\mathbf{x}_{i}^{(t, \ell)}\big) \Big\|^2 \bigg] \Bigg\} \\
    & \label{eqn: conseq of assumption 1} \leq \frac{L^2}{\|\mathbf{a}_{i}^{(t)}\|_{1}} \displaystyle \sum_{\ell = 0}^{\gamma_{i}^{(t)} - 1} \Bigg\{ a_{{i,\ell}}^{(t)} \mathbb{E}\bigg[\Big\|\mathbf{x}^{(t)} - \mathbf{x}_{i}^{(t,\ell)} \Big\|^2 \bigg] \Bigg\} \\
    &  \leq \frac{L^2}{\|\mathbf{a}_{i}^{(t)}\|_{1}} \displaystyle \sum_{\ell = 0}^{\gamma_{i}^{(t)} - 1} \Bigg\{  \mathbb{E}\bigg[\Big\|\mathbf{x}^{(t)} - \mathbf{x}_{i}^{(t,\ell)} \Big\|^2 \bigg] \Bigg\} \label{eqn: intermediate exp of T_3}.
\end{align}
Note that \eqref{eqn: conseq of Jensen 1} is a consequence of Jensen's inequality: $\|\underset{i \in \mathcal{N} \cup \mathcal{S}}{\sum} p_{i}^{(t)} z_i \|^2 \leq \underset{i \in \mathcal{N} \cup \mathcal{S}}{\sum} p_{i}^{(t)} \|z_i\|^2$ where we leverage the fact that  $\underset{i \in \mathcal{N} \cup \mathcal{S}}{\sum} p_{i}^{(t)} = 1$ and \eqref{eqn: conseq of assumption 1} follows from Assumption 1. \\
Henceforth, we bound the right hand side of \eqref{eqn: intermediate exp of T_3} as follows: 
\begin{align}
    \nonumber\mathbb{E}\bigg[\Big\|\mathbf{x}^{(t)} - \mathbf{x}_{i}^{(t,\ell)} \Big\|^2 \bigg] & = \eta^2 \mathbb{E}\bigg[\Big\| \displaystyle \sum_{k = 0}^{\ell-1} (1-\eta\mu)^{\ell-1-k} \widetilde{\nabla} F^{(t)}_{i}(\mathbf{x}_{i}^{(t, k)}) \Big\|^2 \bigg] \\
    & \nonumber\leq 2 \eta^2 \underbrace{\mathbb{E}\bigg[\Big\| \displaystyle \sum_{k = 0}^{\ell-1} (1-\eta\mu)^{\ell-1-k} \big( \widetilde{\nabla} F^{(t)}_{i}(\mathbf{x}_{i}^{(t, k)}) - \nabla  F_{i}^{(t)}(\mathbf{x}_i^{(t,k)} ) \big) \Big\|^2 \bigg]}_\text{(a)} \\
    & + 2 \eta^2 \mathbb{E}\bigg[\Big\| \displaystyle \sum_{k = 0}^{\ell-1} (1-\eta\mu)^{\ell-1-k} \nabla  F_{i}^{(t)}(\mathbf{x}_i^{(t,k)}) \Big\|^2 \bigg] \label{eqn: intermediate exp of T_3 part 2}.
\end{align}
where \eqref{eqn: intermediate exp of T_3 part 2} is a direct consequence of $\|\bm{a}+\bm{b}\|^2 \leq 2\|\bm{a}\|^2 + 2\|\bm{a}\|^2$ for arbitrary vectors $\bm{a}$ and $\bm{b}$. Applying Lemma 2 from~\cite{wang2020tackling} we upper bound term (a) on the RHS of \eqref{eqn: intermediate exp of T_3 part 2} as follows:
\begin{align}
   \nonumber \mathbb{E}\bigg[\Big\|\mathbf{x}^{(t)} - \mathbf{x}_{i}^{(t,\ell)} \Big\|^2 \bigg]  & \leq 2 \eta^2 \sum_{k = 0}^{\ell-1} [(1-\eta\mu)^{l-1-k}]^2 \mathbb{E}\bigg[\Big\| \displaystyle  \widetilde{\nabla} F^{(t)}_{i}(\mathbf{x}_{i}^{(t, k)}) - \nabla  F_{i}^{(t)}(\mathbf{x}_i^{(t,k)} )  \Big\|^2 \bigg] \\
    & \nonumber + 2 \eta^2 \mathbb{E}\bigg[\Big\| \displaystyle \sum_{k = 0}^{\ell-1} (1-\eta\mu)^{\ell-1-k} \nabla  F_{i}^{(t)}(\mathbf{x}_i^{(t,k)}) \Big\|^2 \bigg] \\
    & \label{eqn: conseq of prop 1-2} \leq 4 \eta^2 \frac{(1-m_{i}^{(t)})({{D}}_{i}^{(t)} -1)}{m_{i}^{(t)} ({{D}}_{i}^{(t)})^2} \Theta_i^2 (\Tilde{\sigma}_{i}^{(t)})^2  \sum_{k = 0}^{\ell-1} [(1-\eta\mu)^{\ell-1-k}]^2 \\
    & \nonumber + 2 \eta^2 \mathbb{E}\bigg[\Big\| \displaystyle \sum_{k = 0}^{\ell-1} (1-\eta\mu)^{\ell-1-k} \nabla  F_{i}^{(t)}(\mathbf{x}_i^{(t,k)}) \Big\|^2 \bigg] \\
    & \label{eqn: conseq of Jensen 2} \leq 4 \eta^2 \frac{(1-m_{i}^{(t)})({{D}}_{i}^{(t)} -1)}{m_{i}^{(t)} ({{D}}_{i}^{(t)})^2} \Theta_i^2 (\Tilde{\sigma}_{i}^{(t)})^2  \sum_{k = 0}^{\ell-1} [(1-\eta\mu)^{\ell-1-k}]^2 \\
    & \nonumber + 2 \eta^2 \bigg[ \sum_{k = 0}^{\ell-1} (1-\eta\mu)^{\ell-1-k} \bigg] \sum_{k = 0}^{\ell-1} (1-\eta\mu)^{\ell-1-k} \mathbb{E}\bigg[\Big\| \displaystyle \nabla  F_{i}^{(t)}(\mathbf{x}_i^{(t,k)}) \Big\|^2 \bigg] \\
    &\nonumber \leq 4 \eta^2 \frac{(1-m_{i}^{(t)})({{D}}_{i}^{(t)} -1)}{m_{i}^{(t)} ({{D}}_{i}^{(t)})^2} \Theta_i^2 (\Tilde{\sigma}_{i}^{(t)})^2  \sum_{k = 0}^{\ell-1} [(1-\eta\mu)^{\ell-1-k}]^2 \\
    & \nonumber + 2 \eta^2 \bigg[ \sum_{k = 0}^{\ell-1} (1-\eta\mu)^{\ell-1-k} \bigg] \sum_{k = 0}^{\ell-1} \mathbb{E}\bigg[\Big\| \displaystyle \nabla  F_{i}^{(t)}(\mathbf{x}_i^{(t,k)}) \Big\|^2 \bigg] \\
     & \label{eqn: T3 upper bound expansion 1} \leq 4 \eta^2 \frac{(1-m_{i}^{(t)})({{D}}_{i}^{(t)} -1)}{m_{i}^{(t)} ({{D}}_{i}^{(t)})^2} \Theta_i^2 (\Tilde{\sigma}_{i}^{(t)})^2  \sum_{k = 0}^{\ell-1} [(1-\eta\mu)^{\ell-1-k}]^2 \\
    & \nonumber + 2 \eta^2 \bigg[ \sum_{k = 0}^{\ell - 1} (1-\eta\mu)^{\ell-1-k} \bigg] \sum_{k = 0}^{\gamma_{i}^{(t)}-1} \mathbb{E}\bigg[\Big\| \displaystyle \nabla  F_{i}^{(t)}(\mathbf{x}_i^{(t,k)}) \Big\|^2 \bigg].
\end{align}
We highlight that \eqref{eqn: conseq of prop 1-2} is a consequence of Proposition 1 and \eqref{eqn: conseq of Jensen 2} is a consequence of Jensen's inequality.
 We further note that
\begin{align}
    \label{eqn: coeff L2 bound}  \sum_{\ell = 0}^{\gamma_i^{(t)} -1}  \Big[ \sum_{k = 0}^{\ell-1} [(1-\eta\mu)^{\ell-1-k}]^2 \Big] \leq  {\gamma_i^{(t)}} (\|\mathbf{a}_{i}^{(t)}\|_{2}^2 - [a_{i,-1}^{(t)}]^2),\\
\label{eqn: coeff L1 bound}   \sum_{\ell = 0}^{\gamma_i^{(t)} -1}  \Big[ \sum_{k = 0}^{\ell-1} [(1-\eta\mu)^{\ell-1-k}] \Big] \leq {\gamma_i^{(t)}} (\|\mathbf{a}_{i}^{(t)}\|_{1} - [a_{i,-1}^{(t)}]).
\end{align}

Hence, in order to  upper bound $T_3$ as defined in \eqref{eqn: T3 definition}-\eqref{eqn: intermediate exp of T_3}, we first perform summation over $\ell \in \{0,1, ...., \gamma_{i}^{(t)} -1 \}$ on the LHS and the RHS of \eqref{eqn: T3 upper bound expansion 1}, and using \eqref{eqn: coeff L2 bound}- \eqref{eqn: coeff L1 bound} allows us to obtain the following expression:
\begin{align} 
    \nonumber \displaystyle \sum_{\ell = 0}^{\gamma_{i}^{(t)} - 1} \Bigg\{  \mathbb{E}\bigg[\Big\|\mathbf{x}^{(t)} - \mathbf{x}_{i}^{(t,\ell)} \Big\|^2 \bigg] \Bigg\} & \leq 4 \eta^2 \frac{(1-m_{i}^{(t)})({{D}}_{i}^{(t)} -1)}{m_{i}^{(t)} ({{D}}_{i}^{(t)})^2} \Theta_i^2 (\Tilde{\sigma}_{i}^{(t)})^2{\gamma_i^{(t)}} (\|\mathbf{a}_{i}^{(t)}\|_{2}^2 - [a_{i,-1}^{(t)}]^2)  \\
    &  + 2 \eta^2 {\gamma_i^{(t)}} (\|\mathbf{a}_{i}^{(t)}\|_{1} - [a_{i,-1}^{(t)}]) \sum_{\ell = 0}^{\gamma_i^{(t)} -1}  \mathbb{E}\bigg[ \underbrace{ \Big\| \displaystyle \nabla  F_{i}^{(t)}(\mathbf{x}_i^{(t,\ell)}) \Big\|^2}_\text{(a)} \bigg] \label{eqn: T3 intermediate v2}.
\end{align}
Next, we upper-bound the term (a) within the summation in \eqref{eqn: T3 intermediate v2} using the relation $\|\bm{a} + \bm{b}\|^2 \leq 2\|\bm{a}\|^2 + 2\|\bm{b}\|^2$ as follows:
\begin{align}
    \nonumber \mathbb{E}\bigg[\Big\| \displaystyle \nabla  F_{i}^{(t)}(\mathbf{x}_i^{(t,\ell)}) \Big\|^2 \bigg] & \leq 2 \mathbb{E}\bigg[\Big\| \displaystyle \nabla  F_{i}^{(t)}(\mathbf{x}_i^{(t,\ell)}) - \nabla  F_{i}^{(t)}(\mathbf{x}^{(t)})\Big\|^2 \bigg] + 2 \mathbb{E}\bigg[\Big\| \displaystyle \nabla  F_{i}^{(t)}(\mathbf{x}^{(t)}) \Big\|^2 \bigg] \\
    & \label{eqn: smoothness consequence for T3} \leq \underbrace{2L^2 \mathbb{E}\bigg[\Big \|\mathbf{x}^{(t)} - \mathbf{x}_{i}^{(t,\ell)} \Big\|^2 \bigg]}_\text{(b)} + 2 \mathbb{E}\bigg[\Big\| \displaystyle \nabla  F_{i}^{(t)}(\mathbf{x}^{(t)}) \Big\|^2 \bigg].
\end{align}
Note that (b) in \eqref{eqn: smoothness consequence for T3} is a direct consequence of smoothness of the local loss functions stated in Assumption \ref{assumption: Smooothness}. Henceforth, we replace \eqref{eqn: smoothness consequence for T3} in term (a) of  \eqref{eqn: T3 intermediate v2} and obtain
\begin{align}
   \nonumber  \displaystyle \sum_{\ell = 0}^{\gamma_{i}^{(t)} - 1} \Bigg\{  \mathbb{E}\bigg[\Big\|\mathbf{x}^{(t)} - \mathbf{x}_{i}^{(t,\ell)} \Big\|^2 \bigg] \Bigg\} & \nonumber \leq 4 \eta^2 \frac{(1-m_{i}^{(t)})({{D}}_{i}^{(t)} -1)}{m_{i}^{(t)} ({{D}}_{i}^{(t)})^2} \Theta_i^2 (\Tilde{\sigma}_{i}^{(t)})^2\gamma_i^{(t)}(\|\mathbf{a}_{i}^{(t)}\|_{2}^2 - [a_{i,-1}^{(t)}]^2) \\
    & \nonumber + 4 \eta^2 L^2\gamma_i^{(t)}(\|\mathbf{a}_{i}^{(t)}\|_{1} - [a_{i,-1}^{(t)}]) \sum_{\ell = 0}^{\gamma_i^{(t)} -1}  \mathbb{E}\bigg[\Big\|\mathbf{x}^{(t)} - \mathbf{x}_{i}^{(t,\ell)} \Big\|^2 \bigg] \\
    & + 4 \eta^2\gamma_i^{(t)} (\|\mathbf{a}_{i}^{(t)}\|_{1} - [a_{i,-1}^{(t)}]) \sum_{\ell = 0}^{\gamma_i^{(t)} -1}  \mathbb{E}\bigg[\Big\| \displaystyle \nabla  F_{i}^{(t)}(\mathbf{x}^{(t)}) \Big\|^2 \bigg] \label{eqn: T3 intermediate v3}.
\end{align}
After re-arranging the terms in \eqref{eqn: T3 intermediate v3}, we get
\begin{align}
     \displaystyle \sum_{\ell = 0}^{\gamma_{i}^{(t)} - 1} \Bigg\{  \mathbb{E}\bigg[\Big\|\mathbf{x}^{(t)} - \mathbf{x}_{i}^{(t,\ell)} \Big\|^2 \bigg] \Bigg\} & \leq \frac{4 \eta^2 (1-m_{i}^{(t)})({{D}}_{i}^{(t)} -1) \Theta_i^2 (\Tilde{\sigma}_{i}^{(t)})^2 \gamma_i^{(t)}} {m_{i}^{(t)} ({{D}}_{i}^{(t)})^2(1- 4 \eta^2 L^2\gamma_i^{(t)}(\|\mathbf{a}_{i}^{(t)}\|_{1} - [a_{i,-1}^{(t)}]))}(\|\mathbf{a}_{i}^{(t)}\|_{2}^2 - [a_{i,-1}^{(t)}]^2) \nonumber \\
    &  + \frac{4 \eta^2 (\gamma_i^{(t)})^2 (\|\mathbf{a}_{i}^{(t)}\|_{1} - [a_{i,-1}^{(t)}])}{1- 4 \eta^2 L^2\gamma_i^{(t)}(\|\mathbf{a}_{i}^{(t)}\|_{1} - [a_{i,-1}^{(t)}])} \mathbb{E}\bigg[\Big\| \displaystyle \nabla  F_{i}^{(t)}(\mathbf{x}^{(t)}) \Big\|^2 \bigg]  \label{eqn: T3 intermediate v4}.
\end{align}
We introduce $\Gamma = \underset{t \in [T]}{\max} \ \underset{i \in \mathcal{N} \cup \mathcal{S}}{\max} \frac{4 \eta^2 L^2 (\gamma_i^{(t)})^2 (\|\mathbf{a}_{i}^{(t)}\|_{1} - [a_{i,-1}^{(t)}])}{\|\mathbf{a}_{i}^{(t)}\|_{1}}$ and use the fact that $\|\mathbf{a}_{i}^{(t)}\|_{1} < \gamma_i^{(t)}$ to the RHS of \eqref{eqn: T3 intermediate v4} to obtain
\begin{align}
    \frac{L^2}{\|\mathbf{a}_{i}^{(t)}\|_{1}} \displaystyle \sum_{\ell = 0}^{\gamma_{i}^{(t)} - 1} \Bigg\{  \mathbb{E}\bigg[\Big\|\mathbf{x}^{(t)} - \mathbf{x}_{i}^{(t,\ell)} \Big\|^2 \bigg] \Bigg\} & \leq \frac{4 \eta^2 L^2 (1-m_{i}^{(t)})({{D}}_{i}^{(t)} -1) \Theta_i^2 (\Tilde{\sigma}_{i}^{(t)})^2 \gamma_i^{(t)}} {m_{i}^{(t)} \|\mathbf{a}_{i}^{(t)}\|_{1} ({{D}}_{i}^{(t)})^2(1- \Gamma)}(\|\mathbf{a}_{i}^{(t)}\|_{2}^2 - [a_{i,-1}^{(t)}]^2) \nonumber \\
    &  + \Bigg(\frac{\Gamma}{1-\Gamma} \mathbb{E}\bigg[\Big\| \displaystyle \nabla  F_{i}^{(t)}(\mathbf{x}^{(t)}) \Big\|^2 \bigg]\Bigg) \label{eqn: T3 intermediate v5}.
\end{align}
Furthermore, for \eqref{eqn: T3 intermediate v5} to hold, the step-size parameter $\eta$ is appropriately chosen such that:
\begin{align}
   \Gamma = \underset{t \in [T]}{\max} \ \underset{i \in \mathcal{N} \cup \mathcal{S}}{\max} \frac{4 \eta^2 L^2 (\gamma_i^{(t)})^2 (\|\mathbf{a}_{i}^{(t)}\|_{1} - [a_{i,-1}^{(t)}])}{\|\mathbf{a}_{i}^{(t)}\|_{1}} < 1 \label{eqn: Gamma constraint 1}
\end{align}
Thus, we use \eqref{eqn: T3 intermediate v5} to upper-bound $T_3$ defined in \eqref{eqn: T3 definition}-\eqref{eqn: intermediate exp of T_3} as follows:
\begin{align}
    \nonumber T_3 & = {\mathbb{E}[\|\nabla F_{i}^{(t)}(\mathbf{x}^{(t)} )-\mathbf{h}_{{i}}^{(t)}\|^2]} \\
        &  \nonumber \leq \frac{L^2}{\|\mathbf{a}_{i}^{(t)}\|_{1}} \displaystyle \sum_{\ell = 0}^{\gamma_{i}^{(t)} - 1} \Bigg\{  \mathbb{E}\bigg[\Big\|\mathbf{x}^{(t)} - \mathbf{x}_{i}^{(t,\ell)} \Big\|^2 \bigg] \Bigg\} \\
        & \leq \frac{4 \eta^2 L^2 (1-m_{i}^{(t)})({{D}}_{i}^{(t)} -1) \Theta_i^2 (\Tilde{\sigma}_{i}^{(t)})^2 \gamma_i^{(t)}} {m_{i}^{(t)} \|\mathbf{a}_{i}^{(t)}\|_{1} ({{D}}_{i}^{(t)})^2(1- \Gamma)}(\|\mathbf{a}_{i}^{(t)}\|_{2}^2 - [a_{i,-1}^{(t)}]^2)  + \Bigg(\frac{\Gamma}{1-\Gamma} \mathbb{E}\bigg[\Big\| \displaystyle \nabla  F_{i}^{(t)}(\mathbf{x}^{(t)}) \Big\|^2 \bigg]\Bigg) \label{eqn: T3 intermediate v6}.
\end{align}
In inequality \eqref{eqn: initial ML inequality}, we highlight that $T_3$ is captured via weighted sum across all data-processing units $i \in \mathcal{N} \cup \mathcal{S}$ with individual weights being $\{ p_i^{(t)}\}$ (defined in  \eqref{eqn: ML data ratio definition}). Thus, we compute the weighted sum of $T_3$ by extending inequality \eqref{eqn: T3 intermediate v6} as follows:
\begin{align}
  \nonumber \frac{1}{2} \underset{i \in \mathcal{N} \cup \mathcal{S}}{\sum} & p_{i}^{(t)} {\mathbb{E}[\|\nabla F_{i}^{(t)}(\mathbf{x}^{(t)})-\mathbf{h}_{{i}}^{({t})}\|^2]}  \\
   & \nonumber \leq \frac{2 \eta^2 L^2}{(1- \Gamma)} \underset{i \in \mathcal{N} \cup \mathcal{S}}{\sum} \frac{ (1-m_{i}^{(t)})({{D}}_{i}^{(t)} -1) \Theta_i^2 (\Tilde{\sigma}_{i}^{(t)})^2 p_i^{(t)} \gamma_i^{(t)}} {m_{i}^{(t)} \|\mathbf{a}_{i}^{(t)}\|_{1} ({{D}}_{i}^{(t)})^2}(\|\mathbf{a}_{i}^{(t)}\|_{2}^2 - [a_{i,-1}^{(t)}]^2)  \\
   &  \nonumber + \frac{\Gamma}{2(1-\Gamma)} \underset{i \in \mathcal{N} \cup \mathcal{S}}{\sum} {p_{i}^{(t)}} \mathbb{E}\bigg[\Big\| \displaystyle \nabla  F_{i}^{(t)}(\mathbf{x}_i^{(t)}) \Big\|^2 \bigg] \\
   &  \nonumber \leq \frac{2 \eta^2 L^2}{(1- \Gamma)} \underset{i \in \mathcal{N} \cup \mathcal{S}}{\sum} \frac{ (1-m_{i}^{(t)})({{D}}_{i}^{(t)} -1) \Theta_i^2 (\Tilde{\sigma}_{i}^{(t)})^2 p_i^{(t)} \gamma_i^{(t)}} {m_{i}^{(t)}  \|\mathbf{a}_{i}^{(t)}\|_{1} ({{D}}_{i}^{(t)})^2} (\|\mathbf{a}_{i}^{(t)}\|_{2}^2 - [a_{i,-1}^{(t)}]^2)  \\
   & \label{eqn: T3 for final result} + \underbrace{\frac{\Gamma \zeta_1}{2(1-\Gamma)} \mathbb{E}\bigg[\Big\| \nabla F^{(t)}(\mathbf{x}^{(t)} )\Big\|^2 \bigg] + \frac{\Gamma \zeta_2}{2(1-\Gamma)}}_\text{(a)}.
\end{align}
 We highlight that term (a) in \eqref{eqn: T3 for final result} is a direct consequence of bounded dissimilarity of local loss functions described in Assumption \ref{assumption: bounded dissimilarity}.
 After plugging back \eqref{eqn: T3 for final result} to \eqref{eqn: initial ML inequality}, we obtain
\begin{align}
    \frac{\mathbb{E}[{{F}}^{(t)}(\mathbf{x}^{(t+1)}) ] - {{F}}^{(t)}(\mathbf{x}^{(t)})}{\vartheta \eta} & \leq -\frac{1}{2} \|\nabla {{F}}^{(t)}(\mathbf{x}^{(t)})\|^2 \nonumber \\
    & \hskip -1cm \nonumber + \vartheta \eta L  \Bigg[\underset{i \in \mathcal{N} \cup \mathcal{S}}{\sum}  \frac{4(p_{i}^{(t)})^2 (1-m_{i}^{(t)})({{D}}_{i}^{(t)} -1) \Theta_i (\Tilde{\sigma}_{i}^{(t)})^2}{m_{i}^{(t)} ({{D}}_{i}^{(t)})^2}  \frac{\|\mathbf{a}_{i}^{(t)}\|_{2}^2}{\|\mathbf{a}_{i}^{(t)}\|_{1}^2}  \Bigg] \\
   &\hskip -1cm \nonumber + \frac{2 \eta^2 L^2}{(1- \Gamma)} \underset{i \in \mathcal{N} \cup \mathcal{S}}{\sum} \frac{ (1-m_{i}^{(t)})({{D}}_{i}^{(t)} -1) \Theta_i^2 (\Tilde{\sigma}_{i}^{(t)})^2 p_i^{(t)} \gamma_i^{(t)}} {m_{i}^{(t)}  \|\mathbf{a}_{i}^{(t)}\|_{1} ({{D}}_{i}^{(t)})^2} (\|\mathbf{a}_{i}^{(t)}\|_{2}^2 - [a_{i,-1}^{(t)}]^2)    \\
   & \hskip -1cm  + \frac{\Gamma \zeta_1}{2(1-\Gamma)} \mathbb{E}\bigg[\Big\| \nabla F^{(t)}(\mathbf{x}^{(t)} )\Big\|^2 \bigg] + \frac{\Gamma \zeta_2}{2(1-\Gamma)}\label{eqn: intermediate ML inequality} \\
   & \hskip -1cm \nonumber = -\frac{1}{2}\Big(\underbrace{\frac{1-\Gamma(1+\zeta_1)}{1-\Gamma} }_\text{(a)}\Big)\|\nabla {{F}}^{(t)}(\mathbf{x}^{(t)})\|^2 \\
   & \hskip -1cm \nonumber + \vartheta \eta L  \Bigg[\underset{i \in \mathcal{N} \cup \mathcal{S}}{\sum}  \frac{4(p_{i}^{(t)})^2 (1-m_{i}^{(t)})({{D}}_{i}^{(t)} -1) \Theta_i^2 (\Tilde{\sigma}_{i}^{(t)})^2}{m_{i}^{(t)} ({{D}}_{i}^{(t)})^2}  \frac{\|\mathbf{a}_{i}^{(t)}\|_{2}^2}{\|\mathbf{a}_{i}^{(t)}\|_{1}^2}  \Bigg] \\
   & \hskip -1cm \nonumber + \underbrace{\frac{2 \eta^2 L^2}{(1- \Gamma)}}_\text{(b)} \underset{i \in \mathcal{N} \cup \mathcal{S}}{\sum} \frac{ (1-m_{i}^{(t)})({{D}}_{i}^{(t)} -1) \Theta_i^2 (\Tilde{\sigma}_{i}^{(t)})^2 p_i^{(t)} \gamma_i^{(t)}} {m_{i}^{(t)}  \|\mathbf{a}_{i}^{(t)}\|_{1} ({{D}}_{i}^{(t)})^2} (\|\mathbf{a}_{i}^{(t)}\|_{2}^2 - [a_{i,-1}^{(t)}]^2)  + \underbrace{\frac{\Gamma \zeta_2}{2(1-\Gamma)}}_\text{(c)} .
\end{align}
Additionally, we impose $\Gamma$ to be restricted by the following inequality:
\begin{align}
    \Gamma \leq \frac{1}{2 \zeta_1^2 + 1} \label{eqn: Gamma constraint 2}.
\end{align}
Thus, obtaining:
\begin{align}
 & \frac{1}{1-\Gamma} \leq 1 + \frac{1}{2 \zeta_1^2} \label{eqn: Gamma constraint 3 v2}.
\end{align}
Furthermore, from Assumption \ref{assumption: bounded dissimilarity} we have $\zeta_1 \geq 1$, allowing us to write 
 \begin{align}
 & \frac{\Gamma \zeta_1}{1-\Gamma} \leq \frac{\Gamma \zeta_1^2}{1-\Gamma} \leq \frac{1}{2} \Rightarrow 1 - \frac{\Gamma \zeta_1}{1-\Gamma} \geq \frac{1}{2}. \label{eqn: Gamma constraint 4} 
\end{align}
Therefore, we upper bound terms (a) and (b),(c) in  \eqref{eqn: intermediate ML inequality} with \eqref{eqn: Gamma constraint 4} and \eqref{eqn: Gamma constraint 3 v2}, respectively, and replace the definition of $\Gamma$ in \eqref{eqn: Gamma constraint 1}. This results in
\begin{align}
     \nonumber\frac{\mathbb{E}[{{F}}^{(t)}(\mathbf{x}^{(t+1)}) ] - {{F}}^{(t)}(\mathbf{x}^{(t)})}{\vartheta \eta} & \leq -\frac{1}{4} \|\nabla {{F}}^{(t)}(\mathbf{x}^{(t)})\|^2 \\
    & \nonumber + \vartheta \eta L  \Bigg[\underset{i \in \mathcal{N} \cup \mathcal{S}}{\sum}  \frac{4(p_{i}^{(t)})^2 (1-m_{i}^{(t)})({{D}}_{i}^{(t)} -1) \Theta_i^2 (\Tilde{\sigma}_{i}^{(t)})^2}{m_{i}^{(t)} ({{D}}_{i}^{(t)})^2}  \frac{\|\mathbf{a}_{i}^{(t)}\|_{2}^2}{\|\mathbf{a}_{i}^{(t)}\|_{1}^2}  \Bigg] \\
    & \nonumber \hskip -2cm + {2 \eta^2 L^2}(1 + \frac{1}{2 \zeta_1^2}) \underset{i \in \mathcal{N} \cup \mathcal{S}}{\sum} \frac{ (1-m_{i}^{(t)})({{D}}_{i}^{(t)} -1) \Theta_i^2 (\Tilde{\sigma}_{i}^{(t)})^2 p_i^{(t)} \gamma_i^{(t)}} {m_{i}^{(t)}  \|\mathbf{a}_{i}^{(t)}\|_{1} ({{D}}_{i}^{(t)})^2} (\|\mathbf{a}_{i}^{(t)}\|_{2}^2 - [a_{i,-1}^{(t)}]^2)  \\
    & \nonumber + 2 \eta^2 L^2 \big(\underset{t \in [T]}{\max} \ \underset{i \in \mathcal{N} \cup \mathcal{S}}{\max} (\frac{\gamma_i^{(t)})^2 (\|\mathbf{a}_{i}^{(t)}\|_{1} - [a_{i,-1}^{(t)}])}{\|\mathbf{a}_{i}^{(t)}\|_{1}} \big)\zeta_2(1 + \frac{1}{2 \zeta_1^2}) \\
    & \hskip -2cm \leq \nonumber -\frac{1}{4} \underbrace{\|\nabla {{F}}^{(t)}(\mathbf{x}^{(t)})\|^2}_\text{(a)} \nonumber + \vartheta \eta L  \Bigg[\underset{i \in \mathcal{N} \cup \mathcal{S}}{\sum}  \frac{4(p_{i}^{(t)})^2 (1-m_{i}^{(t)})({{D}}_{i}^{(t)} -1) \Theta_i^2 (\Tilde{\sigma}_{i}^{(t)})^2}{m_{i}^{(t)} ({{D}}_{i}^{(t)})^2}  \frac{\|\mathbf{a}_{i}^{(t)}\|_{2}^2}{\|\mathbf{a}_{i}^{(t)}\|_{1}^2}  \Bigg]\\
    & \nonumber + {3 \eta^2 L^2} \underset{i \in \mathcal{N} \cup \mathcal{S}}{\sum} \frac{ (1-m_{i}^{(t)})({{D}}_{i}^{(t)} -1) \Theta_i^2 (\Tilde{\sigma}_{i}^{(t)})^2 p_i^{(t)} \gamma_i^{(t)}} {m_{i}^{(t)}  \|\mathbf{a}_{i}^{(t)}\|_{1} ({{D}}_{i}^{(t)})^2} (\|\mathbf{a}_{i}^{(t)}\|_{2}^2 - [a_{i,-1}^{(t)}]^2)  \\
    &  + 3 \eta^2 L^2 \zeta_2 \Big(\underset{t \in [T]}{\max} \ \underset{i \in \mathcal{N} \cup \mathcal{S}}{\max} \frac{(\gamma_i^{(t)})^2 (\|\mathbf{a}_{i}^{(t)}\|_{1} - [a_{i,-1}^{(t)}])}{\|\mathbf{a}_{i}^{(t)}\|_{1}}\Big) \label{eqn: gradient bound exp 1}.
\end{align}
Consequently, we bring term (a) in equation \eqref{eqn: gradient bound exp 1} to the LHS and compute average over all global rounds of aggregation i.e. $t = 0, 1,...., T-1$ to obtain the following:
\begin{align}
     \nonumber  \frac{1}{T} \sum_{t = 0}^{T - 1} \mathbb{E} \bigg[  \|\nabla {{F}}^{(t)}(\mathbf{x}^{(t)})\|^2 \bigg] & \leq \underbrace{\Bigg[\frac{4}{T} \sum_{t = 0}^{T -1} \frac{{{F}}^{(t)}(\mathbf{x}^{(t)}) - \mathbb{E}[{{F}}^{(t)}(\mathbf{x}^{(t+1)})]}{\vartheta \eta} \Bigg]}_\text{(a)} \\ 
     & \nonumber + 16 \eta L \vartheta \Bigg[\frac{1}{T} \sum_{t = 0}^{T - 1}\underset{i \in \mathcal{N} \cup \mathcal{S}}{\sum}   \frac{(p_{i}^{(t)})^2 (1-m_{i}^{(t)})({{D}}_{i}^{(t)} -1) \Theta_i^2 (\Tilde{\sigma}_{i}^{(t)})^2}{m_{i}^{(t)} ({{D}}_{i}^{(t)})^2} \frac{\|\mathbf{a}_{i}^{(t)}\|_{2}^2}{\|\mathbf{a}_{i}^{(t)}\|_{1}^2}  \Bigg] \\ 
     & \nonumber + {12 \eta^2 L^2} \Bigg[\frac{1}{T} \sum_{t = 0}^{T - 1}\underset{i \in \mathcal{N} \cup \mathcal{S}}{\sum} \frac{ (1-m_{i}^{(t)})({{D}}_{i}^{(t)} -1) \Theta_i^2 (\Tilde{\sigma}_{i}^{(t)})^2 p_i^{(t)} \gamma_i^{(t)}} {m_{i}^{(t)}  \|\mathbf{a}_{i}^{(t)}\|_{1} ({{D}}_{i}^{(t)})^2} (\|\mathbf{a}_{i}^{(t)}\|_{2}^2 - [a_{i,-1}^{(t)}]^2) \Bigg] \\
    &  + 12 \eta^2 L^2 \zeta_2 \Big(\underset{t \in [T]}{\max} \ \underset{i \in \mathcal{N} \cup \mathcal{S}}{\max} \frac{(\gamma_i^{(t)})^2 (\|\mathbf{a}_{i}^{(t)}\|_{1} - [a_{i,-1}^{(t)}])}{\|\mathbf{a}_{i}^{(t)}\|_{1}}\Big) \label{eqn: gradient bound exp 2}.
\end{align}
We highlight that, it is critical to show that term (a) in \eqref{eqn: gradient bound exp 2} is a bounded series to obtain a finite upper bound for $\frac{1}{T} \sum_{t = 0}^{T - 1} \mathbb{E} \bigg[  \|\nabla {{F}}^{(t)}(\mathbf{x}^{(t)})\|^2 \bigg]$. More specifically, this implies that the drift associated with the local loss functions due to the delay of global aggregations of \texttt{CE-FL} must be bounded, leading to a telescoping bounded summation series. In this regard, our definition of concept drift associated with local loss functions stated in Definition \ref{defn: model drift} allows us to obtain the following:   
\begin{align}
    & \frac{{{D}}_i^{(t+1)}}{{D}^{(t+1)}}F_{i}^{(t+1)}(\mathbf{x}^{(t+1)}) \leq \tau^{(t)} \Delta_i^{(t)} + \frac{{{D}}_i^{(t)}}{{{D}}^{(t)}}F_{i}^{(t)}(\mathbf{x}^{(t+1)}) \label{eqn: drift 1} \\
    & \sum_{i \in \mathcal{N} \cup \mathcal{S}}  \frac{{{D}}_i^{(t+1)}}{{D}^{(t+1)}}F_{i}^{(t+1)}(\mathbf{x}^{(t+1)}) \leq \sum_{i \in \mathcal{N} \cup \mathcal{S}} \tau^{(t)} \Delta_i^{(t)} + \sum_{i \in \mathcal{N} \cup \mathcal{S}} \frac{{{D}}_i^{(t)}}{{{D}}^{(t)}}F_{i}^{(t)}(\mathbf{x}^{(t+1)}) \\
    &\Rightarrow {F}^{(t+1)}(\mathbf{x}^{(t+1)}) \leq \sum_{i \in \mathcal{N} \cup \mathcal{S}} \tau^{(t)} \Delta_i^{(t)} + {F}^{(t)}(\mathbf{x}^{(t+1)}). \label{eqn: drift 3}
\end{align}
In \eqref{eqn: drift 1}-\eqref{eqn: drift 3},  we note that $\tau^{(t)}$ corresponds to the delay incurred due to data offloading, ML processing, parameter transfer and aggregation between consecutive rounds $t$ and $t+1$ of \texttt{CE-FL}. 
Therefore, we plug in \eqref{eqn: drift 3} into term (a) of \eqref{eqn: gradient bound exp 1} which leads to the final expression bounding the global loss gradients across all the rounds of aggregation in \texttt{CE-FL}:
\begin{align}
    \nonumber  \frac{1}{T} \sum_{t = 0}^{T - 1} \mathbb{E} \bigg[  \|\nabla {{F}}^{(t)}(\mathbf{x}^{(t)})\|^2 \bigg] & \leq \frac{4}{\vartheta \eta T} \Bigg[ {{{F}}^{(0)}(\mathbf{x}^{(0)}) - {{F}}^{(T-1)}(\mathbf{x}^{(T)})} \Bigg] + \frac{4}{\vartheta \eta T} \sum_{t = 0}^{T - 1} \sum_{i \in \mathcal{N} \cup \mathcal{S}} \tau^{(t)} \Delta_i^{(t)} \\
    & \nonumber + 16 \eta L \vartheta \Bigg[\frac{1}{T} \sum_{t = 0}^{T - 1}\underset{i \in \mathcal{N} \cup \mathcal{S}}{\sum}   \frac{(p_{i}^{(t)})^2 (1-m_{i}^{(t)})({{D}}_{i}^{(t)} -1) \Theta_i^2 (\Tilde{\sigma}_{i}^{(t)})^2}{m_{i}^{(t)} ({{D}}_{i}^{(t)})^2} \frac{\|\mathbf{a}_{i}^{(t)}\|_{2}^2}{\|\mathbf{a}_{i}^{(t)}\|_{1}^2}  \Bigg] \\ 
     & \nonumber + {12 \eta^2 L^2} \Bigg[\frac{1}{T} \sum_{t = 0}^{T - 1}\underset{i \in \mathcal{N} \cup \mathcal{S}}{\sum} \frac{ (1-m_{i}^{(t)})({{D}}_{i}^{(t)} -1) \Theta_i^2 (\Tilde{\sigma}_{i}^{(t)})^2 p_i^{(t)} \gamma_i^{(t)}} {m_{i}^{(t)}  \|\mathbf{a}_{i}^{(t)}\|_{1} ({{D}}_{i}^{(t)})^2} (\|\mathbf{a}_{i}^{(t)}\|_{2}^2 - [a_{i,-1}^{(t)}]^2) \Bigg] \\
    &  \nonumber + 12 \eta^2 L^2 \zeta_2 \Big(\underset{t \in [T]}{\max} \ \underset{i \in \mathcal{N} \cup \mathcal{S}}{\max} \frac{(\gamma_i^{(t)})^2 (\|\mathbf{a}_{i}^{(t)}\|_{1} - [a_{i,-1}^{(t)}])}{\|\mathbf{a}_{i}^{(t)}\|_{1}}\Big) \\
    & \nonumber \leq \frac{4}{\vartheta \eta T} \Bigg[ {{{F}}^{(0)}(\mathbf{x}^{(0)}) - {{F}}^{*}} \Bigg] + \frac{4}{\vartheta \eta T} \sum_{t = 0}^{T - 1} \sum_{i \in \mathcal{N} \cup \mathcal{S}} \tau^{(t)} \Delta_i^{(t)} \\
    & \nonumber + 16 \eta L \vartheta \Bigg[\frac{1}{T} \sum_{t = 0}^{T - 1}\underset{i \in \mathcal{N} \cup \mathcal{S}}{\sum}   \frac{(p_{i}^{(t)})^2 (1-m_{i}^{(t)})({{D}}_{i}^{(t)} -1) \Theta_i^2 (\Tilde{\sigma}_{i}^{(t)})^2}{m_{i}^{(t)} ({{D}}_{i}^{(t)})^2} \frac{\|\mathbf{a}_{i}^{(t)}\|_{2}^2}{\|\mathbf{a}_{i}^{(t)}\|_{1}^2}  \Bigg] \\ 
     & \nonumber + {12 \eta^2 L^2} \Bigg[\frac{1}{T} \sum_{t = 0}^{T - 1}\underset{i \in \mathcal{N} \cup \mathcal{S}}{\sum} \frac{ (1-m_{i}^{(t)})({{D}}_{i}^{(t)} -1) \Theta_i^2 (\Tilde{\sigma}_{i}^{(t)})^2 p_i^{(t)} \gamma_i^{(t)}} {m_{i}^{(t)}  \|\mathbf{a}_{i}^{(t)}\|_{1} ({{D}}_{i}^{(t)})^2} (\|\mathbf{a}_{i}^{(t)}\|_{2}^2 - [a_{i,-1}^{(t)}]^2) \Bigg] \\
    &  + 12 \eta^2 L^2 \zeta_2 \Big(\underset{t \in [T]}{\max} \ \underset{i \in \mathcal{N} \cup \mathcal{S}}{\max} \frac{(\gamma_i^{(t)})^2 (\|\mathbf{a}_{i}^{(t)}\|_{1} - [a_{i,-1}^{(t)}])}{\|\mathbf{a}_{i}^{(t)}\|_{1}}\Big), \label{eqn: final expression of ML conv}
\end{align}
where $F^* \triangleq \underset{t \in [T]}{\min}~\underset{\mathbf{x} \in \mathbb{R}^p }{\min}{F}^{(t)}(\mathbf{x})$.
\end{proof}

\section{Proof of Corollary 1} \label{proof corollary 1}
\begin{corollary} \label{corollary: step size appendix}
 Consider the conditions stated in Theorem \ref{Thm: ML convergence main}. Further assume that $\eta$ is small enough such that $\eta =\sqrt{{d}/({{\Bar{\gamma}T}}})$, where $d = |\mathcal{N} \cup \mathcal{S}|$ and {\small $\Bar{\gamma}= \displaystyle \sum_{t = 1}^{t= T} \sum_{i\in \mathcal{N} \cup \mathcal{S}} \gamma_i^{(t)}$}. If (i) the local data variability satisfy  $\Theta_i \leq \Theta_{\mathsf{max}}$, $\forall i$, for some positive $\Theta_{\mathsf{max}}$,  (ii) the variances of datasets satisfy $\Tilde{\sigma}_{i}^{(t)} \leq \Tilde{\sigma}_{\mathsf{max}}$, $\forall i$, for some positive $\Tilde{\sigma}_{\mathsf{max}}$, (iii) the mini-batch ratios satisfy ${m_{i}^{(t)}} \geq m_{\mathsf{min}}$, $\forall i$, for some positive $m_{\mathsf{min}}$, (iv) the number of SGD iterations satisfy $\gamma_i^{(t)} \leq \gamma_{\mathsf{max}},~\forall i$, for some positive $\gamma_{\mathsf{max}}$,
 and (v) the duration of global aggregations is bounded as {\small $ \displaystyle \tau^{(t)} \leq \max \big\{ \frac{\tilde{\tau}}{T\sum_{i \in \mathcal{N} \cup \mathcal{S}} \Delta_i^{(t)}},0 \big\}$}, for some positive $\tilde{\tau}$, then the cumulative average of the global loss satisfies~\eqref{eq:convSpecial}, implying {\small $ \displaystyle \frac{1}{T} \sum_{t = 1}^{T} \mathbb{E} \Big[ \big\|\nabla {{F}}^{(t)}(\mathbf{x}^{(t)})\big\|^2\Big]= \mathcal{O}(1/\sqrt{T})$}.
\end{corollary}
\begin{proof}
Considering \eqref{eqn: final expression of ML conv}, the choice of $\eta =\sqrt{\frac{d}{{\Bar{\gamma} T}}}$ allows us to obtain the following:
\begin{align}
    \nonumber  \frac{1}{T} \sum_{t = 0}^{T - 1} \mathbb{E} \bigg[  \|\nabla {{F}}^{(t)}(\mathbf{x}^{(t)})\|^2 \bigg]  & \nonumber \leq \frac{4 \sqrt{{\Bar{\gamma}}}}{ \vartheta \sqrt{dT}} \Bigg[ {{{F}}^{(0)}(\mathbf{x}^{(0)}) - {{F}}^{*}} \Bigg] + \frac{4 \sqrt{{\Bar{\gamma}}}}{ \vartheta \sqrt{dT}} \sum_{t = 0}^{T - 1} \sum_{i \in \mathcal{N} \cup \mathcal{S}} \tau^{(t)} \Delta_i^{(t)} \\
    & \nonumber + 16 L \vartheta \sqrt{\frac{d}{{\Bar{\gamma}T}}} \Bigg[\frac{1}{T} \sum_{t = 0}^{ T-1}\underset{i \in \mathcal{N} \cup \mathcal{S}}{\sum}   \frac{(p_{i}^{(t)})^2 (1-m_{i}^{(t)})({{D}}_{i}^{(t)} -1) \Theta_i (\Tilde{\sigma}_{i}^{(t)})^2}{m_{i}^{(t)} ({{D}}_{i}^{(t)})^2} \frac{\|\mathbf{a}_{i}^{(t)}\|_{2}^2}{\|\mathbf{a}_{i}^{(t)}\|_{1}^2}  \Bigg] \\ 
     & \nonumber + {\frac{12  L^2d}{{\Bar{\gamma}T}}} \Bigg[\frac{1}{T} \sum_{t = 0}^{T - 1}\underset{i \in \mathcal{N} \cup \mathcal{S}}{\sum} \frac{ (1-m_{i}^{(t)})({{D}}_{i}^{(t)} -1) \Theta_i (\Tilde{\sigma}_{i}^{(t)})^2 p_i^{(t)} \gamma_i^{(t)}} {m_{i}^{(t)}  \|\mathbf{a}_{i}^{(t)}\|_{1} ({{D}}_{i}^{(t)})^2} (\|\mathbf{a}_{i}^{(t)}\|_{2}^2 - [a_{i,-1}^{(t)}]^2) \Bigg] \\
    &  +  {\frac{12  L^2 \zeta_2 d}{{\Bar{\gamma} T}}} \Big(\underset{t \in [T]}{\max} \ \underset{i \in \mathcal{N} \cup \mathcal{S}}{\max} \frac{(\gamma_i^{(t)})^2 (\|\mathbf{a}_{i}^{(t)}\|_{1} - [a_{i,-1}^{(t)}])}{\|\mathbf{a}_{i}^{(t)}\|_{1}}\Big). 
\end{align}
We have ${\|\mathbf{a}_{i}^{(t)}\|_{2}^2} \leq {\|\mathbf{a}_{i}^{(t)}\|_{1}^2}$ since $a_{{i,\ell}}^{(t)} \leq 1$. Furthermore, since $p_i^{(t)} < 1$ and $D_i^{(t)} > 1$, we get
\begin{align}
   \nonumber  \frac{1}{T} \sum_{t = 0}^{T - 1} \mathbb{E} \bigg[  \|\nabla {{F}}^{(t)}(\mathbf{x}^{(t)})\|^2 \bigg]  & \nonumber \leq \frac{4 \sqrt{{\Bar{\gamma}}}}{ \vartheta \sqrt{dT}} \Bigg[ {{{F}}^{(0)}(\mathbf{x}^{(0)}) - {{F}}^{*}} \Bigg] + \frac{4 \sqrt{{\Bar{\gamma}}}}{ \vartheta \sqrt{dT}} \sum_{t = 0}^{T - 1} \sum_{i \in \mathcal{N} \cup \mathcal{S}} \tau^{(t)} \Delta_i^{(t)} \\
    & \nonumber + 16 L \vartheta \sqrt{\frac{d}{{\Bar{\gamma}T}}} \Bigg[\underbrace{\frac{1}{T} \sum_{t = 0}^{T - 1}\underset{i \in \mathcal{N} \cup \mathcal{S}}{\sum}   \frac{  \Theta_i (\Tilde{\sigma}_{i}^{(t)})^2}{m_{i}^{(t)} }}_\text{(a)}   \Bigg] \\ 
     & \nonumber + {\frac{12  L^2d}{{\Bar{\gamma}T}}} \Bigg[\underbrace{\frac{1}{T} \sum_{t = 0}^{T - 1}\underset{i \in \mathcal{N} \cup \mathcal{S}}{\sum} \frac{  \Theta_i (\Tilde{\sigma}_{i}^{(t)})^2 \gamma_i^{(t)}} {m_{i}^{(t)} } }_\text{(b)} \Bigg] \\
    &  +  {\frac{12  L^2 \zeta_2 d}{{\Bar{\gamma}T}}} \underbrace{\Big(\underset{t \in [T]}{\max} \ \underset{i \in \mathcal{N} \cup \mathcal{S}}{\max} {(\gamma_i^{(t)})^2 }\Big)}_\text{(c)}. \label{eqn: corollary 2}
\end{align}
Henceforth, we leverage that local data variability parameter $\Theta_i$ and local data sample variances $\Tilde{\sigma}_{i}^{(t)}$ are bounded as $\Theta_i \leq \Theta_{\mathsf{max}}$ and $\Tilde{\sigma}_{i}^{(t)} \leq \Tilde{\sigma}_{\mathsf{max}}$ respectively. Furthermore, it is enforced that each device performs ML training involving at least part of the locally available dataset every round i.e. ${m_{i}^{(t)}} \geq m_{\mathsf{min}}$. Also, we denote  $\gamma_{\mathsf{max}} \triangleq \ \underset{i \in \mathcal{N} \cup \mathcal{S}}{\max} {(\gamma_i^{(t)})^2 } $, thereby simplifying terms (a), (b), (c) in \eqref{eqn: corollary 2} obtaining:
\begin{align}
     \nonumber  \frac{1}{T} \sum_{t = 0}^{T - 1} \mathbb{E} \bigg[  \|\nabla {{F}}^{(t)}(\mathbf{x}^{(t)})\|^2 \bigg]  & \nonumber \leq \frac{4 \sqrt{{\Bar{\gamma}}}}{ \vartheta \sqrt{dT}} \Bigg[ {{{F}}^{(0)}(\mathbf{x}^{(0)}) - {{F}}^{*}} \Bigg] + \frac{4 \sqrt{{\Bar{\gamma}}}}{ \vartheta \sqrt{dT}} \underbrace{\sum_{t = 0}^{T - 1} \sum_{i \in \mathcal{N} \cup \mathcal{S}} \tau^{(t)} \Delta_i^{(t)}}_\text{(d)} \\
     &  + 16 \frac{L \vartheta \Theta_{\mathsf{max}} \Tilde{\sigma}_{\mathsf{max}}^2 }{ m_{\mathsf{min}}}  \sqrt{\frac{d}{{\Bar{\gamma} T}}} + {\frac{12  L^2d \Theta_{\mathsf{max}} \Tilde{\sigma}_{\mathsf{max}}^2 \gamma_{\mathsf{max}}}{\Bar{\gamma} m_{\mathsf{min}} {T}}} + {\frac{12  L^2 \zeta_2 d \gamma_{\mathsf{max}}^2}{{\Bar{\gamma}T}}}. \label{eqn: corollary 3}
\end{align}
Now, our assumption regarding bounded delay between aggregation rounds $t$ and $t+1$ i.e. $ \displaystyle \tau^{(t)} \leq \frac{\tilde{\tau}}{T\sum_{i \in \mathcal{N} \cup \mathcal{S}} \Delta_i^{(t)}}$ allows us to bound term (d) in \eqref{eqn: corollary 3}, and thus obtaining the following bound:
\begin{align}
    \nonumber  \frac{1}{T} \sum_{t = 0}^{T - 1} \mathbb{E} \bigg[  \|\nabla {{F}}^{(t)}(\mathbf{x}^{(t)})\|^2 \bigg]  & \nonumber \leq \frac{4 \sqrt{{\Bar{\gamma}}}}{ \vartheta \sqrt{dT}} \Bigg[ {{{F}}^{(0)}(\mathbf{x}^{(0)}) - {{F}}^{*}} \Bigg] + \frac{4 \tilde{\tau} \sqrt{{\Bar{\gamma}}}}{ \vartheta \sqrt{dT}} + 16 \frac{L \vartheta \Theta_{\mathsf{max}} \Tilde{\sigma}_{\mathsf{max}}^2 }{ m_{\mathsf{min}}}  \sqrt{\frac{d}{{\Bar{\gamma} T}}} \\
     & \nonumber + {\frac{12  L^2d \Theta_{\mathsf{max}} \Tilde{\sigma}_{\mathsf{max}}^2 \gamma_{\mathsf{max}}}{\Bar{\gamma} m_{\mathsf{min}} {T}}} + {\frac{12  L^2 \zeta_2 d \gamma_{\mathsf{max}}^2}{{\Bar{\gamma}T}}} \\
     & = \mathcal{O}(1/\sqrt{T}),
\end{align}
which concludes the proof.
\end{proof}
\section{Decentralized Distributed Optimization Solution Convergence} \label{NOVA stationary proof}
\begin{theorem} If $J\rightarrow \infty$ (see Algorithm~\ref{alg:dist_opt_consensus_algo}), the sequence $\{\bm{w}^{(\ell)}\}$ generated by Algorithm \ref{alg:dist_NOVA main} is feasible for $\bm{\mathcal{P}}$ and non-increasing, which asymptotically reaches a stationary solution of $\bm{\mathcal{P}}$.
\end{theorem}
\begin{proof}
Problem $\bm{\mathcal{P}}$, i.e.,  the network resource allocation and server selection joint optimization  as summarized in \eqref{eqn: non-convex constraints}-\eqref{eqn: convex constraints}, is solved via successive convex solver method  described in Algorithm \ref{alg:dist_NOVA main}. During each round $\ell$ of the iterative procedure, a surrogate convex optimization problem $\widehat{\bm{\mathcal{P}}}_{\bm{{w}}^{(\ell)}}$ is formed via convex relaxation of the objective function and non-convex constraints associated with original problem $\bm{\mathcal{P}}$ at current iterate $\bm{w}^{(\ell)}$. Therefore, it is necessary to show that Algorithm \ref{alg:dist_NOVA main} solves the sequence of surrogate problems $\{\widehat{\bm{\mathcal{P}}}_{\bm{{w}}^{(\ell)}}\}$ in order to produce a non-increasing sequence of feasible solutions of the original problem $\bm{\mathcal{P}}$. To this end, we need to prove that the following statements hold \cite{scutari2016parallel, MultiTierAbubakr}:
\begin{enumerate}
    \item Given the current feasible solution ${\bm{w}}^{(\ell)}$ at each round  $\ell$ of Algorithm \ref{alg:dist_NOVA main}, Algorithm \ref{alg:dist_opt_algo main} generates the sequence of feasible solutions  $\{\hat{\bm{w}}^{[i]}({\bm{w}}^{(\ell)})\}$ converging to a stationary solution of $\widehat{\bm{\mathcal{P}}}_{\bm{{w}}^{(\ell)}}$.
    \item At each iteration $\ell$ of Algorithm \ref{alg:dist_NOVA main}, the sequence of solutions generated by Algorithm \ref{alg:dist_opt_algo main}, i.e. $\{\hat{\bm{w}}^{[i]}({\bm{w}}^{(\ell)})\}$,  are feasible for the original problem $\bm{\mathcal{P}}$.
    \item The sequence $\{{\bm{w}}^{(\ell)}\}$ generated by Algorithm \ref{alg:dist_NOVA main} via  \eqref{eqn:NOVA param update} produce a non-increasing sequence of objective values of $\bm{\mathcal{P}}$ and asymptotically converges to a stationary solution as number of iterations $J\rightarrow \infty$.
\end{enumerate}
\paragraph*{\textbf{Proof of statement 1}} Algorithm \ref{alg:dist_opt_algo main} is an alternating optimization approach consisting of two stages: primal variable update and dual variable update. At each iteration $i$, the first stage comprises of Algorithm \ref{alg:dist_opt_algo main} parallely solving the primal problem \eqref{eqn: Primal Problem} at current dual iterates $\{\bm{\Lambda}_{{d}}$, $\bm{\Omega}_{{d}} \}_{{d} \in \mathcal{N} \cup \mathcal{S} \cup \mathcal{B}}$. More specifically, each network node $d$ minimizes its convex partial Lagrangian objective, i.e., ${\mathcal{L}}(\bm{w}_{{d}}, \bm{\Lambda}^{[i-1]}_{{d}}, \bm{\Omega}^{[i-1]}_{{d}} ;{\bm{{w}}^{(\ell)}})$ with convex constraints $\bm{\mathcal{D}}_{{d}}(\bm{w}_{{d}})\leq \mathbf{0}$ via gradient projection algorithm \cite{GPAarticle}. Therefore, an optimal solution of \eqref{eqn: Primal Problem} is guaranteed to be achieved as a consequence of Theorem 1.1 in \cite{GPAarticle}.

The second stage of the alternating optimization described in Algorithm \ref{alg:dist_opt_algo main} consists of dual updates mathematically expressed via \eqref{eqn: Lambda update}-\eqref{eqn: Omega update}. However, the absence of a central entity in the network leads to a decentralized consensus based update given by \eqref{eqn: Lambda update Local}-\eqref{eqn: combined_consensus_update}. We note that this update is achieved via Algorithm \ref{alg:dist_opt_algo main} by leveraging communication among the network nodes. At each iteration $i$ of Algorithm \ref{alg:dist_opt_algo main},  \underline{Consensus \texttt{CE-FL}} algorithm performs initialization as $\bm{\Lambda}^{\{ 0 \}}_{d} = \bm{\Lambda}^{[i]}_d$, $\bm{\Omega}^{\{ 0 \}}_{d} = \bm{\Omega}^{[i]}_d$, where local copies $\bm{\Lambda}^{[i]}_d, \bm{\Gamma}^{[i]}_d$ are given by \eqref{eqn: Lambda update Local}-\eqref{eqn: Omega update Local}. Subsequently, it performs several rounds of consensus among the network nodes via \textit{consensus matrix} $\bm{{W}}=[{W}_{d,d'}]_{d,d'\in\mathcal{N}\cup\mathcal{B}\cup\mathcal{S}}$ thereby creating sequence of iterates $\{ \bm{\Lambda}^{\{ j \}}_{d}, \bm{\Omega}^{\{ j \}}_{d} \}$. We note that the construction of consensus weights and subsequently $\bm{{W}}$ as described in the text follows \textit{Assumption} 2 in \cite{johansson2008subgradient}, thereby satisfying the necessary conditions stated in Theorem 1 in \cite{xiao2004fast} which allows us to obtain the following:
\begin{align}
    & \lim_{j \rightarrow \infty} ~\bm{\Lambda}^{\{ j \}}_{d} = \frac{1}{|\mathcal{N} \cup \mathcal{S} \cup \mathcal{B}|}\sum_{d \in \mathcal{N} \cup \mathcal{S} \cup \mathcal{B}} \bm{\Lambda}^{[i]}_d \label{eqn: Lambda convergence},\\
   & \lim_{j \rightarrow \infty} ~\bm{\Omega}^{\{ j \}}_{d} = \frac{1}{|\mathcal{N} \cup \mathcal{S} \cup \mathcal{B}|}\sum_{d \in \mathcal{N} \cup \mathcal{S} \cup \mathcal{B}} \bm{\Omega}^{[i]}_d \label{eqn: Omega convergence}.
\end{align}
Therefore, at each iteration $i$ of Algorithm \ref{alg:dist_opt_algo main}, the result stated in \eqref{eqn: Lambda convergence}-\eqref{eqn: Omega convergence} implies that actual central update in \eqref{eqn: Lambda update}-\eqref{eqn: Omega update} is asymptotically achieved via sufficiently large number of iterations of Algorithm \ref{alg:dist_opt_consensus_algo} \underline{Consensus \texttt{CE-FL}}.

Hence, we have both the optimality of the solution produced by gradient projection algorithm on convex primal problem during first stage and asymptotic convergence of dual updates to corresponding global updates during second stage of alternating optimization procedure \underline{PD \texttt{CE-FL}}. Therefore, under appropriately chosen step-sizes pertaining to dual variables, Theorem 4 described in \cite{scutari2016parallel} ensures convergence to a stationary solution of $\widehat{\bm{\mathcal{P}}}_{\bm{{w}}^{(\ell)}}$ via Algorithm \ref{alg:dist_opt_algo main} during each iteration $\ell$ of Algorithm \ref{alg:dist_NOVA main}.        

\paragraph*{\textbf{Proof of statement 2}} The convexification of the constraints as performed via  \eqref{eqn:prox grad non convex constr 1}-\eqref{eqn:prox grad non-convex constr 2} is constructed via Lipschitz continuity based approximations described in Sec III (A) in \cite{scutari2016parallel}. This ensures that the convexified constraints upper-bound the original non-convex constraints while satisfying Assumption 3 in \cite{scutari2016parallel}, i.e., ${\widetilde{\bm{\mathcal{C}}}}(\bm{w};\bm{{w}}^{(\ell)}) \geq \bm{\mathcal{C}}(\bm{w})$. As a result Lemma 9 in \cite{scutari2016parallel} holds, which ensures feasibility of the solution generated by Algorithm \ref{alg:dist_opt_algo main} for the original problem $\bm{\mathcal{P}}$.
\paragraph*{\textbf{Proof of statement 3}} The objective function and the constraints associated with $\bm{\mathcal{P}}$ satisfy Assumptions 1-3 and 5 in \cite{scutari2016parallel} and the iterates generated by Algorithm \ref{alg:dist_opt_algo main} converge to a stationary point of surrogate convex problem, leading to straightforward application of Theorem 2 in \cite{scutari2016parallel} which guarantees non-increasing property of the sequence generated by Algorithm \ref{alg:dist_NOVA main} for objective values of network optimization problem $\bm{\mathcal{P}}$. Thus, as $J\rightarrow \infty$, the non-increasing sequence converges to a stationary solution of $\bm{\mathcal{P}}$. 
\end{proof}
\newpage

\section{Methodology for Measuring over CBRS and 5G Networks} \label{5G appendix}
We choose two sites based on the availability of 5G and CBRS edge testbeds to capture the effect of varying actual radio propagation and backhaul connections across operational 5G/4G networks.
We collect data on the transport layer's throughput and latency statistics at both ends, as well as cellular data from CBRS USB modems and 5G phones, as well as data on the base station's power consumption.

\subsection{CBRS and 5G Edge Testbeds}
We utilize a fully functional end-to-end 5G network testbed in the Indy 5G Zone \cite{indy5gzone} in the United States. This indoor testbed is approximately 150$m^2$ in size and is comprised of commercial 5G gNBs and 4G eNBs, a local gateway, and edge computing servers placed by AT\&T.
Additionally, we utilize a fully operating end-to-end CBRS network testbed on the Purdue University campus \cite{cbrsconvergence}. This indoor and outdoor testbed comprised of commercial 4G eNBs, an edge computing gateway, and edge computing servers deployed by SBA Communications.
4G eNBs operate at frequencies of 700MHz, 1.9GHz in the Indy 5G Zone, and 3.5GHz on campus.
The architecture is based on 3GPP NSA option 3 (as seen in Fig. \ref{fig:5gtestbed}), in which UEs are connected to the operational core network via the current LTE/EPC control plane. Each testbed connects to its edge server via a local breakout. Local gateways are placed on-site to connect cellular base stations to the carrier's centralized LTE core. Due to the fact that the local aggregation point routes traffic to the edge computing servers located directly alongside the 5G gNB and 4G eNB, this network deployment option enables us to experiment with an edge computing scenario. This testbed is used to measure throughput and latency at the transport layer in the context of 5G edge computing.

\begin{figure}[h]
  \includegraphics[width=0.8\textwidth]{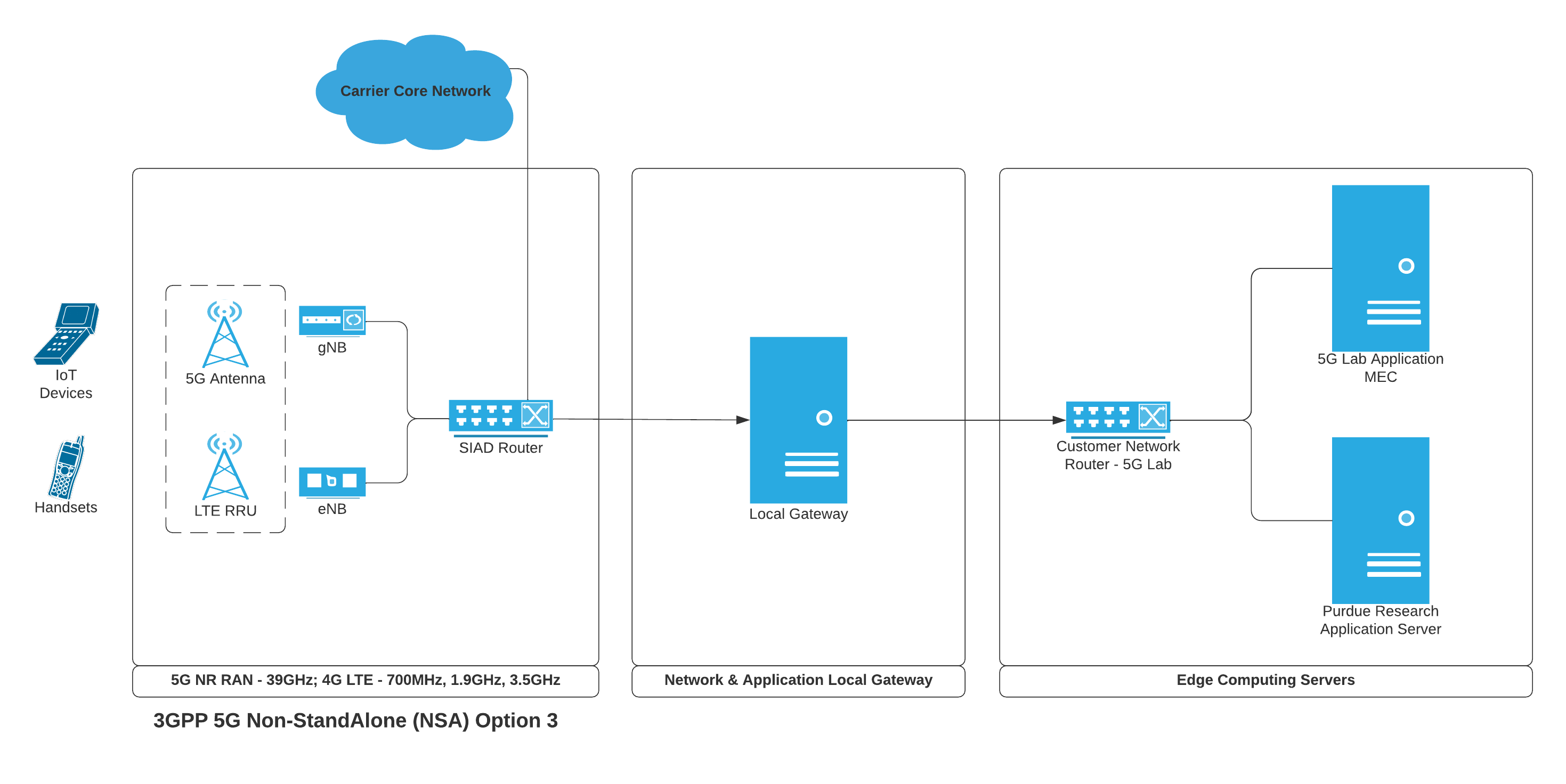}
  \centering
  \caption{Architecture of CBRS/5G network edge testbeds.}
  \label{fig:5gtestbed}
\end{figure}

\subsection{Measurement Setup}
In terms of server setup, we use the edge server at the Indy 5G Zone, the edge server at the CBRS network, and servers from CloudLab \cite{cloudlab}, a provider of high-performance computing and networking services throughout various states in the United States. Between our 5G phone and edge/cloud servers in the United States, we measured an average latency of between 5ms and 75ms, depending on server location. We utilized servers distributed across several locations to represent the spread of commercial cloud computing resources and the predicted performance range.

In terms of client setup, we use a Samsung Galaxy S20 Ultra (Qualcomm Snapdragon 865 \cite{snapdragon865}) and a MultiConnect® microCell LTE USB Cellular Modem as user equipments (UEs). The phones and USB modems are  operating at sub-6GHz frequencies. Android 12 is pre-installed on all phones, along with Linux kernel 4.14.19. However, we employ a TCP congestion control algorithm implemented in the Linux kernel on the servers to measure throughput and latency at the transport layer across CBRS networks, and we use phones and USB modems as receivers.

\subsection{Data Collection}
We generate TCP/UDP traffic using iperf3. To access extensive TCP information, we altered \texttt{tcp\_info.c} and cross-compiled iperf3. To gather more internal TCP variables at a finer resolution, we modified the \texttt{tcp\_probe} module's \texttt{tcp\_probe.c} file.

We gather data on achieved throughput and response time, as well as their variances 1) between devices and edge/Wisconsin/Clemson/Utah servers; 2) between the CBRS eNB and edge/Wisconsin/Clemson/Utah servers; and 3) between the Indy 5G edge server and Wisconsin/Clemson/Utah servers.
We omit the first ten seconds of each experiment from the average throughput and latency calculations to eliminate starting effects.
Additionally, we gather data on CBRS eNB radio unit's transmission power.


\subsection{5G testbed data processing} \label{5G scaled up dataset creation}
We further process the data obtained from the 5G testbed to create a network consisting of 20 UEs, 10 BSs and the original 5 edge servers. For each UE-BS, BS-DC and DC-DC network edge, we construct Normal distributions for both uplink and downlink transfer rates by leveraging the emperical mean and variance of the collected data pertaining to the corresponding link. Subsequently, we populate the rate characteristics of each  edge of the expanded network using i.i.d sampling from the obtained normal distributions.

\section{\texttt{CE-FL} network setup and Datasets for experiments} \label{Network Settings for experiments}
\subsection{ML datasets for Experiments}
For our ML experiments, we have used Fashion-MNIST \cite{xiao2017/online} and CIFAR-10 \cite{krizhevsky2009learning} multi-class image classification datasets. Fashion-MNIST comprises of 60k training samples and CIFAR-10 contains 50k training samples, whereas both the datasets consist of 10k testing samples. In both the datasets, each image belongs to 10 classification labels. Furthermore, each example in Fashion-MNIST is a 28 $\times$ 28 grayscale image. In CIFAR-10, each example is 32 $\times$ 32 RGB image. During each round of global aggregation, UEs acquire datasets
with sizes distributed according to normal distribution with mean 2000 and variance 200 containing only 5 out of the 10 classification labels to resemble a non-iid data distribution.  

\subsection{Description of default network settings for conducted experiments}
For the purpose of our simulations, the default network consists of 20 UE devices, 10 Base Stations and 5 edge servers. Furthermore, we have used Fashion-MNIST dataset for all our proof-of-concept experiments. In Table \ref{table:cefl_params}, we summarize all the network/system characteristics and ML training parameter values that were used in our experimental evaluations:
\begin{table}[ht]
\centering
\caption{Network and system characteristic parameter values}
\begin{tabular}[t]{c c | c c | c c}

\toprule[.2em]
\textbf{Param} & \textbf{Value} & \textbf{Param} & \textbf{Value} & \textbf{Param} & \textbf{Value} \\
\hline
$\eta$ & 0.001 & $\beta^{\mathsf{M}}$ & 6272  & $\Delta_{i}$ &  0.3 \\
$\mu$ & 0.01 & $\beta^{\mathsf{D}}$ & $4 \times 10^7$  & $C_s$ & $5 \times 10^6 $ \\
$\vartheta$ & 0.01 &  $f_{n}^{\mathsf{min}}$ & 100KHz  & $\varrho$ & 0.4 \\
$\kappa$ & 0.001   &  $f_{n}^{\mathsf{max}}$ & 2.3GHz  & $ P_s $ & 200W \\ 
$\epsilon$ &  0.001   & $\alpha_n$ & 2e-16  &  $ M_s $ & 700 \\ 
$\zeta$ & 0.01  & $c_n$ & 300 &  $z$ & 0.02 \\  
\hline
\end{tabular}
\label{table:cefl_params}
\end{table}%
\subsection{Network graph construction for decentralized consensus}
For our decentralized consensus mechanism used in Sec.~\ref{sec:distSolv}, we simulate a communication graph by considering all possible UE-BS, BS-DC and DC-DC edges and include each edge in the graph with probability $p=0.3$. Furthermore, in order to ensure that the network nodes are strongly connected, the randomized graph construction procedure also enforces that each UE is connected to atleast one BS : $ \exists b \in \mathcal{B}: ~{{A}}_{n, b} = 1,~\forall n\in\mathcal{N}$. Similarly, each BS is connected to at least one DC:
$\exists s \in \mathcal{S}:  {{A}}_{b, s} = 1,~\forall b\in\mathcal{B}$; and, each DC is at least connected to another DC: $
   \exists s'\in\mathcal{S}: {{A}}_{s, s'} = 1, ~\forall s \in \mathcal{S}$.
\clearpage
\section{\texttt{CE-FL} ML Parameter Estimation } \label{App:param_est}

{\color{black}
 We note that Assumptions 1-3 introduce parameters $L, \{\Theta_i\}, \{\zeta_1, \zeta_2\}$ to characterize smoothness, local variability, and inter-node dissimilarity, respectively, for the ML convergence characterization. We highlight that finding the \textit{exact} values for this parameters is impractical in real-world systems as it will require measuring the loss/gradient over all realizations of ML model parameters over all the datapoints across all the devices. Nevertheless, we can \textit{estimate} these parameters using practical methods suitable for real-world systems and use them in our algorithms. To estimate these parameters, we introduce two parameter estimation procedures that can be potentially deployed in our network setting. The first procedure is one shot estimation and the second is dynamic estimation. We later show that one shot estimation is enough in our network setting as it leads to reliable estimation results.
\\

\noindent \underline{\textit{1)} One-Shot Pre-training Parameter Estimation Methodology:} In this procedure, the parameters $L, \{\Theta_i\}, \{\zeta_1, \zeta_2\}$ are estimated before the start of machine learning model training procedure. These parameters are estimated at the data processing units (DPUs), i.e., UEs and DCs. More specifically, at an arbitrary data processing unit $i \in \mathcal{N} \cup \mathcal{S}$, a small representative dataset denoted by $\Tilde{\mathcal{D}}_i$ is separately held to estimate the aforementioned ML smoothness and variability parameters. It is worth highlighting that $L, \{\zeta_1, \zeta_2\}$ are global (i.e., cross-device) parameters defined across all the devices at the network. Whereas, $\{\Theta_i\}$ are device-specific and only determine data variability locally at the DPUs. Furthermore, note that these parameters are always defined in the upper bounds in Assumptions 1-3, which implies that once a certain range for them is obtained experimentally, in implementation we can use slightly larger values to ensure the validity of upper bounds. For example, if the value obtained for $\zeta_1$ is $10$, we can use $15$ in the algorithm and it will still satisfy Eq. (40). Below, we will summarize how each of these parameters are estimated.

\textit{(i) Estimation of $\{\Theta_i\}$:} We first describe how local data variability parameter $\Theta_i$ is estimated at each individual DPUs $i \in \mathcal{N} \cup \mathcal{S}$. According to Assumption 2, we have:

\begin{align}\label{eq:def:datavar}
   \hspace{-2.1mm} \|\nabla {f}(\mathbf{x} ; \xi) \hspace{-.2mm}- \hspace{-.2mm} \nabla {f}(\mathbf{x} ; \xi')\| \hspace{-.2mm}\leq\hspace{-.2mm} \Theta_{i} \|\xi - \xi'\|, \hspace{-.1mm} \forall \xi,\hspace{-.2mm}\xi' \hspace{-1mm}\in \mathcal{D}_{i}^{(t)},\forall t.\hspace{-2mm}
\end{align}

To estimate this parameter, we use the Monte-Carlo based parameter estimation procedure detailed in Alg.~\ref{alg:local_theta_est}, which estimates the local data variability based on the above definition. Note that in practice, we can use the maximum estimated value of $\{\hat{\Theta}_i\}$ obtained in the algorithm as ${\Theta}_i$ for each DPU $i$ since that will satisfy the upper bound definition. Also, note that the sampled dataset used in the estimation is obtained from ${\mathcal{D}}^{(0)}_i$ (i.e., the initial dataset) at each DPU since the parameter is estimated before the execution of model training. Thus, to ensure that \eqref{eq:def:datavar} holds for ${\mathcal{D}}^{(t)}_i$, for the duration of the machine learning model training $t\in[0,T]$, we will scale the obtained value and use a slightly larger value in our algorithm execution (we scale the parameter by $1.5$ as discussed in the subsequent plots). Note that for the rest of the parameters the same method is applied. Our algorithm has also shown its adaptability to this choice of practical parameter estimation as can be seen by its performance gains obtained in the simulation section of the manuscript.

\begin{algorithm}[h]
	\caption{Monte-Carlo-based estimation of $\Theta_i$ at DPU $i$} \label{alg:local_theta_est}
    \begin{algorithmic}[1]
          \STATE  \textbf{Input:} Total number of iterations $J$, sampled local dataset $\Tilde{\mathcal{D}}_i$ used for parameter estimation, datapoints of which are chosen uniformly at random from $\mathcal{D}^{(0)}_i$. 
          \STATE \textbf{Output:} Final averaged estimate $\hat{\Theta}_i$.
          \STATE Iteration count $j \leftarrow 1$ 
          \WHILE{$j \leq J$}
          \STATE Randomly generate a model parameter weight vector $\mathbf{x}$
          \STATE Compute $\Theta_i^j = \frac{1}{|\Tilde{\mathcal{D}}_i|(|\Tilde{\mathcal{D}}_i| - 1)} \sum_{\xi, \xi' \in \Tilde{\mathcal{D}}_i} \frac{ \|\nabla {f}(\mathbf{x} ; \xi)- \nabla {f}(\mathbf{x} ; \xi')\|}{ \|\xi - \xi'\|}$
          \STATE $j \longleftarrow j + 1$
          \ENDWHILE
          \STATE $\hat{\Theta}_i = \max_{1\leq j \leq J} \{ \Theta^j_i\}$
    \end{algorithmic}
\end{algorithm}
\vspace{3mm}

\textit{(ii) Estimation of $L$:} Next, we focus on how the smoothness parameter $L$ is estimated across the DPUs. Assumption 1 defines the mathematical relationship between the loss functions and smoothness parameter:
\begin{equation}   \label{eqn:L_est}
 \| \nabla F_{i}^{(t)}(\mathbf{x}) - \nabla F_{i}^{(t)}(\mathbf{y}) \| \leq L \|\mathbf{x}-\mathbf{y}\|,
\end{equation}
where $\mathbf{x}$ and $\mathbf{y}$ are arbitrary machine learning model parameters.
Hence, using the locally held sampled datasets for parameter estimation $\{\Tilde{\mathcal{D}}_i\}$,  we use a Monte-Carlo-based local estimation followed by a global consensus to obtain a common smoothness parameter approximation $\hat{L}$ across the DPUs. Our proposed procedure is summarized in in Alg.~\ref{alg:local_smoothness_est}. In this algorithm, each DPU $i$ first computes the local estimate of $L$, denoted by $\hat{L}_i$ using \eqref{eqn:L_est}. Then, all the local estimates are sent to a predetermined DC $s^{\mathsf{est}}\in\mathcal{S}$ used for parameter estimation that takes the maximum over the local estimates to obtain the global estimate of smoothness parameter $\hat{L}$. Finally, we use slighter larger value for $\hat{L}$ (scaled by $1.5$) in our algorithm to make sure that the upper bound in \eqref{eqn:L_est} holds during the execution of the machine learning model training. 


\begin{algorithm}[h]
	\caption{Monte-Carlo-based estimation of ${L}$ across DPUs} \label{alg:local_smoothness_est}
    \begin{algorithmic}[1]
          \STATE  \textbf{Input:} Total number of iterations $J$, sampled local dataset $\{\Tilde{\mathcal{D}}_i\}_{i \in \mathcal{N}}$ used for parameter estimation, datapoints of which are chosen uniformly at random from $\mathcal{D}^{(0)}_i$.  
          \STATE \textbf{Output:} Final global estimate $\hat{L}$.
          \FOR{each device $i \in \mathcal{N}$}
          \STATE Iteration count $j \leftarrow 1$, 
          \WHILE{$j \leq J$}
          \STATE Randomly sample a pair of machine learning model vectors $\mathbf{x}_{j,1}, \mathbf{x}_{j,2}$
          \STATE Compute $L_i^j = \frac{\| \nabla F_{i}^{(t)}(\mathbf{x}_{j,1}) - \nabla F_{i}^{(t)}(\mathbf{x}_{j,2}) \|}{\|\mathbf{x}_{j,1}-\mathbf{x}_{j,2}\|}$
          \STATE $j \longleftarrow j + 1$
          \ENDWHILE
          \STATE $\hat{L}_i = \max_{1\leq j \leq J} \{ L^j_i\}$
          \STATE Share local estimate $\hat{L}_i$ to DC $s^{\mathsf{est}}$
          \ENDFOR
          \STATE Broadcast $\hat{L} = \max_{1\leq i \leq |\mathcal{N}|} \{\hat{L}_i\}$ from $s^{\mathsf{est}}$ to all the DPUs.
    \end{algorithmic}
\end{algorithm}


\textit{(iii) Estimation of $\{\zeta_1, \zeta_2\}$:} The bounded dissimilarity relationship across the local loss functions given in Assumption 3 is as follows:
\begin{align}\label{eq:zetas}
   \underset{i \in \mathcal{N} \cup \mathcal{S}}{\sum}  p_i \big\|\nabla F_{i}^{(t)}(\mathbf{x})\big\|^2 \leq \zeta_1 \Big\| \underset{i \in \mathcal{N} \cup \mathcal{S}}{\sum} p_i\nabla F_{i}^{(t)}(\mathbf{x})\Big\|^2 + \zeta_2, \forall t,
\end{align}
where $\mathbf{x}$ is an arbitrary machine learning model parameter.
We subsequently present a Monte-Carlo-based parameter estimation algorithm for obtaining the global estimates of $\{\hat{\zeta}_1, \hat{\zeta}_2\}$ in Alg. \ref{alg:local_zeta_est}. Similar to the above procedure for estimating $L$, the algorithm uses a predetermined DC $s^{\mathsf{est}}\in\mathcal{S}$ for parameter estimation. In particular, the algorithm estimates $\{\hat{\zeta}_1, \hat{\zeta}_2\}$ via an iterative procedure. In each iteration $j$, DC $s^{\mathsf{est}}$ first generates a random machine learning model parameter $\mathbf{x}_j$ and shares with all the DPUs $i \in \mathcal{N}$. Using $\mathbf{x}_j$, each DPU $i$ computes the gradient vector $ F^{(0)}_{i}(\mathbf{x}_j)$, given by (note that since the parameter estimation is run before the machine learning model training procedure, the initial dataset $\mathcal{D}^{(0)}_i$ is used for gradient computation)
 \begin{align}
 \label{eq:localLossAfter_R17}
    \nabla F_{i}^{(0)}(\mathbf{x}_j) =  \frac{1}{|\mathcal{D}^{(0)}_i|} \underset{\xi \in \mathcal{D}^{(0)}_i}{\sum} \nabla {f}(\mathbf{x}_j ; \xi).
\end{align}
After obtaining the local gradient, DPU $i$ shares the pair $\big\{ |\mathcal{D}^{(0)}_i|, \nabla F_{i}^{(0)}(\mathbf{x}_j) \big\}$ with DC $s^{\mathsf{est}}$.

DC $s^{\mathsf{est}}$ subsequently computes the pair $\Big\{ \underset{i \in \mathcal{N} }{\sum}  {p}_i \big\|\nabla F_{i}^{(0)}(\mathbf{x}_j)\big\|^2, \big\|\underset{i \in \mathcal{N}}{\sum} {p}_i\nabla F_{i}^{(0)}(\mathbf{x}_j)\Big\|^2 \Big\}$, where $p_i$ is given by
\begin{align}
   \label{eqn:dataset_weight} {p}_i = \frac{|\mathcal{D}^{(0)}_i|}{\underset{i \in \mathcal{N}}{\sum}{|\mathcal{D}^{(0)}_i|}}.
\end{align}
Finally, considering the relationship in~\eqref{eq:zetas}, at the end of $J$ iterations, $s^{\mathsf{est}}$ solves the following conventional \texttt{Linear Regression} optimization problem to obtain final global estimates $\{\hat{\zeta}_1, \hat{\zeta}_2\}$:
\begin{align}
   \label{eqn:LinRegr_zeta} \{\hat{\zeta}_1, \hat{\zeta}_2\} = \argmin_{\zeta_1, \zeta_2} \underset{j \in J}{\sum} \Big( \underset{i \in \mathcal{N} }{\sum}  {p}_i \big\|\nabla F_{i}^{(0)}(\mathbf{x}_j)\big\|^2 - \zeta_1 \big\|\underset{i \in \mathcal{N}}{\sum} {p}_i\nabla F_{i}^{(0)}(\mathbf{x}_j)\Big\|^2 - \zeta_2\Big)^2.
\end{align}
Finally, to ensure that \eqref{eq:zetas} holds for the duration of the machine learning model training $t\in[0,T]$, we will scale the obtained values (by $1.5$) and use a slightly larger value in our algorithm execution.

\begin{algorithm}[h]
	\caption{Monte-Carlo-based estimation procedure of $\{{\zeta}_1, {\zeta}_2\}$ across DPUs} \label{alg:local_zeta_est}
    \begin{algorithmic}[1]
          \STATE  \textbf{Input:} Total number of iterations $J$.
          \STATE \textbf{Output:} Final estimates $\{\hat{\zeta}_1, \hat{\zeta}_2\}$.
          \STATE Iteration count $j \leftarrow 1$, 
          \WHILE{$j \leq J$}
          \STATE Randomly generate a machine learning model parameter vector $\mathbf{x}_{j}$ at the DC  $s^{\mathsf{est}}$.
          \STATE Communicate $\mathbf{x}_{j}$ to each DPU $i \in \mathcal{N}$
          \FOR{each device $i \in \mathcal{N}$}
          \STATE DPU $i$ computes local gradient on locally held dataset ${\mathcal{D}}^{(0)}_i$ using  \eqref{eq:localLossAfter_R17} and sends back $\big\{ |\mathcal{D}^{(0)}_i|, \nabla F_{i}^{(0)}(\mathbf{x}_j) \big\}$ to $s^{\mathsf{est}}$
          \ENDFOR
          \STATE $s^{\mathsf{est}}$ locally computes and collects $\Big\{ \underset{i \in \mathcal{N} }{\sum}  {p}_i \big\|\nabla F_{i}^{(0)}(\mathbf{x}_j)\big\|^2, \big\|\underset{i \in \mathcal{N}}{\sum} {p}_i\nabla F_{i}^{(0)}(\mathbf{x}_j)\Big\|^2 \Big\}$, wherein $\{{p}_i\}$ is given by \eqref{eqn:dataset_weight}
          \STATE $j \longleftarrow j + 1$
          \ENDWHILE
          \STATE $\{\hat{\zeta}_1, \hat{\zeta}_2\}$ are calculated using  \eqref{eqn:LinRegr_zeta}
          \STATE $s^{\mathsf{est}}$ broadcasts $\{\hat{\zeta}_1, \hat{\zeta}_2\}$ across the DPUs
    \end{algorithmic}
\end{algorithm}

\noindent \underline{\textit{2)} Dynamic Parameter Estimation Methodology:} Note that the parameter estimation discussed above is executed only once before the machine learning model training. However, the estimation of the parameters can also be dynamically updated at the beginning of each global model training round. Let us denote the estimated parameters at round $t$ as $\{\tilde{\Theta}_i^{(t)}\}, \tilde{L}^{(t)},\{\tilde{\zeta}_1^{(t)}, \tilde{\zeta}_2^{(t)}\}$. To obtain these estimates, we introduce Alg. \ref{alg:dynamic_est_algo}, which dynamically executes Alg. \ref{alg:local_theta_est}, \ref{alg:local_smoothness_est}, \ref{alg:local_zeta_est} and subsequently post-processes the corresponding obtained estimates to obtain the final estimates for $\{\tilde{\Theta}_i^{(t)}\}, \tilde{L}^{(t)},\{\tilde{\zeta}_1^{(t)}, \tilde{\zeta}_2^{(t)}\}$. The post-processing step is a simple element-wise max operation between estimates at $t-1$ round and $t$ round given the fact that all these estimates are defined in upper bounds.  
\begin{algorithm}[h]
	\caption{Monte-Carlo-based dynamic parameter estimation procedure} \label{alg:dynamic_est_algo}
    \begin{algorithmic}[1]
          \STATE  \textbf{Input:} Total number of estimation iterations $J$, local datasets $\{{\mathcal{D}}_i^{(t)}\}_{i \in \mathcal{N}}$ at time $t$.
          \STATE \textbf{Output:}  Estimates of  $\{\tilde{\zeta}_1^{(t)}, \tilde{\zeta}_2^{(t)}\}, \tilde{L}^{(t)}, \{\tilde{\Theta}_i^{(t)}\}$ at round $t$.
          \FOR{$t = 0,2,..., T$}
          \FOR{$i \in \mathcal{N}$}
          \STATE Run \text{Alg. \ref{alg:local_theta_est}} with
          sampled local dataset $\{\Tilde{\mathcal{D}}_i\}_{i \in \mathcal{N}}$ used for parameter estimation, datapoints of which are chosen uniformly at random from $\mathcal{D}^{(t)}_i$, to obtain  $\hat{\Theta}_i^{(t)}$
          \STATE Obtain $\tilde{\Theta}_i^{(t)} \leftarrow \max_{0\leq t'\leq t}\{\hat{\Theta}_i^{(t')}\}$
          \ENDFOR
          \STATE Run Alg.~\ref{alg:local_smoothness_est} with
          sampled local dataset $\{\Tilde{\mathcal{D}}_i\}_{i \in \mathcal{N}}$ used for parameter estimation, datapoints of which are chosen uniformly at random from $\mathcal{D}^{(t)}_i$, to obtain  $\hat{L}^{(t)}$ at DC $s^{\mathsf{est}}$
          \STATE Obtain $\tilde{L}^{(t)} \leftarrow \max_{0\leq t'\leq t}\{\hat{L}^{(t')}\}$ at DC $s^{\mathsf{est}}$
          \STATE Run \text{Alg. \ref{alg:local_zeta_est}} with current datasets of the devices $\{\mathcal{D}^{(t)}_i\}$ to obtain $\{\hat{\zeta}_1^{(t)}, \hat{\zeta}_2^{(t)}\}$ at DC $s^{\mathsf{est}}$
          \STATE Obtain $\{\hat{\zeta}_1^{(t)}, \hat{\zeta}_2^{(t)}\} \leftarrow \max_{0\leq t'\leq t}\{\hat{\zeta}_1^{(t')}, \hat{\zeta}_2^{(t')}\} $ at DC $s^{\mathsf{est}}$
         \STATE Broadcast $\tilde{L}^{(t)}$,  $\{\tilde{\zeta}_1^{(t)}, \tilde{\zeta}_2^{(t)}\}$ from DC $s^{\mathsf{est}}$ to the DPUs.
          \ENDFOR

    \end{algorithmic}
\end{algorithm}



\newpage

We next demonstrate that the results obtained via our one-shot parameter estimation are reliable and execution of dynamic parameter estimation, which will require re-running the distributed optimization solver at each global aggregation index with the newest estimated parameters, is not necessary.
We first compare the results obtained via one shot pre-training parameter estimation and dynamic parameter estimation respectively in Fig. \ref{fig:Lest}-\ref{fig:Zeta2est}. In both one-shot pre-training and dynamic parameter estimation, we use $J = 10$ iterations. In all the plots, the curve labeled with ``Simulation Value" is the scaled value (by $1.5\textrm{x}$) of the estimation of one-shot pretraining method, which is used in our simulations in Sec. VI.

As can be observed from Fig. \ref{fig:Lest}-\ref{fig:Zeta2est},  scaling the value of the pre-training estimation for each parameter (i.e., Simulation Value) produces a reliable value that can be used in practice.
We also observe that dynamic procedure does not significantly improve the estimation over one-shot pre-training algorithm. 



\begin{figure}[h]
    \centering
    \includegraphics[width=0.4\textwidth]{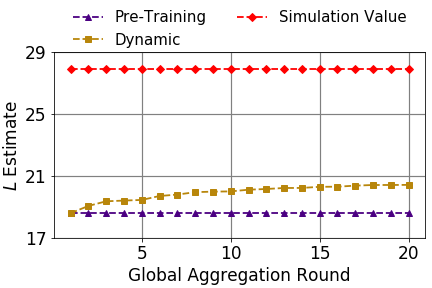} 
    \caption{Lipschitz constant $L$ estimation comparison.}
    \label{fig:Lest}
\end{figure}
\begin{figure}[h]
    \centering
    \includegraphics[width=0.4\textwidth]{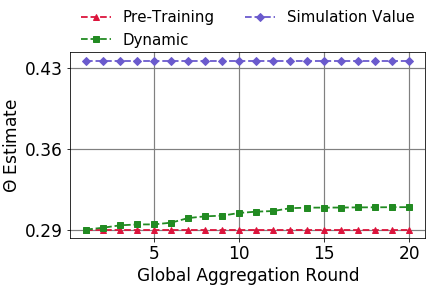} 
    \caption{Local data variability constant $\Theta$ estimation comparison.}
    \label{fig:Thetaest}
\end{figure}

\begin{figure}[h]
    \centering
    \includegraphics[width=0.4\textwidth]{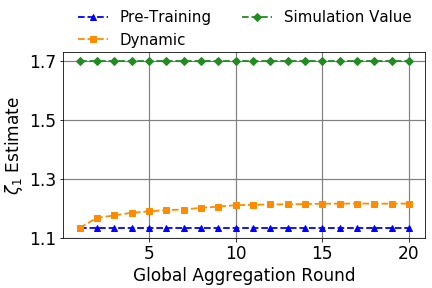} 
    \caption{Bounded dissimilarity constant $\zeta_1$ estimation comparison.}
    \label{fig:Zeta1est}
\end{figure}
\begin{figure}[h]
    \centering
    \includegraphics[width=0.4\textwidth]{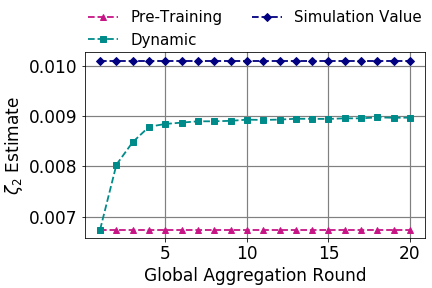} 
    \caption{Bounded dissimilarity constant $\zeta_2$ estimation comparison.}
    \label{fig:Zeta2est}
\end{figure}

}

\clearpage

\section{\texttt{CE-FL} Complexity Analysis} \label{App:complexity}
{\color{black}
We  first elaborate on the space complexity analysis of the problem in terms of the number of parameters associated with the \textit{network optimization} as well as the \textit{ML model training} frameworks. We summarize below the variables computed at each node of the network pertaining to network optimization:
\begin{enumerate}[leftmargin=4.2mm]
    \item Each UE $n$: 
    {\small $ \bm{w}^{\mathsf{Local}}_{n} =  \big\{ [f_n^{(t)}] ,[m_n^{(t)}], [\gamma_n^{(t)}], {[I_{n,b}^{(t)}]_{b\in\mathcal{B}}} \big\}_{t=1}^{T} , \\
       \bm{w}^\mathsf{Shared}_{n} = \big\{ [{\rho^{(t), n}_{n,b}}]_{b\in\mathcal{B}}, [I_s^{ (t), n}]_{s\in\mathcal{S}}, [{\delta}^{\mathsf{A},(t), n}]\big\}_{t=1}^{T}, \label{eqn: UE_local_vars}
    $}
    where {\small $\rho^{(t), n}_{n,b},I_s^{ (t), n},{\delta}^{\mathsf{A},(t), n}$} denote the local copies of the shared variables (i.e., {\small $\rho^{(t)}_{n,b},I_s^{ (t)},{\delta}^{\mathsf{A},(t)}$}) at node $n$. Therefore, each UE $n$ locally optimizes for $T(4 + 2|\mathcal{B}| + |\mathcal{S}|)$ variables.
    \item Each BS $b$: 
    {\small
    $ \bm{w}^{\mathsf{Local}}_{b} = \big\{ [I_{b,n}^{(t)}]_{n\in\mathcal{N}} \big\}_{t=1}^{T} ,
      \bm{w}^\mathsf{Shared}_{b} =  \nonumber \big\{ [{\rho^{(t), b}_{n,b}}]_{n\in\mathcal{N}}, [{\rho^{(t), b}_{b,s}}]_{s\in\mathcal{S}} , [I_s^{(t), b}]_{s\in\mathcal{S}}, [{\delta}^{\mathsf{A},(t), b}], [{\delta}^{\mathsf{R},(t), b}] , [{R^{(t), b}_{b,s}}]_{s\in\mathcal{S}} \big\}_{t=1}^{T}
    $
    }
     where {\small${\rho^{(t), b}_{n,b}},{\rho^{(t), b}_{b,s}},I_s^{(t), b},{\delta}^{\mathsf{A},(t), b},{\delta}^{\mathsf{R},(t), b}, {R^{(t), b}_{b,s}} $} denote the local copies of the respective shared variables (i.e., {\small${\rho^{(t)}_{n,b}},{\rho^{(t)}_{b,s}},I_s^{(t)},{\delta}^{\mathsf{A},(t)},{\delta}^{\mathsf{R},(t)}, {R^{(t)}_{b,s}}$}) at node $b$.
      Therefore, each BS $b$ locally optimizes for $T(2 + 2|\mathcal{N}| + 3|\mathcal{S}|)$ variables.
    \item Each DC $s$: 
     {\small $\bm{w}^{\mathsf{Local}}_{s} = \big\{  [z_s^{(t)}],[\gamma_s^{(t)}],  
      [m_s^{(t)}] \big\}_{t=1}^{T}  
   , \bm{w}^\mathsf{Shared}_{s} = \big\{ [{\rho^{(t), s}_{n,b}}]_{n\in\mathcal{N},b\in\mathcal{B}}, [{\rho^{(t), s}_{b,s}}]_{b\in\mathcal{B}}, [I_s^{(t), s}]_{s\in\mathcal{S}},
      [{\delta}^{\mathsf{A},(t), s}], 
[{\delta}^{\mathsf{R},(t), s}]$, 
$[{R^{(t), s}_{b,s}}]_{b\in\mathcal{B}} \big\}_{t=1}^{T}$}
       where ${\rho^{(t), s}_{n,b}},{\rho^{(t), s}_{b,s}},{\delta}^{\mathsf{A},(t), s},I_s^{(t), s},{\delta}^{\mathsf{R},(t), s}, {R^{(t),s}_{b,s}}$ denote the local copies of the respective shared variables (i.e., ${\rho^{(t)}_{n,b}},{\rho^{(t)}_{b,s}},{\delta}^{\mathsf{A},(t)},I_s^{(t), s},{\delta}^{\mathsf{R},(t)}, {R^{(t)}_{b,s}}$) at  node $s$. Therefore, each DC $s$ locally optimizes for $T(5 + |\mathcal{N}||\mathcal{B}| + |\mathcal{S}| + 2|\mathcal{B}|)$ variables.
\end{enumerate}
Subsequently, the time complexity in terms of number of machine instructions pertaining to the network optimization solver involving $J_1$ rounds of \texttt{Successive Convex Solver} (Algorithm 1 in manuscript), $J_2$ rounds of \texttt{Iterative Primal-Dual Method} (Algorithm 2 in manuscript) and $J_3$ rounds of \texttt{Iterative Decentralized Consensus Method} (Algorithm 3 in manuscript) is ${\mathcal{O}}(J_1 J_2 J_3 |\mathcal{N}||\mathcal{B}|)$. This time complexity results is a direct consequence of the following two facts:
\begin{enumerate}
    \item The UEs in $\mathcal{N}$, BSs in $\mathcal{B}$, and DCs in $\mathcal{S}$ compute their associated network optimization parameters in parallel within the \texttt{Successive Convex Solver} (Algorithm 1 in manuscript). Therefore, the time complexity at any node is ${\mathcal{O}}(|\mathcal{N}||\mathcal{B}|)$ per iteration.
    \item The number of machine instructions involved in each iteration of \texttt{Iterative Primal-Dual Method} (Algorithm 2 in manuscript) is linearly proportional to the number of the parameters it updates. Therefore, each iteration of \texttt{Iterative Primal-Dual Method} leads to worst case time complexity ${\mathcal{O}}(J_2*\text{max number of variables to be updated across all the nodes})$. Here, $J_2$ is the total number of iterations of \texttt{Iterative Decentralized Consensus Method} (Algorithm 3 in manuscript).
\end{enumerate}

The time and storage complexity associated with ML model training involving neural network architectures are explained  in Sec. 1.2.1 of~\cite{NN_complexity}. 
}

\end{document}